\documentclass[pdflatex,sn-mathphys-num]{sn-jnl}
\usepackage[utf8]{inputenc} % allow utf-8 input
\usepackage[T1]{fontenc}    % use 8-bit T1 fonts
\usepackage{hyperref}       % hyperlinks
\usepackage{url}            % simple URL typesetting
\usepackage{booktabs}       % professional-quality tables
\usepackage{enumitem}

\usepackage{graphicx}
\usepackage{caption}
\usepackage{subcaption}
\usepackage{xcolor}

\usepackage{amsmath}
\usepackage{mathtools}
\usepackage{amssymb}
\usepackage{amsfonts}
\usepackage{bbm}
\usepackage{hyperref}
\usepackage{algorithm2e}
\RestyleAlgo{ruled}

\usepackage{amsthm}

\newtheorem{lem}{Lemma}

\theoremstyle{definition}

\theoremstyle{remark}

\begin{document}

%=====================Front========================
%
\title{
A Unified Variational Framework for Deep Weakly Supervised Image Segmentation
}
%
% ---------------------Authors and Institutes------------------
\author[1]{\fnm{Yin King} \sur{Chu}}\email{ykchuac@connect.ust.hk}
\author[2]{\fnm{Lingfeng} \sur{Li}}\email{lilingfeng@himis-sz.cn}
\author[3]{\fnm{Sung Ha} \sur{Kang}}\email{kang@math.gatech.edu}
\author[4]{\fnm{Jianping} \sur{Zhang}}\email{jpzhang@xtu.edu.cn}
\author[5]{\fnm{Xue-Cheng} \sur{Tai}}\email{xtai@norceresearch.no}

\affil[1]{
    \orgdiv{Department of Mathematics}, 
    \orgname{The Hong Kong University of Science and Technology}, 
    \orgaddress{\city{Clear Water Bay, Kowloon}, \state{Hong Kong}, \country{China}}
}

\affil[2]{
    \orgname{Hetao Institute of Mathematics and Interdisciplinary Sciences}, 
    \orgaddress{\city{Shenzhen}, \state{Guangdong}, \country{China}}
}

\affil[3]{
    \orgdiv{School of Mathematics}, 
    \orgname{Georgia Institute of Technology}, 
    \orgaddress{\city{Atlanta}, \state{GA}, \postcode{30332}, \country{USA}}
}

\affil[4]{
    \orgdiv{School of Mathematics and Computational Science}, 
    \orgname{Xiangtan University}, 
    \orgaddress{\city{Xiangtan}, \state{Hunan}, \postcode{411105}, \country{P.R. China}}
}

\affil[5]{
    \orgname{Norwegian Research Center, Fantoftveien 38}, \postcode{5072}, 
    \orgaddress{\city{Bergen}, \country{Norway}}
}

%=====================Abstract=====================
%
\abstract{
We propose a unified variational framework for image segmentation under sparse pixel-level supervision. Our method is based on a simplex-constrained Potts model with a smooth perimeter regularizer, yielding a convex, smooth energy functional that can be used as a training loss in weakly supervised deep learning paradigms or optimized efficiently using iterative methods. Sparse labels are incorporated into the data fidelity term by constructing a fuzzy membership function via a function extension problem in a Reproducing Kernel Hilbert Space (RKHS), which can effectively capture inhomogeneous intensity statistics. The derived discrete loss for training standard networks demonstrates robustness and consistent improvements over non-training and partial cross-entropy (PCE) baselines in experiments, achieving comparable performance without requiring ground-truth segmentation images.
}
%
% \keywords{keyword1 \and keyword2 \and keyword3}

\maketitle

%=====================Introduction===================

% \tai{The figures need to be aligned. The captions need to be revised and detailed. }
\section{Introduction}
Image segmentation is a fundamental task in computer vision that involves partitioning an image into non-overlapping regions. It has a broad range of applications across diverse fields \cite{mattyus2017deeproadmapper,milletari2016v,kampffmeyer2016semantic,pham2000current}.

Generally, segmentation algorithms can be categorized into two primary frameworks: iterative methods and deep learning-based approaches. Iterative methods provide a mathematically rigorous and interpretable framework by formulating segmentation as an energy minimization problem and solving it using iterative algorithms. The energy functional typically consists of a data-fidelity term, a boundary-regularization term, and additional task-dependent constraints. Representative functionals include the Mumford-Shah model \cite{mumford1989optimal}, the Chan-Vese model \cite{chan2001active}, and the Potts model \cite{potts1952some}. The data-fidelity term is crucial for driving the segmentation \cite{chan2001active}. For complex semantic segmentation tasks, its construction often relies on user-specified scribble points as input. Consequently, the key to obtaining accurate segmentation results lies in effectively incorporating information from these scribbles.

In contrast, deep learning methods leverage Deep Neural Networks (DNNs) to directly learn the mapping from an input image to a segmentation mask. Benefiting from the high expressive power of DNNs \cite{zhou2020universality}, these methods have achieved remarkable success in recent years \cite{ronneberger2015u,chen2017deeplab,strudel2021segmenter}. However, standard DNN-based methods require large-scale datasets with costly, fully-annotated pixel-level labels for training. To overcome this, a line of research has focused on weak supervision, using either a small number of image-level labels \cite{pathak2015constrained,ahn2018learning,duan2019automatic,lu2020geometry,chen2023adversarial} or sparse pixel-level labels like scribbles \cite{kim2019mumford,liang2022tree,lin2016scribblesup,tang2018normalized,tang2018regularized,wu2023sparsely,zhang2023weakly}. While sparse labels provide more information for accurate segmentation, existing approaches often rely on a partial cross-entropy (PCE) loss. This term is data-driven and provides no explicit regularization, which can be unstable in many scenarios. Several approaches \cite{kim2019mumford,tang2018normalized,tang2018regularized,zhang2023weakly} have tried to alleviate this problem by embedding the term along with a classic segmentation energy to impose regularization. They have demonstrated promising improvements but failed to provide a concrete analysis. Consequently, their performance is severely limited. 

We introduce a unified variational framework for both iterative methods with sparse labels and weakly-supervised deep learning methods. Our framework is derived from the variational setting of image segmentation. We formulate a new energy that utilizes both intensity and spatial information based on a simplex-constrained Potts model. A smooth approximation of the total variation term is employed as the perimeter regularization. The label information is incorporated into a data fidelity term by a pre-computed fuzzy membership function. Our energy minimization problem is convex and smooth. It can be easily solved by existing fast first-order algorithms and avoids the need for alternating projection. 

The label-based fuzzy membership function is obtained by solving a function extension problem in the Reproducing Kernel Hilbert Space (RKHS). The choice of kernel parameters can deterministically control the region-distribution regularization, learning the inhomogeneous intensity distribution. Unlike \cite{ha2010image,kang2014supervised}, we use a thresholding projection, which avoids the computation of a Moore-Penrose inverse and reduces the complexity to $O(Km^3+K\log(K))$, where $m$ is the number of labeled pixels and $K$ is the number of classes. Since the labels are sparse, $m$ is small, so this step is efficient. A downsampling strategy is also implemented on the original image, which further reduces the computational and storage costs. Experimental results show that our model can effectively segment various objects even in semantically complex cases, such as in the presence of water splashes, illumination bias, and noise. 

In addition, we illustrate the relationship between our model and a continuous learning problem. This learning problem explains the process of network training. We then present a relaxation of our energy and derive a discrete loss for weakly-supervised learning. We carefully analyze the effects via extensive experiments. We find that training exhibits three major effects in comparison to the standard model: denoising, object edge refinement, and boundary artifact elimination. In prediction, our results also demonstrate robustness and consistent improvements over existing PCE baselines.

The rest of the paper is organized as follows. We illustrate the proposed model and its relaxed framework in Section~\ref{Sec:Model}. The computation of the fuzzy membership function is presented in Section~\ref{Sec:RKHS_Function_Extension}. The model implementation for single image segmentation is discussed in Section~\ref{Sec:Single_Image_Segmentation}. The experimental studies of the relaxed model for weakly-supervised network training are analyzed in Section~\ref{Sec:Network}. Finally, the conclusion is drawn in Section~\ref{Sec:Conclusion}.

Throughout this paper, we use lowercase letters to denote scalars or scalar-valued functions. Vector-valued functions are denoted by lowercase boldface letters, and operators are denoted by calligraphic letters. The vectors and matrices obtained by discretizing them will be notated by a hat or check. 

%=============The Proposed Energy============
\section{Model Formulation} \label{Sec:Model}

\subsection{The General Potts Model}
Let $I(x)$ be a continuous image defined on an open rectangular domain $\Omega\subset\mathbb R^2$. The segmentation task is to partition $\Omega$ into $K$ distinct phases $\Omega_k$ based on the information of $I$. We consider solving the Potts model to achieve this. 
The Potts model for image segmentation is a linear combination of the data fidelity term and the perimeter regularization with a weight parameter $\lambda$. It is formulated as the following optimization problem:
\begin{align}
    \min_{\{\Omega_k\}_{k=1}^K} & \sum_{k=1}^K\int_{\Omega_k} f_k(x)dx + \lambda\sum_{k=1}^K |\partial\Omega_k| \ 
    s.t.\  \bigcup_{k=1}^K\Omega_k=\Omega;\ \Omega_k\cap\Omega_l=\emptyset,\ \forall k\neq l.
\end{align}
By taking $f_k=\|I-c_k\|_2^2$ with the $l_2$-norm $\|\cdot\|_2^2$ and $c_k=\frac{\int_{\Omega_k}Idx}{|\Omega_k|}$, this model reduces to the classical Chan-Vese energy. Thus, the Potts model can be regarded as an extension of the Chan-Vese model. Indeed, it can be further extended to more sophisticated models by replacing the perimeter regularization as other spatial regularization, such as the Cheeger's cut $\frac{|\partial\Omega|}{|\Omega|}$ \cite{bresson2014multi} or the compactness term $\frac{|\partial\Omega_k|^2}{|\Omega_k|}$ \cite{gui2017medical,zhang2025deep}. These general Potts models will be in the form of:
\begin{align}
    \min_{\{\Omega_k\}_{k=1}^K} & \sum_{k=1}^K\int_{\Omega_k} f_k(x)dx + \lambda\sum_{k=1}^K \mathcal{R}(\Omega_k)\     \label{eq:Potts_general} 
    s.t.\  \bigcup_{k=1}^K\Omega_k=\Omega;\  \Omega_k\cap\Omega_l=\emptyset,\ \forall k\neq l. 
\end{align}
with general spatial regularization $\mathcal{R}(\Omega_k)$. 

\subsection{Simplex-constrained Potts Model with Smooth Approximations}

In this paper, we consider a simplex-constrained representation for (\ref{eq:Potts_general}). Each phase $\Omega_k$ is represented by an indicator function $v_k:\Omega\rightarrow\{0,1\}\in BV(\Omega)$, where $BV(\Omega)$ denotes the space of functions of bounded variation, for $k=1,\dots,K$, and $\mathbf{v}=(v_1, ..., v_K)$ belongs to the constraint set 
\begin{align*}
    \Delta^K = \left\{(v_1, ..., v_K) : v_k\in BV(\Omega),v_k(x)\in\{0,1\} , \sum_{k=1}^K v_k(x)=1,\forall x\in\Omega \right\}.
\end{align*}
The perimeter of $\partial\Omega_k$ can be approximated by
\begin{align}
    |\partial\Omega_k|\approx C_\sigma\int_\Omega \left(1-v_k(x)\right)(G_\sigma\ast v_k)(x)dx \label{eq:TD_approximation}
\end{align}
when $\sigma$ is sufficiently small \cite{pallara2007short}.
Here, $C_\sigma$ is some constant depending on $\sigma$. 
%$g(x)=\frac{1}{1+\beta\|\nabla G\ast I(x)\|}$ is the edge detector and 
$G_\sigma(x)=\frac{1}{\sqrt{2\pi\sigma^2}}\exp({-\frac{|x|^2}{2\sigma^2}})$ is the Gaussian kernel and $\ast$ denotes the convolution operator. 
This approximation is simple to compute, and has been shown in \cite{Liu2011} to be highly effective for image segmentation compared to other methods for the total variation (TV) term and particularly adept at accurately addressing the well-known triple-junction problem. Since its introduction, there has been a significant increase in interest in using this approximation for image segmentation issues, as evidenced in studies such as \cite{esedog2015threshold, wang2017efficient, liu2022deep, liu2025convex, tai2024pottsmgnet, zhang2025deep}.

Combining the simplex representation and smooth approximation of the boundary length, we have the following optimization problem:
\begin{align} \label{eq:simplexPotts_nonconvex}
    \min_{\textbf{v}\in\Delta^K} \sum_{k=1}^K\int_\Omega f_k(x)v_k(x)dx + \lambda\sum_{k=1}^K \int_\Omega \left(1-v_k(x)\right)(G_\sigma\ast v_k)(x)dx .
\end{align}
The solution space is non-convex, but we can relax the model as follows:
\begin{align} \label{eq:simplexPotts_convex}
    \min_{\textbf{v}\in\tilde{\Delta}^K} \sum_{k=1}^K\int_\Omega f_k(x)v_k(x)dx + \lambda\sum_{k=1}^K \int_\Omega \left(1-v_k(x)\right)(G_\sigma\ast v_k)(x)dx ,
\end{align}
where
\begin{align*}
    \tilde{\Delta}^K = \left\{\mathbf{v}=(v_1, ..., v_K):v_k\in BV(\Omega),  v_k(x)\in [0,1] , \sum_{k=1}^K v_k(x)=1,\forall x\in\Omega \right\} .
\end{align*}
We note here that the solution of the relaxed problem (\ref{eq:simplexPotts_convex}) is equivalent to the original problem (\ref{eq:simplexPotts_nonconvex}). See explanations and proofs in \cite[section 3]{tai2023potts} and \cite[Lemma 2.1]{wang2017efficient}.

Our proposed model will be based on (\ref{eq:simplexPotts_convex}). More detailed introduction to the Potts model in image segmentation problems can be found in \cite{tai2023potts}.

\subsection{The Proposed Model}
%(u retrieves the label info.)
Consider the segmentation problem with sparse labels, where only a small subset of pixels in the image domain $\Omega$ have pre-assigned labels. Let $D \subset \Omega$ be this set of labeled pixels. Define a label function $\psi = (\psi_1, \ldots, \psi_K)$ to represent the corresponding labeling, where each component $\psi_k: D \to \{0, 1\}$ is an indicator function for class $k$, and for any pixel $x \in D$, we have $\psi_k(x) = 1$ if $x$ belongs to class $k$, and $\psi_k(x) = 0$ otherwise, with the constraint that $\sum_{k=1}^K \psi_k(x) = 1$.

This label function is valuable for obtaining accurate segmentation results, especially for complex semantic tasks. Effective incorporation of the label information is crucial. One existing approach poses these labeled data as a hard constraint, i.e., by adding the constraint
\[ v_k(x)=\psi_k(x),\quad x\in D,  \]
directly to (\ref{eq:simplexPotts_convex}). However, this makes the problem non-smooth and non-convex. An alternating descent algorithm is required to keep projecting the solution onto the constraint space (see, e.g., \cite{luo2019convex}).

Our strategy for handling sparse labels is to incorporate them into the data fidelity term of the Potts model (\ref{eq:simplexPotts_convex}). First, we define a fuzzy membership function $\mathbf{u}(x)=(u_1(x),\ldots, u_K(x))$ over the entire image domain $\Omega$ to represent the sparse label information. This function $\mathbf{u}$ belongs to the space $\tilde{\Delta}_U^K$:
\begin{align*}
    \tilde{\Delta}_U^K = \left\{\mathbf{u}=(u_1, ..., u_K) : u_k(x)\in [0,1] , \sum_{k=1}^K u_k(x)=1,\forall x\in\Omega \right\}.
\end{align*}
Note that $\mathbf{u}$ is not required to be in $BV(\Omega)$. 
By taking the data fidelity term $f_k=1-2u_k$, we obtain the final segmentation by solving the following minimization problem:
\begin{align}\label{eq:main_model}
    \min_{\textbf{v}\in\tilde{\Delta}^K} \sum_{k=1}^K\int_\Omega \left(1-2u_k(x)\right)v_k(x)dx + \lambda\sum_{k=1}^K \int_\Omega \left(1-v_k(x)\right)(G_\sigma\ast v_k)(x)dx .
\end{align}
The smoothness of (\ref{eq:main_model}) makes it suitable for gradient-based solvers and for transformation into a loss function for weakly-supervised learning, as will be discussed in the next subsection. 

For single image segmentation, we solve the model using a fast iterative thresholding method \cite{Liu2011,wang2017efficient}. The per-iteration complexity of this algorithm depends only on a Gaussian convolution. Furthermore, a rough segmentation, which we term 'pre-segmentation', can be obtained by thresholding the fuzzy membership function $\mathbf u$. This pre-segmentation serves as an effective initialization, accelerating the solver. Appendix~\ref{sec:td_implementation} provides the algorithmic details.

Another benefit of pre-computing $\mathbf{u}$ is that the model parameters are easier to tune. We can choose the parameters for solving $\mathbf{u}$ based only on label information. After that, the only parameter of the Potts model, i.e., $\lambda$, can be chosen solely depending on the degree of the edge-smoothing effect we want. 
% In another fold, the model (\ref{eq:main_model}) can be easily solved by linearization methods such as iterative thresholding or smooth methods like gradient descent when $\mathbf{u}$ is available, as it is just an example of Problem (\ref{eq:simplexPotts_convex}). This 

To construct this label-based fuzzy membership function, we propose a method based on a function extension problem in the Reproducing Kernel Hilbert Space (RKHS). The label function $\psi$ is extended to $\Psi=(\Psi_1,\ldots, \Psi_K)\in\mathcal{H_K}(\Omega)$ after solving this extension problem, where $\mathcal{H_K}(\Omega)$ is an RKHS governed by a reproducing kernel $\mathcal{K}$. The $\Psi_k$ are then projected as fuzzy membership functions via a thresholding method to obtain the final $\mathbf u$. Details of the computation are illustrated in Section \ref{Sec:RKHS_Function_Extension}. 

The primary advantage of this RKHS extension lies in "kernel learning." The choice of a reproducing kernel $\mathcal{K}$ allows for imposing useful regularity conditions on $\Psi$ and reduces the problem to a few parameter choices. For example, a nonlinear kernel based on image intensity or spatial information can establish inhomogeneous pixel similarities, leading to a high-quality partition. In the simplest case, only one parameter is needed to tune intensity and spatial similarity, respectively. We demonstrate this advantage experimentally in Section~\ref{sec:rkhs_thre_effect}.

\subsection{Loss Function for Weakly Supervised Learning}\label{sec:loss_function}

We derive a weakly-supervised training loss from model~(\ref{eq:main_model}) by casting learning in a continuous variational form.

% \tai{Need to define $\mathcal{M}$ before using it. }

Let $\{I_i\}_{i=1}^n$ be a training set of images, where each $I_i$ is defined on a domain $\Omega \subset \mathbb{R}^2$. For each image $I_i$, sparse labels are provided on a subset $D_i \subset \Omega$ via a label function $\psi_i = (\psi_{i,1}, \ldots, \psi_{i,K})$, where $\psi_{i,k}: D_i \to \{0,1\}$ and $\sum_{k=1}^K \psi_{i,k}(x) = 1$ for all $x \in D_i$. We aim to learn a segmentation operator $\mathcal{N} = (\mathcal{N}_1, \ldots, \mathcal{N}_K)$ that maps an image to a segmentation mask in $\tilde{\Delta}^K$.

In a fully supervised setting where labels are available on the entire domain ($D_i = \Omega$), the operator $\mathcal{N}$ is typically found by minimizing a loss function over the training set:
\begin{equation}
    \min_{\mathcal{N}} \sum_{i=1}^n \int_{\Omega} l(\mathcal{N}(I_i)(x), \psi_i(x)) dx, \label{eq:supervised_learning_loss}
\end{equation}
where $l$ is a suitable discrepancy measure, such as cross-entropy. However, in our weakly supervised setting, $\psi_i$ is undefined on $\Omega \setminus D_i$, making \eqref{eq:supervised_learning_loss} ill-defined.
To overcome this, we first construct a fuzzy membership function $\mathbf{u}_i = (u_{i,1}, \ldots, u_{i,K})$ for each image, extending the sparse label information from $D_i$ to the entire domain $\Omega$. We then replace the optimization variable $\mathbf{v}$ in \eqref{eq:main_model} with the output of the operator $\mathcal{N}(I_i)$. This leads to the following continuous learning problem:
\begin{align}
\begin{split}
\min_{\mathcal{N}} \sum_{i=1}^n \Bigg[
&\sum_{k=1}^K \int_{\Omega} \left(1-2u_{i,k}(x)\right)\mathcal{N}_k(I_i)(x)\,dx \\
&+ \lambda \sum_{k=1}^K \int_{\Omega} \left(1-\mathcal{N}_k(I_i)(x)\right)
\left(G_\sigma * \mathcal{N}_k(I_i)\right)(x)\,dx
\Bigg].
\end{split}
\label{eq:Potts_rkhs_operator}
\end{align} 
For network training, we discretize \eqref{eq:Potts_rkhs_operator} and parameterize the operator $\mathcal{N}$ with a neural network $N_\theta$, such as a UNet \cite{ronneberger2015u}, where $\theta$ denotes its learnable parameters. Let $\hat{I}_i$ and $\hat{\mathbf{u}}_i$ be the discretizations of the image $I_i$ and fuzzy membership function $\mathbf{u}_i$ on a uniform $h \times w$ grid. The network output $N_\theta(\hat{I}_i) = (N_{\theta,1}(\hat{I}_i), \ldots, N_{\theta,K}(\hat{I}_i))$ is a probability map, typically taking softmax layer as the last layer of the networks, with each $N_{\theta,k}(\hat{I}_i) \in [0,1]^{h \times w}$. The per-image loss function is then:
\begin{equation} \label{eq:Potts_rkhs_loss}
    L_i(\theta)=\sum_{k=1}^K \langle 1-2\hat{u}_{i,k},N_{\theta,k}(\hat I_i)\rangle_{\mathbb R^{h\times w}}+\lambda\sum_{k=1}^K \langle 1-N_{\theta,k}(\hat{I}_i),\hat{G}_{\sigma}\hat  \ast N_{\theta,k}(\hat{I}_i)\rangle_{\mathbb R^{h\times w}} .
\end{equation}
where $\hat \ast$ is the discrete 2-dimensional convolution, $\hat{G}_\sigma$ is a discrete Gaussian convolutional kernel with standard deviation $\sigma$, $\hat{\mathbf{u}}_i=(\hat{u}_{i,1},\ldots, \hat{u}_{i,K})$, and $\langle \cdot,\cdot\rangle_{\mathbb R^{h\times w}}$ is an inner product defined by
\[ \langle A,B\rangle_{\mathbb R^{h\times w}}=\sum_{i=1}^h\sum_{j=1}^w A[i,j]B[i,j],\quad A,B\
\in\mathbb R^{h\times w}.\]
We refer to 'training' as the process of solving
\begin{align} \label{eq:Potts_rkhs_training}
    \min_{\theta}\sum_{i=1}^n L_i(\theta)
\end{align}
Implementation details are given in Section~\ref{sec:network_implementation}.

%=========================RKHS=====================
\section{Label Function Extension with Reproducing Kernels} \label{Sec:RKHS_Function_Extension}

% \textit{Subsection Content:
% \begin{enumerate}
%     \item function extension problem, propose to choose $\Psi\in\mathcal{H_K}$.
%     \item RKHS $\mathcal{H_K}$ basic, as in \cite{ha2010image}.
%     \item Discrete $D$ to solve (1), Regularized Least Square Algorithm as in \cite{ha2010image}.
%     \item mentions the Euclidean case.
%     \item Projection of $\Psi$ into the space of fuzzy membership function $\{u|u:\Omega\rightarrow[0,1]^K\}$, vs \cite{kang2014supervised}
% \end{enumerate}}

\subsection{The Regularized Least-squares Function Extension Problem}

Let $\mathcal{V}$ be a Hilbert space and $\psi:D \rightarrow \mathcal{V}$ be a function defined on $D$, where $D$ is a subset of $\Omega$. A general function extension problem aims to construct a function $\Psi:\Omega\rightarrow \mathcal{V}$ defined on the larger domain $\Omega$ such that it provides a good approximation of $\psi$ on $D$. In a variational approach, this extension is computed via the following minimization:
\begin{align} \label{eq:general_function_extension}
    \inf_{\Psi\in\mathcal{H}_2(\Omega)} \{\mathcal{F}_1(\psi-\Psi)+\gamma\mathcal{F}_2(\Psi)\}
\end{align}
where $\mathcal{F}_1$ and $\mathcal{F}_2$ are functionals defined on the function spaces $\mathcal{H}_1(D)$ and $\mathcal{H}_2(\Omega)$ to which $\psi$ and $\Psi$ belong, respectively, and $\gamma$ is a tuning parameter.
This problem is strongly regularized by the choice of $\mathcal{H}_2(\Omega)$. In this work, we choose $\mathcal H_2(\Omega)$ as a Reproducing Kernel Hilbert Space (RKHS). This choice is widely used in classic machine learning, as the inner product and smoothing properties of the space can be derived by defining a special type of operator-valued kernel $\mathcal{K}$, i.e., the reproducing kernel. A more detailed introduction to RKHS can be found in Appendix \ref{sec:rkhs}. We will use $\mathcal{H_K}(\Omega)$ to denote the RKHS induced by the reproducing kernel $\mathcal{K}$.

We now introduce the function extension problem for constructing the label-based fuzzy membership function $\mathbf{u}$ in the context of image segmentation. Our derivation is an extension of \cite{ha2010image,kang2014supervised}. In \eqref{eq:general_function_extension}, we set $\Omega\in\mathbb R^2$ as the continuous image domain. Assume $D=\{x_i\}_{i=1}^m$ is a set of discrete points in $\Omega$ with known labels and $\psi:D\rightarrow\mathcal V$ is a given label function that assigns a label to each $x_i\in D$. Here, $\mathcal{V}$ is set to $\mathbb{R}$ or $\mathbb{R}^K$, corresponding to the binary or multiphase segmentation problems, respectively. We aim to find an extension of $\psi$ from $D$ to $\Omega$ and use it as the fuzzy membership function in \eqref{eq:main_model}.

Let $\mathcal{V}^D$ be the vector space of all functions $\psi:D\rightarrow\mathcal{V}$, where the induced norm is defined as $\Vert \psi\Vert_{\mathcal V^D}:=\sqrt{\sum_{i=1}^m \Vert\psi(x_i)\Vert_{\mathcal V}^2}$ for any $\psi\in\mathcal V^D$. Given a reproducing kernel $\mathcal{K}$, we take $\mathcal{H}_1(D)$ and $\mathcal{H}_2(\Omega)$ as $\mathcal{V}^D$ and $\mathcal{H_K}(\Omega)$, respectively, and define $\mathcal F_1$ and $\mathcal F_2$ as the corresponding form of the squared norm. Then, we have the following minimization problem for the labeled function extension:
\begin{align} \label{eq:rkhs_function_extension}
    \min_{\Psi\in\mathcal{H_K}(\Omega)} \frac{1}{2m}\sum_{i=1}^m \| \psi(x_i)-\Psi(x_i)\|_{\mathcal{V}}^2
    + \frac{\gamma}{2} \|\Psi\|_{\mathcal{H_K}(\Omega)}^2.
\end{align}
Let $\mathcal{S}_{\mathbf x}:\mathcal{H_K}(\Omega)\rightarrow\mathcal{V}^D$ be the sampling operator defined by $\mathcal{S}_{\mathbf x}\Psi=(\Psi(x_1),\dots,\Psi(x_m))$; the first term in (\ref{eq:rkhs_function_extension}) can then be written as $\frac{1}{2m}\Vert \psi-S_{\mathbf x}\Psi\Vert_{\mathcal V^D}^2$. 
The functional (\ref{eq:rkhs_function_extension}) is convex and has a unique minimizer $\Psi^\gamma$ that satisfies the normal equation (see \cite[Theorem 6.5]{clason2020regularization}):
\begin{align} \label{eq:rkhs_normal_eq}
    (\mathcal{S}_\mathbf{x}^*\mathcal{S}_\mathbf{x}+m\gamma \mathcal{I})\Psi^\gamma=\mathcal{S}_\mathbf{x}^*\psi
\end{align}
where $\mathcal S^*_{\mathbf x}$ is the adjoint operator of $\mathcal S_{\mathbf x}$.
By the reproducing property of the RKHS (See (\ref{eq:reproducing_property}) in Appendix \ref{sec:rkhs}), we get 
\begin{align*}
    \left\langle\mathcal{S}_\mathbf{x}\Psi, \psi\right\rangle_{\mathcal{V}^D} &= \sum_{i=1}^m\langle\Psi(x_i), \psi(x_i)\rangle_{\mathcal{V}} 
    = \sum_{i=1}^m\langle\Psi, \mathcal{K}_{x_i}\psi(x_i)\rangle_{\mathcal{H_K}(\Omega)} 
    = \left\langle\Psi, \sum_{i=1}^m\mathcal{K}_{x_i}\psi(x_i)\right\rangle_{\mathcal{H_K}(\Omega)}
\end{align*}
Thus, the adjoint operator $\mathcal{S}_\mathbf{x}^*\psi(\mathbf{x})=\sum_{i=1}^m\mathcal{K}_{x_i}\psi(x_i)$. Substituting the formula of $\mathcal{S}_\mathbf{x}$ and $\mathcal{S}_\mathbf{x}^*$ into (\ref{eq:rkhs_normal_eq}), we have
\begin{align}
    \sum_{i=1}^m\mathcal{K}_{x_i}\Psi^\gamma(x_i)+m\gamma\Psi^\gamma=\sum_{i=1}^m\mathcal{K}_{x_i}\psi(x_i),
%\end{align*}
\mbox{  which implies  }
%\begin{align} 
    \label{eq:sol_Psi_forward}
    \Psi^\gamma=\sum_{i=1}^m\mathcal{K}_{x_i}a_i,
\end{align}
where $a_i=\frac{\psi(x_i)-\Psi^\gamma(x_i)}{m\gamma}$. Note that 
$\Psi^\gamma(x)=\sum_{j=1}^m\mathcal{K}(x,x_j)a_j$. Then $a_i=\frac{\psi(x_i)-\sum_{j=1}^m\mathcal{K}(x_i,x_j)a_j}{m\gamma}$. Equivalently, $a_i$ can be computed by solving the inverse problem directly \cite[Proposition 1]{ha2010image}:
\begin{align}\label{eq:sol_Psi_backward}
    \sum_{j=1}^m\mathcal{K}(x_i,x_j)a_j+m\gamma a_i=\psi(x_i).
\end{align}

For multiphase segmentation, we set $\mathcal{V}=\mathbb{R}^K$. The label function $\psi:D\to\mathbb R^K$ is given by $\psi(x_i) = (\psi_1(x_i),\ldots,\psi_K(x_i))$, where $\psi_k(x_i)$ is the one-hot encoding for class $k$. We seek an extension $\Psi=(\Psi_1,\ldots,\Psi_K)$ where each component $\Psi_k$ lies in a scalar-valued RKHS $\mathcal{H}_{\mathcal{K}_k}(\Omega)$ induced by a real-valued, symmetric, positive-definite kernel $\mathcal{K}_k$. This corresponds to an operator-valued kernel $\mathcal{K}(x,y) = \mathrm{Diag}(\mathcal{K}_1(x,y),\ldots,\mathcal{K}_K(x,y))$. The norms in \eqref{eq:rkhs_function_extension} become:
\begin{align} \label{eq:multi-phase_psi}
\begin{split}
    \| \psi(x_i)-\Psi(x_i)\|_{\mathcal{V}}^2&=\sum_{k=1}^K| \psi_k(x_i)-\Psi_k(x_i)|^2,  \quad 
    \|\Psi\|_{\mathcal{H_K}(\Omega)}^2 = \sum_{k=1}^K  \|\Psi_k\|_{\mathcal{H}_{\mathcal{K}_k}(\Omega)}^2.
\end{split}
\end{align}
Substituting (\ref{eq:multi-phase_psi}) into (\ref{eq:rkhs_function_extension}), the model becomes:
\begin{align*}
    \min_{\Psi\in\mathcal{H_K}(\Omega)} \frac{1}{2m}\sum_{i=1}^m \sum_{k=1}^K| \psi_k(x_i)-\mathcal{S}_\mathbf{x}\Psi_k|^2
    + \frac{\gamma}{2} \sum_{k=1}^K \|\Psi_k\|_{\mathcal{H}_{\mathcal{K}_k}(\Omega)}^2
\end{align*}
or equivalently, we have $K$ minimization problems:
\begin{align*}
    \min_{\Psi_k\in\mathcal{H}_{\mathcal{K}_k}(\Omega)} \frac{1}{2m}\sum_{i=1}^m | \psi_k(x_i)-\mathcal{S}_\mathbf{x}\Psi_k|^2
    + \frac{\gamma}{2} \|\Psi_k\|_{\mathcal{H}_{\mathcal{K}_k}(\Omega)}^2
\end{align*}
sharing the same tuning parameter $\gamma$, where the minimizers are computed by:
\begin{align*}
    \Psi_k^\gamma(x)=\sum_{i=1}^m\mathcal{K}_k(x,x_i)a_i^k
\end{align*}
and $a_i^k$ satisfies:
\begin{align*}
    \sum_{j=1}^m\mathcal{K}_k(x_i,x_j)a_j^k+m\gamma a_i^k=\psi_k(x_i) .
\end{align*}

\subsection{Projecting the Extended Function into the Simplex}

We note that the extended function $\Psi^\gamma$ lies in the RKHS $\mathcal{H_K}(\Omega)$, and thus $\Psi^\gamma(x)$ does not necessarily belong to $[0,1]^K$. Therefore, a projection onto the simplex is required to construct the fuzzy membership function $\mathbf{u}$ used in (\ref{eq:main_model}).  
For each $x\in\Omega$, we solve the pointwise Euclidean projection:
\begin{align}\label{eq:simplex_proj}
    \min_{\mathbf{u}\in\tilde{\Delta}_U^K} \frac{1}{2}\|\mathbf{u}(x)-\Psi^\gamma(x)\|_2^2 .
\end{align}
This problem is convex and admits a unique solution, and the solution is strongly characterized by $\Psi^\gamma$. In a modeling sense, we can say $\mathbf{u}$ is label-based and characterized by the reproducing kernel $\mathcal{K}$.

Follows from \cite{duchi2008efficient}, the solution is given by the following lemmas.
\begin{lem} \label{lem:proj1}
    Let $\mathbf{u}^*=(u_1^*,\ldots, u_K^*)=\arg\min_{\mathbf{u}\in\tilde{\Delta}_U} \frac{1}{2}\|\mathbf{u}(x)-\Psi^\gamma(x)\|^2$. Denote $u_{(1)}, \ldots, u_{(K)}$ and $\Psi_{(1)}^\gamma, \ldots, \Psi_{(K)}^\gamma$ be the vector obtained by sorting $u_1,\ldots,u_K$ and $\Psi_1^\gamma, \ldots, \Psi_K^\gamma$ in a descending order, respectively. Assume we are given $\rho=\max\left\{j\in[K]:u_{(j)}^*>0\right\}$. Then
    \begin{align} \label{eq:proj_lemma1_u}
        u_k^*=\max\{\Psi_k^\gamma-\xi,\ 0\},
    \end{align}
    where:
    \begin{align*}
        \xi=\frac{1}{\rho}\left(\sum_{q=1}^\rho \Psi_{(q)}^\gamma-1\right) .
    \end{align*}
\end{lem}

\begin{lem}\label{lem:proj2}
With $\rho$ defined above, we have
\begin{align}\label{eq:proj_lemma2_rho}
\rho=\max\left\{j\in[K]:
\Psi_{(j)}^\gamma-\frac{1}{j}\left(\sum_{q=1}^{j}\Psi_{(q)}^\gamma-1\right)>0
\right\}.
\end{align}
\end{lem}

Hence, $\mathbf{u}$ can be computed by Algorithm~\ref{alg:u_multi}, where we loop through the indices to search for $\rho$. The loop starts with $j=2$, because $\tau=1$ for $j=1$. The summation $\sum_{q=1}^j\Psi_{(q)}^\gamma-1$ is computed by accumulating $\eta$. Along with the sorting algorithm, this gives the projection algorithm a computational complexity $\mathcal{O}(K\log(K))$.  
The function-extension step is dominated by solving the linear systems for $(a_1^k,\ldots,a_m^k)$, with worst-case cost $\mathcal{O}(Km^3)$. Since labels are sparse, $m$ is typically small in practice. Implementation details are provided in Section~\ref{sec:u_implementation}.

\begin{algorithm}[!ht]
    \caption{Computing the label-based fuzzy membership function}\label{alg:u_multi}
    \SetKwInOut{Input}{Input}
    \SetKw{Return}{Return:}
    \SetKwComment{Comment}{/* }{ */}
    \Input{Image $I$; The labeled set $\{(x_i, \psi(x_i))\}_{i=1}^m$; Kernel $\mathcal{K}=\mathrm{Diag}(\mathcal{K}_1,\ldots, \mathcal{K}_K)$, tuning parameter $\gamma$;}
    \For{$k=1,2,\ldots, K$}{
        Solve $\sum_{j=1}^m\mathcal{K}_k(x_i,x_j)a_j^k+m\gamma a_i^k=\psi_k(x_i)$ for $(a_1^k,\ldots, a_m^k)$\\
        $\Psi_k^\gamma(x)=\sum_{i=1}^m\mathcal{K}_k(x,x_i)a_i^k$
        \Comment*[r]{The extended function in $\mathcal{H_K}(\Omega)$}
    }
    Sort $(\Psi_1^\gamma, \ldots, \Psi_K^\gamma)$ into $(\Psi_{(1)}^\gamma, \ldots, \Psi_{(K)}^\gamma)$: $\Psi_{(1)}^\gamma \geq \ldots\geq \Psi_{(K)}^\gamma$;\\
    Set $\eta=\Psi_{(1)}^\gamma-1$;\\
    \For{$j=2,\ldots, K$}{
    $\eta=\eta+\Psi_{(j)}^\gamma$; \\
    $\tau=\Psi_{(j)}^\gamma-\frac{1}{j}\eta$ \\
    Stop when $\tau\leq0$. Set $\rho=j-1$, where $j$ is the index that stops. If $j=K$ and $\tau>0$, set $\rho=K$.
    }
    $\xi=\frac{1}{\rho}\left(\sum_{q=1}^\rho \Psi_{(q)}^\gamma-1\right)$
    \Comment*[r]{The projection criterion} 
    \For{$k=1,2,\ldots, K$}{
    $u_k=\max\{\Psi_k^\gamma-\xi,\ 0\}$
    \Comment*[r]{Thresholding}
    }
    \Return{$\mathbf{u}=(u_1,\ldots,u_K)$}
\end{algorithm} 

Our construction of $\mathbf{u}$ differs from \cite{kang2014supervised}.  
In \cite{kang2014supervised}, $\Psi$ is treated as the least-squares projection of a fuzzy membership function onto $\mathcal{H_K}(\Omega)$, which requires computing a Moore--Penrose inverse of $\mathcal{K}_k(x,x_i)$ and incurs $\mathcal{O}(Km^3)$ complexity.  
Here, we project $\Psi^\gamma$ directly onto the simplex $\tilde{\Delta}_U^K$, yielding a natural formulation and an additional projection cost of only $\mathcal{O}(K\log K)$.

\subsection{Special Case: Binary Segmentation} 

We remark this special case as the computation of the label-based fuzzy membership function can be further simplified. 

First, to represent binary phases, we only need a single indicator for the label function. i.e. $K=1$. We will denote it as $\psi_0$ to distinguish between the multiphase version. Then the extended function can be found by \eqref{eq:sol_Psi_forward} and \eqref{eq:sol_Psi_backward} directly. i.e. we have:
\begin{align} \label{eq:sol_Psi_binary}
\begin{split}
    \Psi_0^\gamma(x)=\sum_{i=1}^m\mathcal{K}_0(x,x_i)a_i^0, \qquad \sum_{j=1}^m\mathcal{K}_0(x_i,x_j)a_j^0+m\gamma a_i^0=\psi_0(x_i). 
\end{split}
\end{align}

At the projection step, we will be finding a $u_0:\Omega\rightarrow[0,1]$ s.t.
\begin{align*}
    \min_{u_0} \frac{1}{2}|u_0(x)-\Psi_0^\gamma(x)|^2 .
\end{align*}
The solution is, trivially:
\begin{align*}
     u_0^*(x) = 
\begin{cases}
    1, & {\mathrm{if}}\ \Psi_0^\gamma(x)\geq 1 \\
    \Psi_0^\gamma(x), & {\mathrm{if}}\ \Psi_0^\gamma(x)\in (0,1) \\
    0, & {\mathrm{otherwise}}
\end{cases}
\end{align*}
which can also be written as the thresholding form:
\begin{align*}
    u_0^*=\max\{\Psi_0^\gamma-\xi, 0\} ,
\end{align*}
where
\begin{align*}
    \xi=\max\{\Psi_0^\gamma-1,0\} .
\end{align*}
This means we just need to truncated all $\Psi_0^\gamma(x)\notin[0,1]$ to the boundary $\{0,1\}$ to obtain $u_0$. The computational complexity is $\mathcal{O}(1)$, which is optimal.

\begin{algorithm}[!ht]
    \caption{Computing the label-based fuzzy membership function (binary)}\label{alg:u_binary}
    \SetKwInOut{Input}{Input}
    \SetKw{Return}{Return:}
    \SetKwComment{Comment}{/* }{ */}
    \Input{Image $I$; The labeled set $\{(x_i, \psi_0(x_i))\}_{i=1}^m$; Kernel $\mathcal{K}_0$, tuning parameter $\gamma$;}
    Solve $\sum_{j=1}^m\mathcal{K}_0(x_i,x_j)a_j^0+m\gamma a_i^0=\psi_0(x_i)$ for $(a_1^0,\ldots, a_m^0)$\\
    $\Psi_0^\gamma(x)=\sum_{i=1}^m\mathcal{K}_0(x,x_i)a_i^0$
    \Comment*[r]{The extended function in $\mathcal{H_K}(\Omega)$}
    $\xi=\max\{\Psi_0^\gamma-1,0\}$
    \Comment*[r]{The projection criterion} 
    $u_0=\max\{\Psi_0^\gamma-\xi,\ 0\}$
    \Comment*[r]{Thresholding}
    \Return{$u_0$}
\end{algorithm} 
Hence, we have Algorithm \ref{alg:u_binary} for binary partition, and obtain the segmentation by solving:
\begin{align}\label{eq:main_model_binary}
    \min_{v_0\in BV(\Omega),v_0(x)\in[0,1]} \int_\Omega \left(1-2u_0(x)\right)v_0(x)dx + \lambda \int_\Omega \left(1-v_0(x)\right)(G_\sigma\ast v_0)(x)dx .
\end{align}

\section{Experimental Results for Single Image Segmentation}\label{Sec:Single_Image_Segmentation}

In this section, we present implementation details and practical performance of the proposed model in the single-image setting. We first study the pre-segmentation obtained by thresholding the fuzzy membership function $\mathbf u$ in Section~\ref{sec:rkhs_thre_effect}. We found that, on the one hand, $\mathbf u$ effectively propagates sparse labels and can produce high-quality pre-segmentations for relatively simple images. On the other hand, for semantically complex images, these pre-segmentations may miss edge information and become sensitive to parameter choices, especially in large inhomogeneous regions. Nevertheless, this issue can be well addressed with the additional perimeter regularization through the smooth approximation in our standard full models~\eqref{eq:main_model} or \eqref{eq:main_model_binary}. We, therefore, present the experimental results in Section~\ref{sec:td_results}. The final output $\mathbf v$ consistently improves over $\mathbf u$ and performs robustly in challenging scenarios, including water splashes, illumination bias, and noise. Noteworthy, we solve \eqref{eq:main_model} and \eqref{eq:main_model_binary} using the threshold dynamics (TD) \cite{Liu2011,wang2017efficient,liu2022deep} algorithm (see Appendix~\ref{sec:td_implementation}). For simplicity, we will refer to the additional smooth perimeter regularization as TD regularization in the rest of the paper. 

% Results of RKHS prob and RKHS thre
\subsection{Effect on the RKHS Pre-segmentation} \label{sec:rkhs_thre_effect}

\subsubsection{Implementation details} \label{sec:u_implementation}

We begin with discretization. Let $\hat I$ be a digital image on a discrete domain $\hat\Omega$ of size $h\times w$. Since pre-segmentation only requires a coarse partition, we downsample $\hat I$ (bilinear interpolation) to reduce cost. We use a scale
\[
c_0 := \max\left\{1,\left\lceil \frac{\max(h,w)}{150}\right\rceil\right\},
\]
which yields $\hat I_0$ on a coarser domain $\hat\Omega_0$ of size $h_0\times w_0=(h/c_0)\times (w/c_0)$, with at most about $150^2$ pixels. For typical images of size $300\times 400$, this reduces the number of pixels to roughly one quarter. We then apply the discretized versions of Algorithms~\ref{alg:u_multi} and~\ref{alg:u_binary}: all operators are first computed on $\hat\Omega_0$ (with zero Dirichlet boundary condition) to obtain $\Psi^\gamma$, then upsampled back to size $h\times w$, followed by simplex projection to obtain $\hat{\mathbf u}$.

We use the reproducing kernels from \cite{ha2010image,kang2014supervised}:
\begin{align}\label{eq:kernel}
    \mathcal{K}_k(x,y)=\exp{\left(-\frac{\|\hat I_0(\mathbf{x})-\hat I_0(\mathbf{y})\|_2^2}{2\sigma_{I,k}(2R_k+1)^2}\right)}\exp{\left(-\frac{|x-y|^2}{\sigma_{s,k}(h_0^2+w_0^2)}\right)}
\end{align}
for $k=0,1,\ldots,K$, where $x,y\in\hat\Omega_0$, $\mathbf x$ and $\mathbf y$ are $(2R_k+1)\times(2R_k+1)$ patches centered at $x$ and $y$, respectively. The kernel combines patch-intensity similarity and spatial closeness, controlled by $\sigma_{I,k}$ and $\sigma_{s,k}$.

With these kernels, the extended function is computed by solving
\begin{align}\label{eq:sol_Psi_backward_discrete}
    (\hat{\mathcal{K}}_k+m\gamma\hat{\mathcal{I}}) \mathbf{a}^k=\hat{\psi}_k
\end{align}
followed by
\begin{align}\label{eq:sol_Psi_forward_discrete}
    \hat{\Psi}_k=\check{\mathcal{K}}_k\mathbf{a}^k .
\end{align}
Here $\hat{\mathcal I}$ is the identity matrix, $\hat\psi_k$ is the discretized label vector, and $\mathbf a^k=(a_1^k,\ldots,a_m^k)^\top$. The matrices $\hat{\mathcal K}_k\in(0,1]^{m\times m}$ and $\check{\mathcal K}_k\in(0,1]^{h_0w_0\times m}$ contain kernel values $\mathcal K_k(x_i,x_j)$ ($x_i,x_j\in D_0$) and $\mathcal K_k(x,x_i)$ ($x\in\Omega_0$, $x_i\in D_0$), respectively. Since $\hat{\mathcal K}_k$ is symmetric positive definite, it is invertible; we set $\gamma=0$ and solve \eqref{eq:sol_Psi_backward_discrete} using conjugate gradient.

\subsubsection{Real Image Examples}

% Content:
% \begin{enumerate}
%     \item (Good) Perfect example
%     \item (Good) Multiphase, control of kernel parameter in different phases for different textures
%     \item (Bad) Sensitive to parameter
%     \item (Bad) Cannot deal with the Big Noise region
% \end{enumerate}

Figure~\ref{fig:dolphin} shows a dolphin in an underwater scene. The target is blurred and the intensity distribution is fairly homogeneous, which makes phase separation difficult. Our method first downsamples the raw image from size $300\times 400$ to $112\times 150$ and then applies Algorithm~\ref{alg:u_multi} to compute $\Psi^\gamma$. The result is upsampled back to the original resolution and projected onto the simplex. We use $\sigma_{I,0}=0.05$, $\sigma_{s,0}=3$, and $R_0=3$ as kernel parameters to balance patch-intensity similarity and spatial closeness. The resulting label-based fuzzy membership function $\mathbf u$ clearly separates the target from the background. The thresholded result $\mathbbm{1}_{\{x:\hat{u}_0(x)>0.5\}}$, shown in subfigure~(d), is already a satisfactory segmentation.

%=========Dolphin==============
\begin{figure}[!ht]
    \centering
    \begin{subfigure}[b]{0.22\textwidth}
        \centering
        \includegraphics[width=\textwidth]{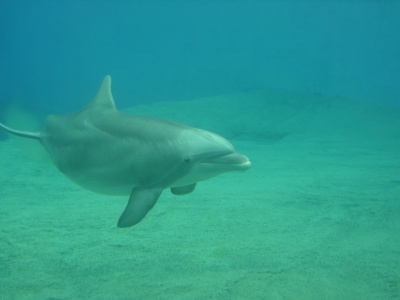} \\
        \caption{Image}
    \end{subfigure}
    \begin{subfigure}[b]{0.22\textwidth}
        \centering
        \includegraphics[width=\textwidth]{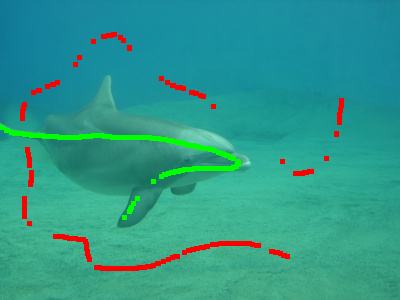} \\
        \caption{Scribbles}
    \end{subfigure}
    \begin{subfigure}[b]{0.22\textwidth}
        \centering
        \includegraphics[width=\textwidth]{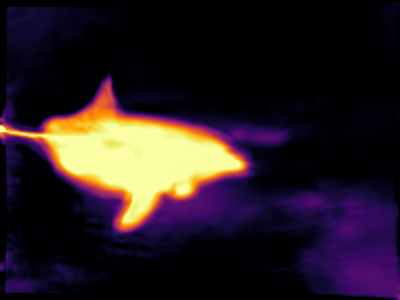} \\
        \caption{Proj}
    \end{subfigure}
    \begin{subfigure}[b]{0.22\textwidth}
        \centering
        \includegraphics[width=\textwidth]{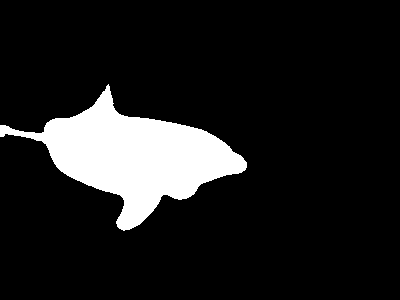} \\
        \caption{Thre}
    \end{subfigure}
    % \begin{subfigure}[b]{0.18\textwidth}
    %     \centering
        % \includegraphics[width=\textwidth]{Figures/iterative_method/window/0630_PottsTD_ea10sig3B0_iter14.png} \\
    %     \includegraphics[width=\textwidth]{Figures/Sal500/selected/complicated/0797_PottsTD_05.png} \\
    %     \caption{TD}
    % \end{subfigure}
    
    \caption{$u_0$ is computed in a domain smaller than $150\times150$, and sampled back to the original scale by bicubic interpolation. (a) Image: The raw image. (b) Scribbles: The drawn scribbles. (c) Proj: The fuzzy membership function $\hat{u}_0$. (d) Thre: The $\hat{u}_0$ threshold, $\mathbbm{1}_{\{x:\hat{u}_0(x)>0.5\}}$}
    
    \label{fig:dolphin}
\end{figure}

\paragraph{Effect of parameter choice}

We further illustrate the effect of parameter selection in Figure~\ref{fig:smooth stable2}. The example is a short-axis cardiac MRI slice from \cite{bernard2018deep}. The bright circular region in the ring-shaped object is the target, namely the left ventricle. We fix $\sigma_{I,0}=0.01$ and $R_0=2$ based on the intensity distribution, and vary $\sigma_{s,0}$. In column (c), a larger $\sigma_{s,0}$ makes the spatial kernel
\[
\exp\!\left(-\frac{|x-y|^2}{\sigma_{s,0}(h_0^2+w_0^2)}\right)
\]
close to $1$ for most pixels, so $\Psi_0^\gamma$ and hence $u_0$ depend mainly on patch intensity. In column (a), a smaller $\sigma_{s,0}$ makes spatial proximity dominant, so only pixels near the scribbles are assigned to the same phase. Pixels far from all scribbles are classified as background. Column (b) gives an intermediate choice and produces a more balanced and noise-free pre-segmentation. We will use a similar strategy in choosing the kernel parameters for other examples. 

\begin{figure*}[!htb]
    \centering
    \addtolength{\tabcolsep}{-5pt}
    \begin{tabular}[c]{cccccc}
    
    \begin{subfigure}[h]{0.17\textwidth}
        \centering
        \includegraphics[width=\linewidth]{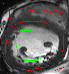}
    \end{subfigure}
    \hfill
    &
    \hfill
    &
    \hfill
    &
        \begin{subfigure}[h]{0.17\textwidth}
            \centering
            \includegraphics[width=\textwidth]{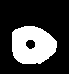}
            \includegraphics[width=\textwidth]{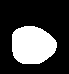}
            \caption{$\sigma_{s,0}=0.01$}
        \end{subfigure}
    &
        \begin{subfigure}[h]{0.17\textwidth}
            \centering
            \includegraphics[width=\textwidth]{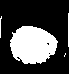}
            \includegraphics[width=\textwidth]{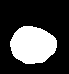}
            \caption{$\sigma_{s,0}=0.1$}
        \end{subfigure}
    &
        \begin{subfigure}[h]{0.17\textwidth}
            \centering
            \includegraphics[width=\textwidth]{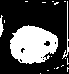}
            \includegraphics[width=\textwidth]{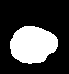}
            \caption{$\sigma_{s,0}=1$}
        \end{subfigure}
    
    \end{tabular}
    \addtolength{\tabcolsep}{5pt}
    \caption{Smoothing effect of the perimeter term on a cardiac MRI image under different values of $\sigma_{s,0}$. Left: input. Top: thresholded $\hat{u}_0$. Bottom: $\hat{v}_0$ obtained by Algorithm~\ref{alg:TD} with $\lambda=10$.}
    \label{fig:smooth stable2}
\end{figure*}

\paragraph{Multiphase}
Parameter selection is more delicate in the multiphase setting, since changing one $\sigma_{s,k}$ can alter the relative memberships across all phases after simplex projection. Our practical strategy is to fix the common kernel parameters across $\mathcal{K}_k$ and tune only the phase-specific spatial scales. Figure~\ref{fig:butterfly} shows a three-phase example (sky/background, butterfly, flowers). We first downsample the image from $300\times 400$ to $112\times 150$, compute $\Psi^\gamma$ on the coarse grid, and then upsample to obtain $\mathbf{u}$. We fix the patch-intensity variance and patch radius to $\sigma_{I,k}=0.01$ and $R_k=3$ for all phases, and tune only $\sigma_{s,k}$. Since the sky is nearly homogeneous, we set $\sigma_{s,1}=\infty$. The projected map in Figure~\ref{fig:butterfly}(e) is already close to the final partition; the remaining pepper noise mainly appears in the flower region and is removed by the final Potts regularization.

\begin{figure}[!ht]
    \centering
    \begin{subfigure}[b]{0.16\textwidth}
        \centering
        \includegraphics[width=\textwidth]{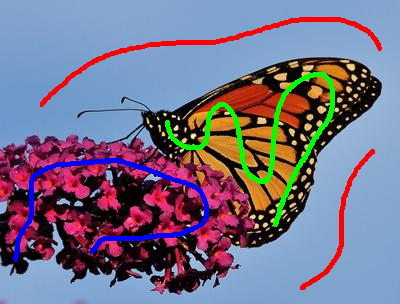} \\
        \caption{Scribbles}
    \end{subfigure}
    \begin{subfigure}[b]{0.16\textwidth}
        \centering
        \includegraphics[width=\textwidth]{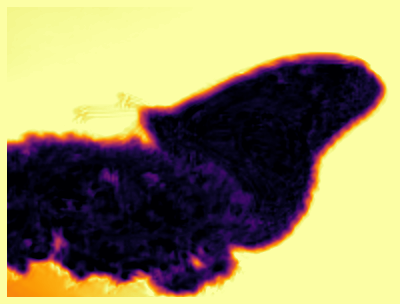} \\
        \caption{Proj-BG}
    \end{subfigure}
        \begin{subfigure}[b]{0.16\textwidth}
        \centering
        \includegraphics[width=\textwidth]{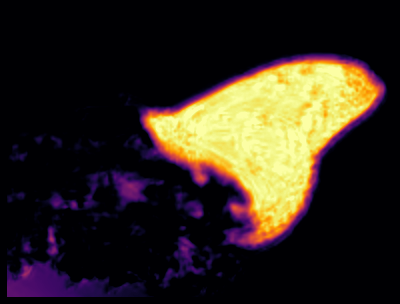} \\
        \caption{Proj-FG1}
    \end{subfigure}
        \begin{subfigure}[b]{0.16\textwidth}
        \centering
        \includegraphics[width=\textwidth]{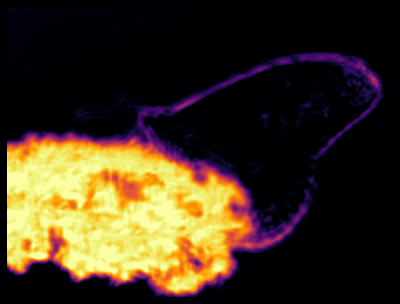} \\
        \caption{Proj-FG2}
    \end{subfigure}
    \begin{subfigure}[b]{0.16\textwidth}
        \centering
        \includegraphics[width=\textwidth]{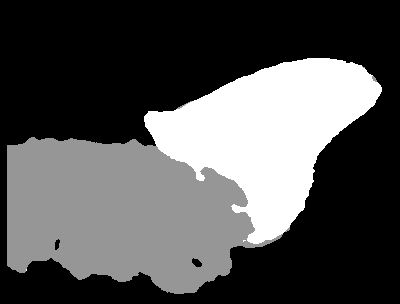} \\
        \caption{Thre}
    \end{subfigure}
    \begin{subfigure}[b]{0.16\textwidth}
        \centering
        \includegraphics[width=\textwidth]{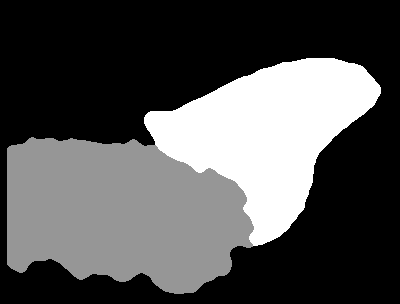} \\
        \caption{Iter}
    \end{subfigure}

    \caption{A result of the multiphase model. We set $\sigma_{I,1}=\sigma_{I,2}=\sigma_{I,3}=0.01$, $\sigma_{s,1}=\infty$, $\sigma_{s,2}=\sigma_{s,3}=1$, and $R_1=R_2=R_3=3$. (a) Scribbles: raw image with scribbles. (b)--(d) Proj: $\hat{u}_k$ for $k=1,2,3$. (e) Thre: $\mathbbm{1}_{\{x:u_k(x)>\max_{k'\neq k}u_{k'}(x)\}}$. (f) Iter: $\hat{\mathbf{v}}$.}
    \label{fig:butterfly}
\end{figure}

\paragraph{An ineffective example}

Nevertheless, some real images remain difficult: no preferable partition is obtained, regardless of kernel-parameter choices. Figure~\ref{fig:cactus} shows such a case. The small white spines have a strong intensity contrast with the green stem, although both belong to the foreground. As a result, patch-intensity similarity tends to classify these spines as background. In contrast, spatial closeness provides only limited correction because the two scribbles are close to each other and near the boundary. To propagate labels to distant foreground regions, $\sigma_{s,0}$ must be set relatively large. Figure~\ref{fig:cactus}(c) presents the best thresholded $\hat{u}_0$ we obtained with $\sigma_{I,0}=0.05$, $\sigma_{s,0}=60$, and $R_0=3$. Even in this setting, the spines still appear as pepper noise in the pre-segmentation.

\begin{figure}[!ht]
    \centering
    \begin{subfigure}[b]{0.2\textwidth}
        \centering
        \includegraphics[width=\textwidth]{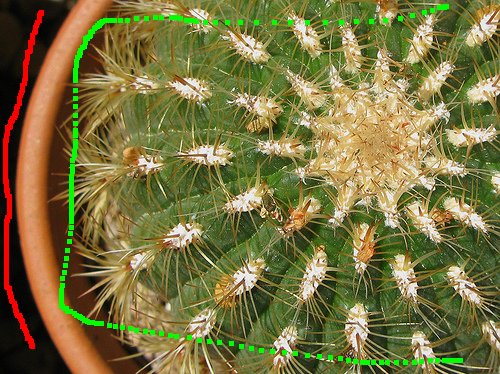} \\
        \caption{Scribbles}
    \end{subfigure}
    \begin{subfigure}[b]{0.2\textwidth}
        \centering
        \includegraphics[width=\textwidth]{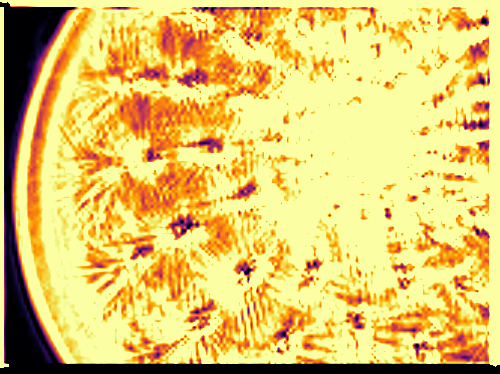} \\
        \caption{Proj}
    \end{subfigure}
    \begin{subfigure}[b]{0.2\textwidth}
        \centering
        \includegraphics[width=\textwidth]{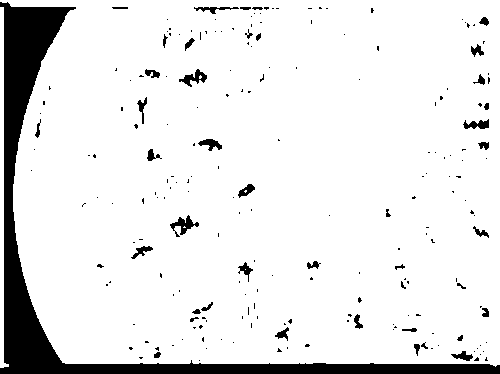} \\
        \caption{Thre}
    \end{subfigure}
        \begin{subfigure}[b]{0.2\textwidth}
        \centering
        \includegraphics[width=\textwidth]{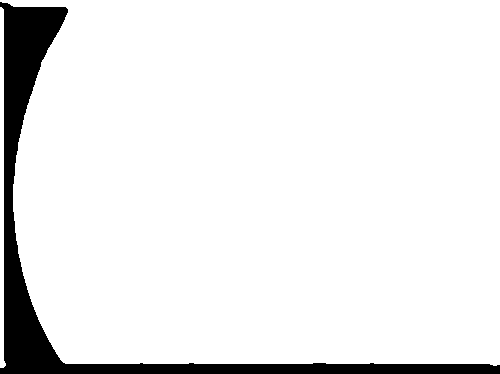} \\
        \caption{Iter}
    \end{subfigure}
    
    \caption{Cactus. The white glochids have sharp intensity differences with the green stem. Our model effectively smooths the pepper noises. (a) Scribbles: The raw image along with drawn scribbles. (b) Proj: $\hat{u}_0$. (c) Thre: $\mathbbm{1}_{\{x:\hat{u}_0(x)>0.5\}}$. (d) Iter: $\hat{v}_0$}
    \label{fig:cactus}
\end{figure}

These examples indicate that the pre-segmentation may not be enough for a perfect partition. Later, in the results of our proposed model (\ref{eq:main_model}) or (\ref{eq:main_model_binary}), we will see that the perimeter term can effectively remove this salt and pepper region. First, we first present the optimization details of our model.

%=================Iterative=====================

\subsection{Experiments Results}\label{sec:td_results}

All experiments in this section were conducted on a laptop with a 13th Gen Intel(R) Core(TM) i5-13500H (2.60\,GHz) CPU and 16\,GB RAM. Unless otherwise specified, we set $\sigma=3$. The implementation is written in Python. We use the \texttt{GaussianBlur} operator in PyTorch \cite{paszke2019pytorch} to approximate Gaussian convolution. This implementation applies kernel-based convolution rather than FFT. Although this approximation of the perimeter term is less accurate than total variation (TV), the results below show that it still provides effective smoothing.

\subsubsection{Effect of Threshold Dynamics Regularization}

Figure~\ref{fig:cactus}(d) shows that the solution $v_0$ from the full model yields a more compact mask. We solve it using Algorithm~\ref{alg:TD} with $\lambda=35$. The threshold dynamics (TD) regularization effectively removes pepper noise. This effect is also observed in the multiphase example in Figure~\ref{fig:butterfly}: small pepper noise in the flower phase is suppressed, while cleaner phases are preserved. A more challenging case is shown in Figure~\ref{fig:rabbit}. Since the rabbit and background have similar intensity distributions, the thresholded $\hat{u}_0$ contains both salt and pepper noise. TD regularization reduces both types of noise in foreground and background regions.

\begin{figure}[!ht]
    \centering
    \begin{subfigure}[b]{0.18\textwidth}
        \centering
        \includegraphics[width=\textwidth]{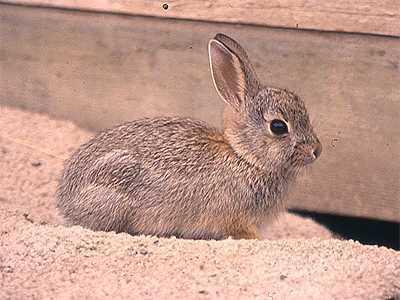} \\
        \caption{Image}
    \end{subfigure}
    \begin{subfigure}[b]{0.18\textwidth}
        \centering
        \includegraphics[width=\textwidth]{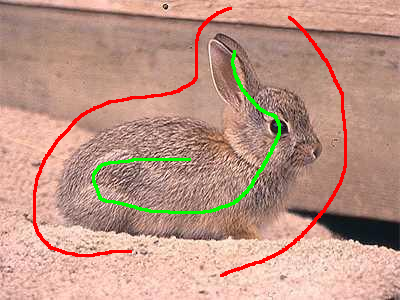} \\
        \caption{Scribbles}
    \end{subfigure}
    \begin{subfigure}[b]{0.18\textwidth}
        \centering
        \includegraphics[width=\textwidth]{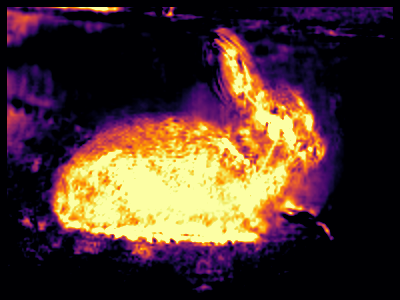} \\
        \caption{Proj}
    \end{subfigure}
    \begin{subfigure}[b]{0.18\textwidth}
        \centering
        \includegraphics[width=\textwidth]{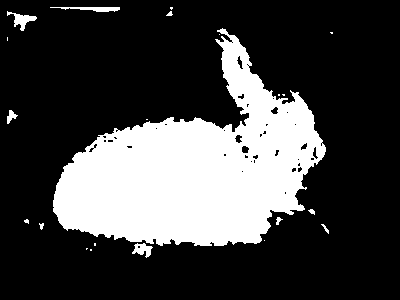} \\
        \caption{Thre}
    \end{subfigure}
        \begin{subfigure}[b]{0.18\textwidth}
        \centering
        \includegraphics[width=\textwidth]{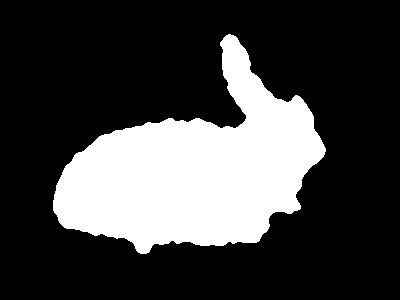} \\
        \caption{Iter}
    \end{subfigure}
    
    \caption{Homogeneous image. Objects and background intensity distribution are similar. Our model is solved under the parameters: $\sigma_{I,0}=0.01$, $\sigma_{s,0}=8$, $R_0=3$, $\lambda=30$}
    
    \label{fig:rabbit}
\end{figure}

The additional term also makes the partition more robust against parameter choices. In figure~\ref{fig:smooth stable2}, we demonstrate the effect of $\hat{u}_0$ threshold under different choices of kernel parameter. The 2nd row images are the corresponding $\hat{v}_0$ obtained. All three $\hat{v}_0$ are compact and noise-free. This suggests that our method is robust under parameter choices with the spatial regularization in the Potts model.

\subsubsection{Noisy Images}

The full model is also robust when images are heavily corrupted by noise. In Figure~\ref{fig:noisy}, we show two skin-lesion images with several types of additive noise, with intensity up to $0.5$, under sparse labels.  
For the first image, only one foreground point is provided; we therefore add the four image corners as background points.  
For the second image, we use a short foreground scribble on the lesion and a sparse circular background scribble.  

We use
$(\sigma_{I,0},\sigma_{s,0},R_0,\lambda)=(0.03,3,2,10)$ for the first image and
$(0.01,1,3,10)$ for the second image, for all noise settings.  
The resulting masks preserve consistent lesion shapes and remain close to the ground truth, showing robustness to different noise types.

%==========Noisy ===================
\begin{figure}[!h]
    \centering
        \begin{tabular}{ccccc}
            \begin{subfigure}[b]{0.18\textwidth}
                \centering
                \includegraphics[width=\textwidth]{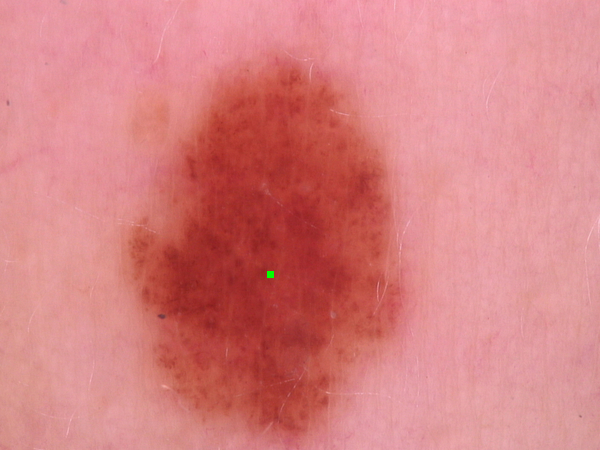} \\
                \includegraphics[width=\textwidth]{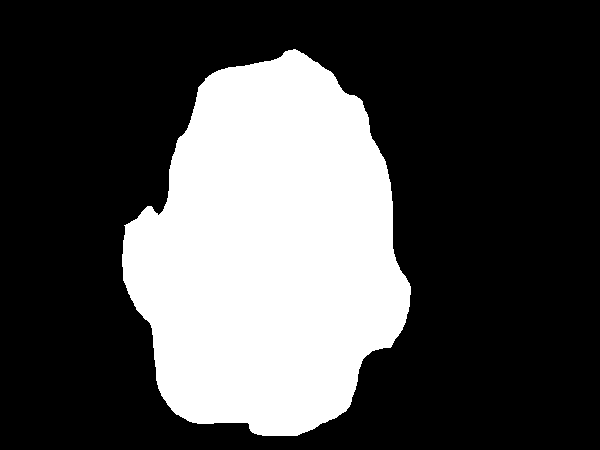}
                %\caption{Input\&CE}
            \end{subfigure}
            &
            \begin{subfigure}[b]{0.18\textwidth}
                \centering
                \includegraphics[width=\textwidth]{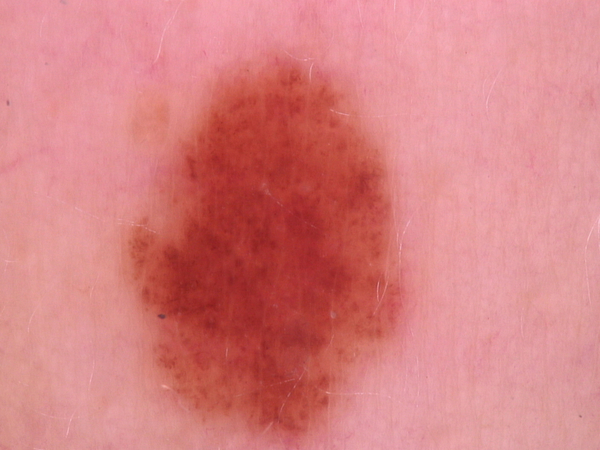} 
                \\
                % \includegraphics[width=\textwidth]{Figures/iterative_method/noisy/ISIC_0024334_thre_rgb0003xy3r2.png}
                % \\
                \includegraphics[width=\textwidth]{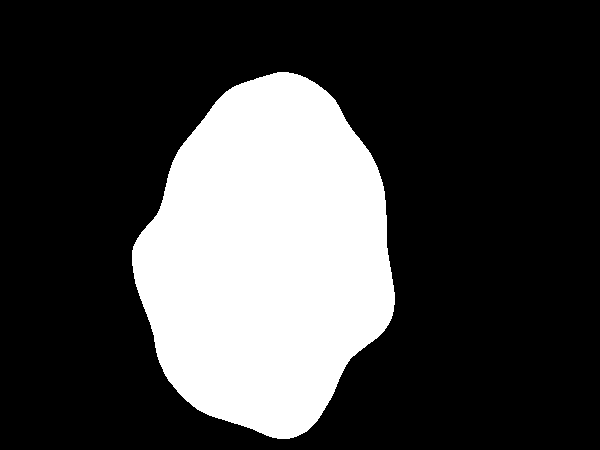}
                %\caption{raw}
            \end{subfigure}
            &
            \begin{subfigure}[b]{0.18\textwidth}
                \centering
                \includegraphics[width=\textwidth]{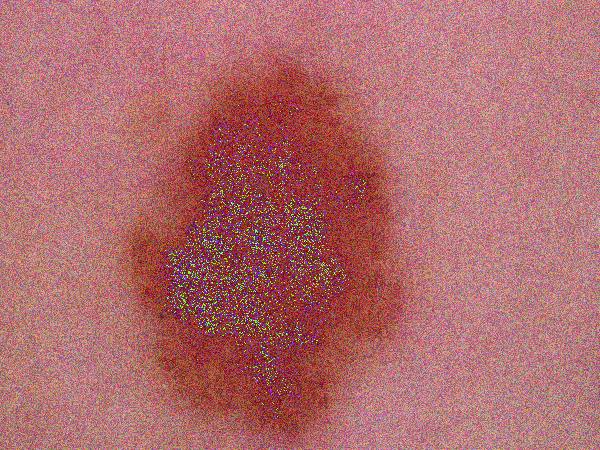} \\
                \includegraphics[width=\textwidth]{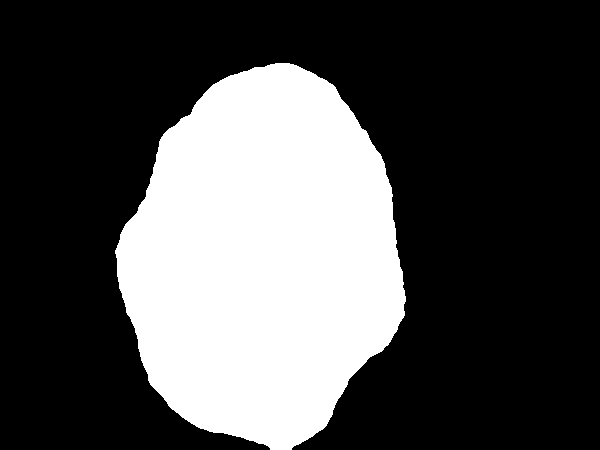}
                %\caption{Uniform}
            \end{subfigure}
            &
            \begin{subfigure}[b]{0.18\textwidth}
                \centering
                \includegraphics[width=\textwidth]{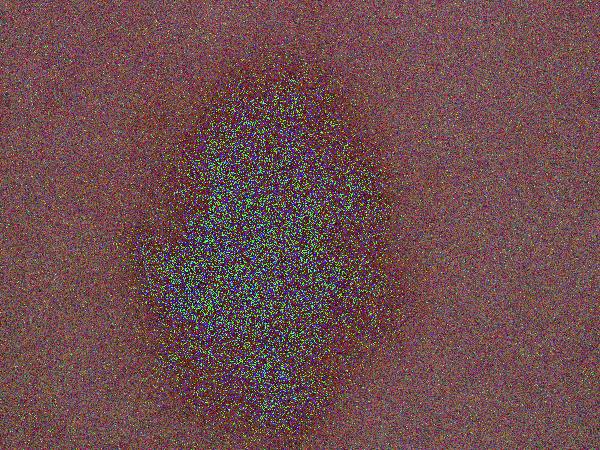} \\
                \includegraphics[width=\textwidth]{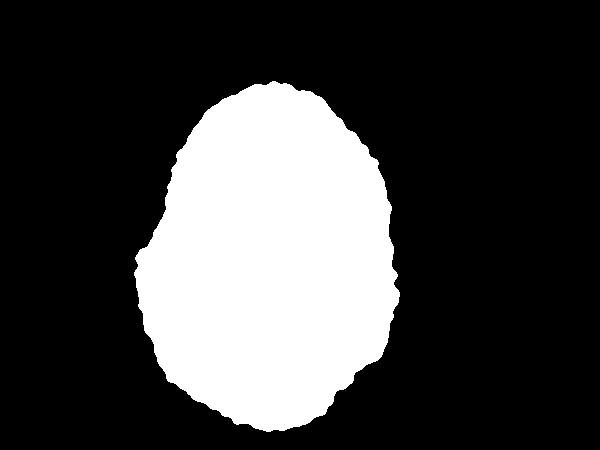}
                %\caption{Gaussian}
            \end{subfigure}
            &
            \begin{subfigure}[b]{0.18\textwidth}
                \centering
                \includegraphics[width=\textwidth]{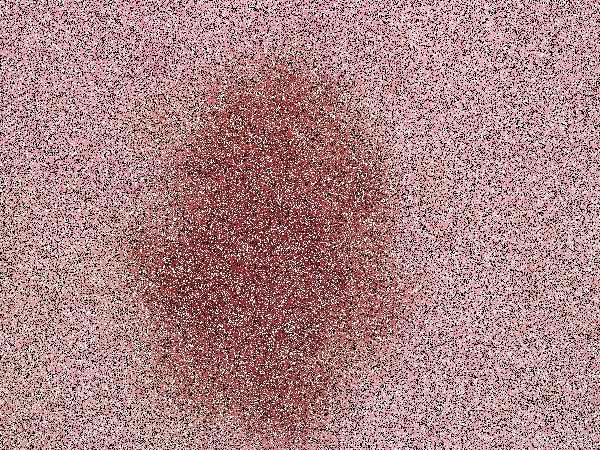} \\
                \includegraphics[width=\textwidth]{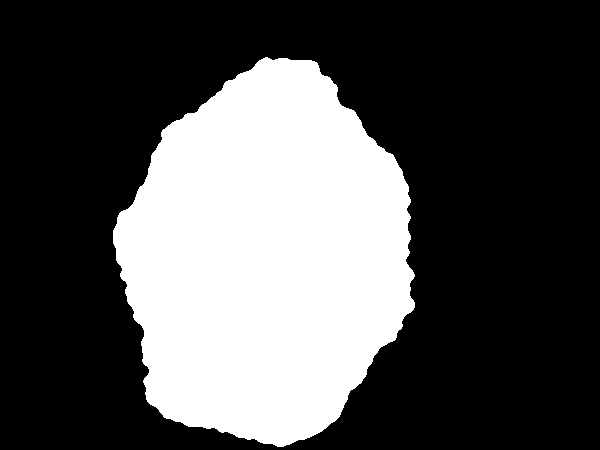}
                %\caption{Salt}
            \end{subfigure}
        \\
            \begin{subfigure}[b]{0.18\textwidth}
                \centering
                \includegraphics[width=\textwidth]{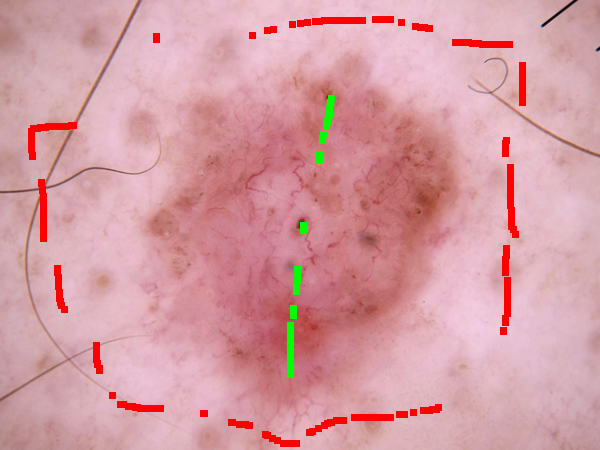} \\
                \includegraphics[width=\textwidth]{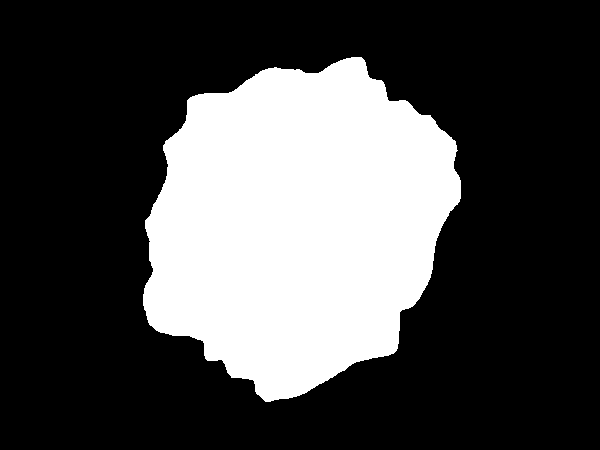}
                \caption{Scribbles\&GT}
            \end{subfigure}
            &
            \begin{subfigure}[b]{0.18\textwidth}
                \centering
                \includegraphics[width=\textwidth]{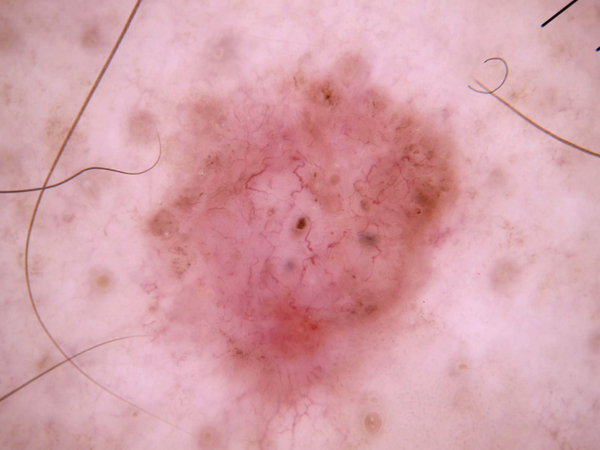} \\
                \includegraphics[width=\textwidth]{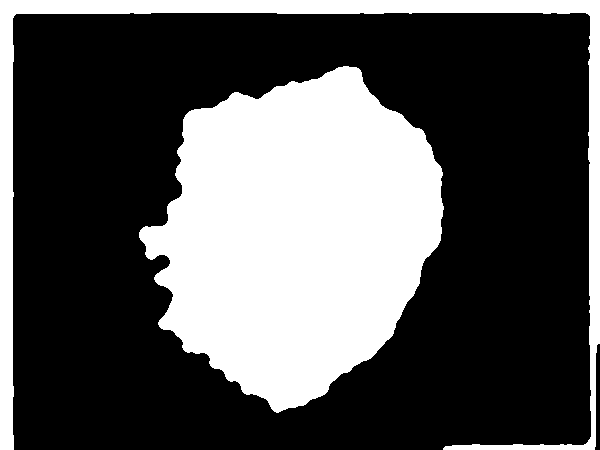}
                \caption{\\No noise}
            \end{subfigure}
            &
            \begin{subfigure}[b]{0.18\textwidth}
                \centering
                \includegraphics[width=\textwidth]{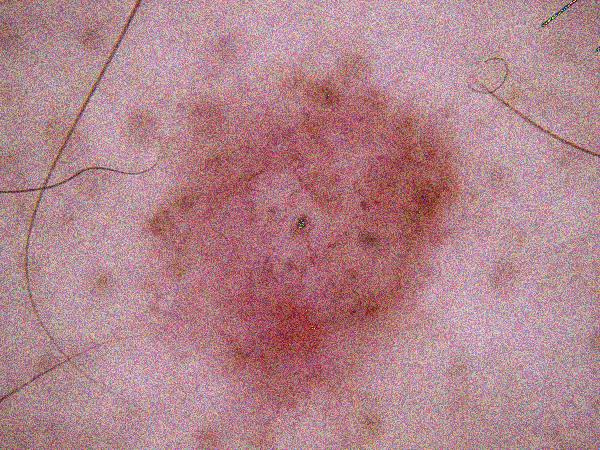} \\
                \includegraphics[width=\textwidth]{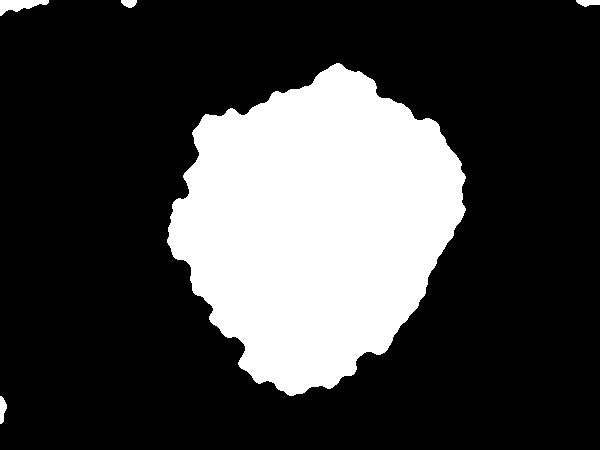}
                \caption{\\Uniform-$0.5$}
            \end{subfigure}
            &
            \begin{subfigure}[b]{0.18\textwidth}
                \centering
                \includegraphics[width=\textwidth]{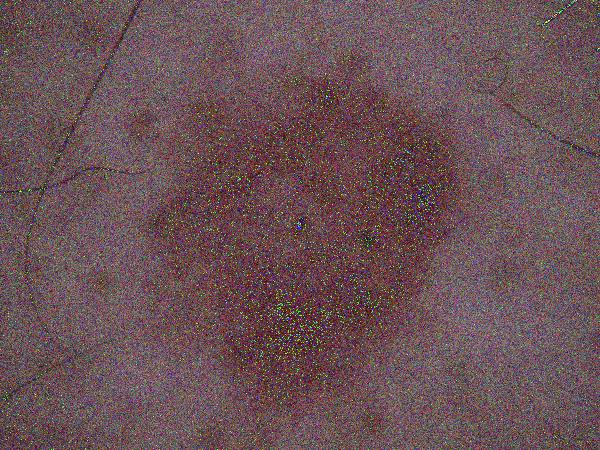} \\
                \includegraphics[width=\textwidth]{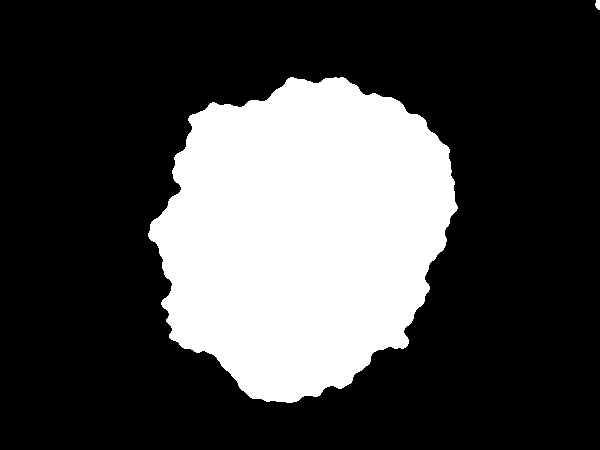}
                \caption{\\Gaussian-$0.25$}
            \end{subfigure}
            &
            \begin{subfigure}[b]{0.18\textwidth}
                \centering
                \includegraphics[width=\textwidth]{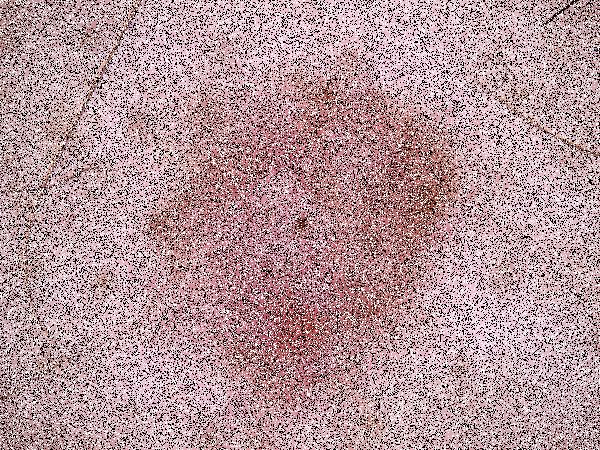} \\
                \includegraphics[width=\textwidth]{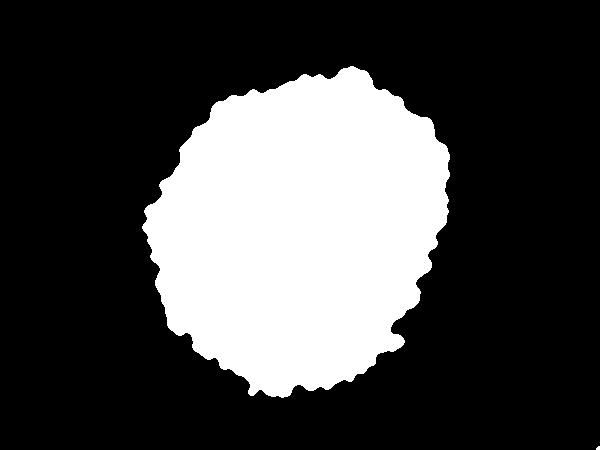}
                \caption{\\Salt-$0.5$}
            \end{subfigure}
        \end{tabular}
    \caption{Noise-corrupted skin-lesion images. Column (a) shows sparse labels and expert ground truth. Columns (b)--(e) show results on clean images, uniform noise ($0.5$), Gaussian noise ($0.25$), and salt-and-pepper noise ($0.5$), respectively.}
    \label{fig:noisy}
\end{figure}

\subsubsection{More Real Image Examples}

We provide additional examples of the standard full model on real images with different characteristics.

\paragraph{Homogeneous image}

We call an image homogeneous when the object and background intensity distributions strongly overlap, even if each class is internally inhomogeneous. Examples are shown in Figures~\ref{fig:dolphin} and \ref{fig:rabbit}; their distributions are provided in Appendix~\ref{sec:homo_img_dist}.  
Classical data-fidelity terms often struggle in this case because they cannot separate overlapping distributions. Background pixels with similar intensity and proximity to the object are easily misclassified, producing artifacts or salt-and-pepper noise.  
Our kernel-based RKHS construction captures subtle distributional differences. In Figure~\ref{fig:dolphin}, both regions are underwater emerald tones, with only slight gray differences on the dolphin; $\hat{u}_0$ already gives a clean segmentation. In Figure~\ref{fig:rabbit}, both object and background are tan, but the rabbit has a richer texture. The resulting $\hat{u}_0$ highlights the rabbit with limited noise, and $\hat{v}_0$ yields a compact final mask.

\paragraph{Illumination bias}

Figure~\ref{fig:low_light} shows examples with illumination bias under low-light conditions.  
In the horse image, partial darkening increases intensity inhomogeneity and blurs object boundaries. We use $\sigma_{I,0}=0.05$, $\sigma_{s,0}=3$, $R_0=3$, and $\lambda=10$. A slightly larger $\sigma_{I,0}$ allows sufficient intensity variation. Although the ears are partially missed, $\hat{u}_0$ still captures most of the dark object region.  
In the second row, the boat image is strongly underexposed. We use $\sigma_{I,0}=0.003$, $\sigma_{s,0}=1$, $R_0=4$, and $\lambda=10$ to exploit subtle intensity differences and obtain a smooth result in $\hat{v}_0$.

\begin{figure}[!h]
    \centering
    \begin{subfigure}[b]{0.18\textwidth}
        \centering
        \includegraphics[width=\textwidth]{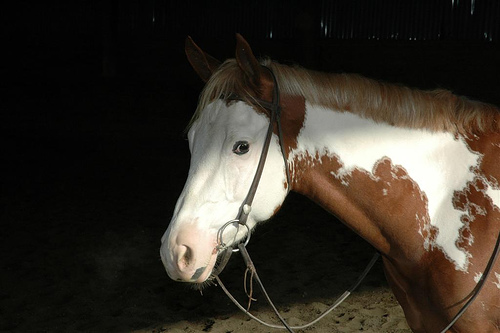} \\
        \includegraphics[width=\textwidth]{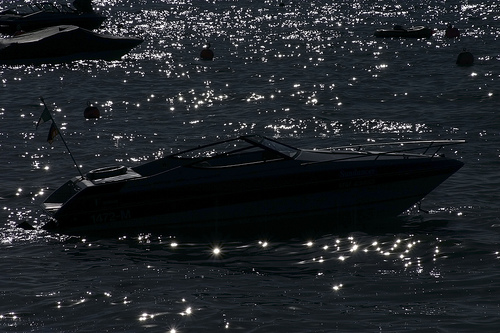} \\
        \caption{Image}
    \end{subfigure}
    \begin{subfigure}[b]{0.18\textwidth}
        \centering
        \includegraphics[width=\textwidth]{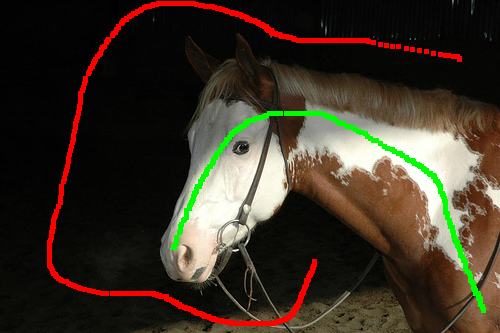} \\
        \includegraphics[width=\textwidth]{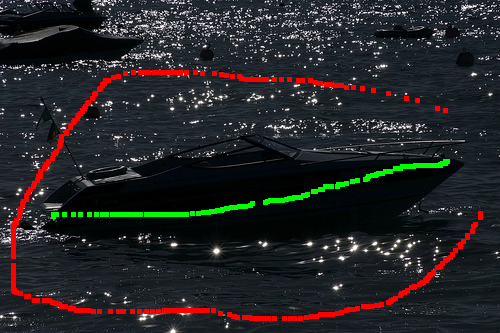} \\
        \caption{Scribbles}
    \end{subfigure}
    \begin{subfigure}[b]{0.18\textwidth}
        \centering
        \includegraphics[width=\textwidth]{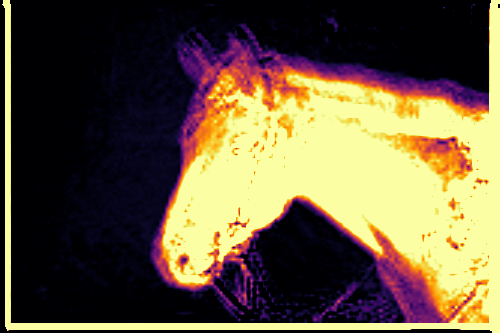} \\
        \includegraphics[width=\textwidth]{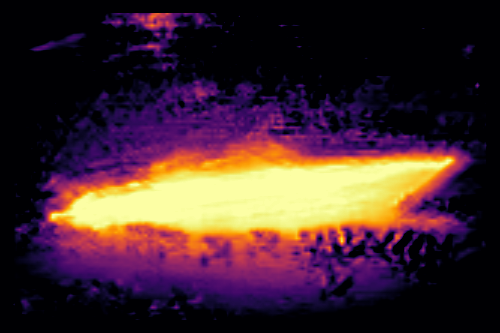} \\
        \caption{Proj}
    \end{subfigure}
    \begin{subfigure}[b]{0.18\textwidth}
        \centering
        \includegraphics[width=\textwidth]{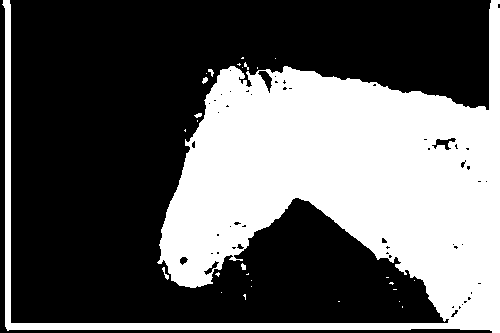} \\
        \includegraphics[width=\textwidth]{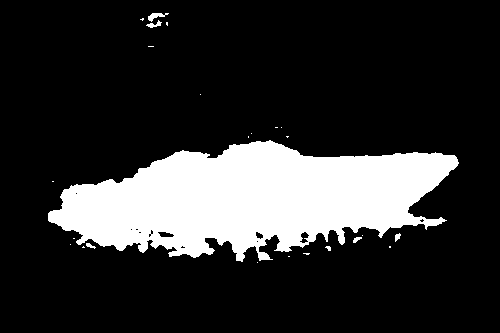} \\
        \caption{Thre}
    \end{subfigure}
    \begin{subfigure}[b]{0.18\textwidth}
        \centering
        \includegraphics[width=\textwidth]{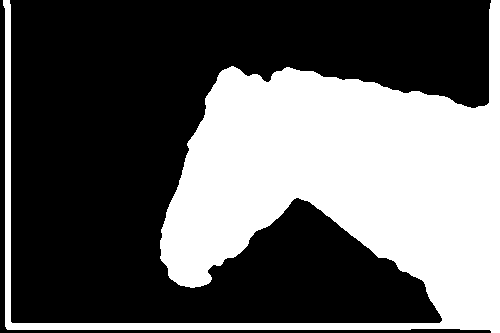} \\
        \includegraphics[width=\textwidth]{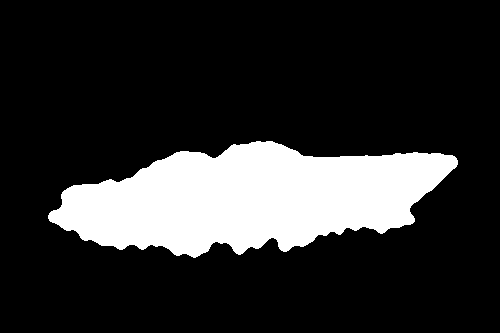} \\
        \caption{Iter}
    \end{subfigure}
    
    \caption{Images exhibit illumination bias due to low light conditions}
    
    \label{fig:low_light}
\end{figure}

\paragraph{Water splashes}

Images captured in rain or water-activity scenes often contain splash artifacts. Boundaries become blurred and intensity distributions become more oscillatory. Figure~\ref{fig:misty} shows two such examples.  
We set $\sigma_{I,0}=0.01$, $R_0=3$, with $\sigma_{s,0}=100$ for the lady image and $\sigma_{s,0}=3$ for the window image. The resulting $\hat{u}_0$ captures the targets with only mild boundary splitting. Remaining salt-and-pepper artifacts are removed in $\hat{v}_0$ using $\lambda=10$.

\begin{figure}[!ht]
    \centering
    \begin{subfigure}[b]{0.2\textwidth}
        \centering
        \includegraphics[width=\textwidth]{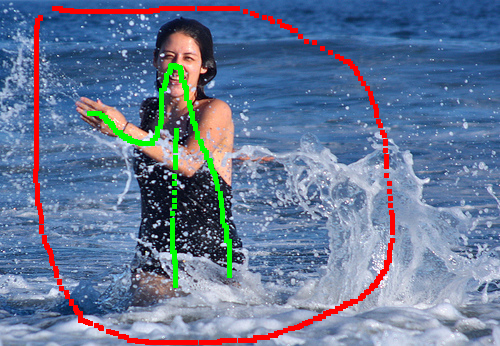} \\
        \includegraphics[width=\textwidth]{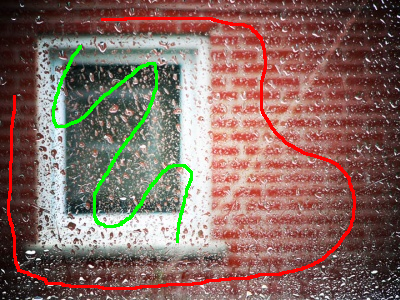} \\
        \caption{Scribbles}
    \end{subfigure}
    \begin{subfigure}[b]{0.2\textwidth}
        \centering
        \includegraphics[width=\textwidth]{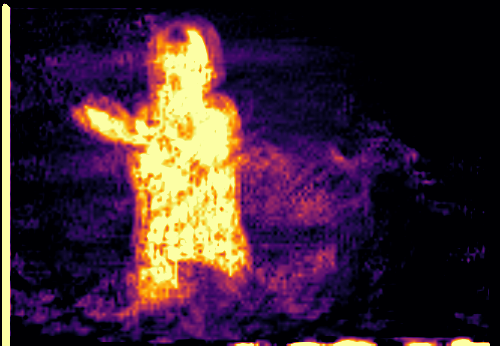} \\
        \includegraphics[width=\textwidth]{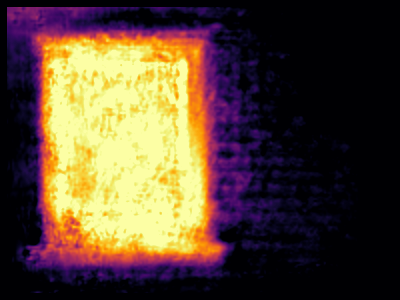} \\
        \caption{Proj}
    \end{subfigure}
    \begin{subfigure}[b]{0.2\textwidth}
        \centering
        \includegraphics[width=\textwidth]{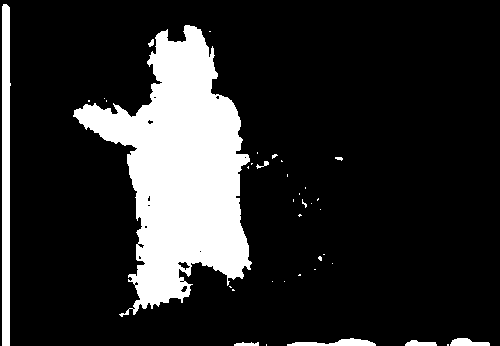} \\
        \includegraphics[width=\textwidth]{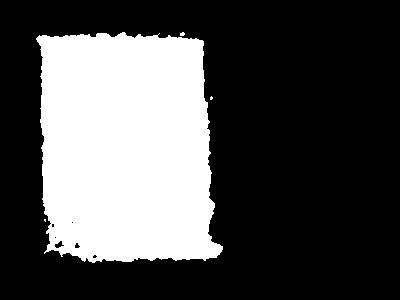} \\
        \caption{Thre}
    \end{subfigure}
        \begin{subfigure}[b]{0.2\textwidth}
        \centering
        \includegraphics[width=\textwidth]{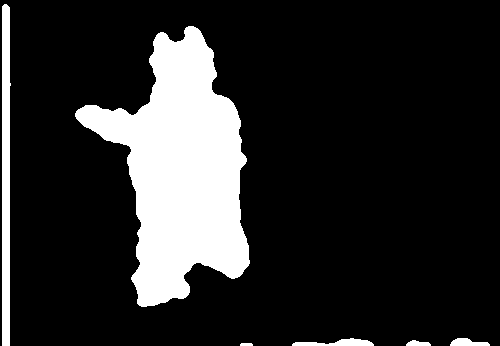} \\
        \includegraphics[width=\textwidth]{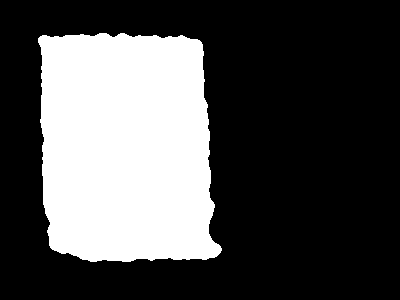} \\
        \caption{Iter}
    \end{subfigure}
    
    \caption{Water splashes blur the images.}
    
    \label{fig:misty}
\end{figure}

\paragraph{Sharp thin objects}

If edge regularization is too strong, thin and sharp object parts may be oversmoothed. With scribble guidance, these components can be preserved by placing foreground labels directly on them.  
Figure~\ref{fig:thin_obj} illustrates this on a bird image: the beak and legs are explicitly labeled as foreground, so $u_0=1$ on those regions. We use $\sigma_{I,0}=0.01$, $\sigma_{s,0}=0.1$, and $R_0=3$ to suppress undesired thin structures. As seen in Figure~\ref{fig:thin_obj}(e), reflection artifacts are removed while the beak and legs are retained.

\begin{figure}[!h]
    \centering
    \begin{subfigure}[b]{0.18\textwidth}
        \centering
        \includegraphics[width=\textwidth]{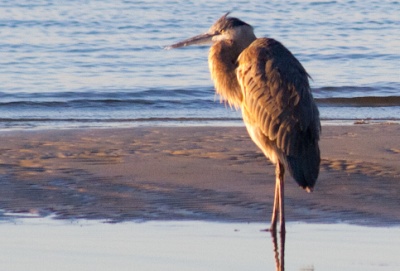} \\
        \caption{Image}
    \end{subfigure}
    \begin{subfigure}[b]{0.18\textwidth}
        \centering
        \includegraphics[width=\textwidth]{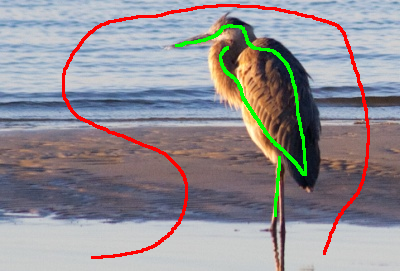} \\
        \caption{Scribbles}
    \end{subfigure}
    \begin{subfigure}[b]{0.18\textwidth}
        \centering
        \includegraphics[width=\textwidth]{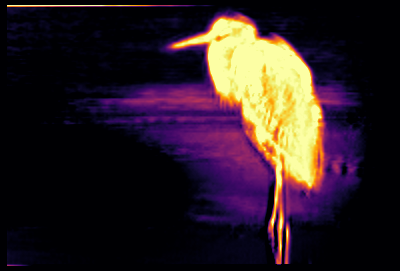} \\
        \caption{Proj}
    \end{subfigure}
    \begin{subfigure}[b]{0.18\textwidth}
        \centering
        \includegraphics[width=\textwidth]{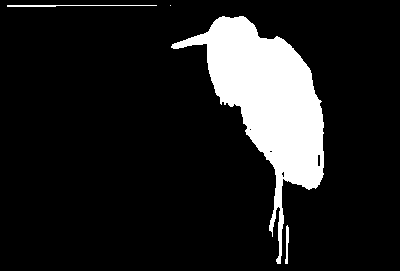} \\
        \caption{Thre}
    \end{subfigure}
    \begin{subfigure}[b]{0.18\textwidth}
        \centering
        \includegraphics[width=\textwidth]{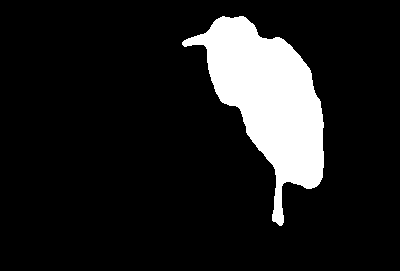} \\
        \caption{Iter}
    \end{subfigure}
    
    \caption{Thin legs of the bird are preserved. $\lambda=10$}
    
    \label{fig:thin_obj}
\end{figure}

\paragraph{Oscillating regions}

Regions with rapidly varying intensity are semantically complex and difficult for simple data-fidelity terms. Figure~\ref{fig:oscillate} shows three such examples.  
We set $\sigma_{s,0}=0.1$ so dense scribbles can guide coarse localization, and use $\sigma_{I,0}=0.01$, $R_0=3$ to balance patch-intensity information. As shown in Figure~\ref{fig:oscillate}(c)--(d), $\hat{u}_0$ captures the target but includes boundary splits. With larger $\lambda$, the final $\hat{v}_0$ becomes more compact (Figure~\ref{fig:oscillate}(e)).

\begin{figure}[!ht]
    \centering
    \begin{subfigure}[b]{0.18\textwidth}
        \centering
        \includegraphics[width=\textwidth]{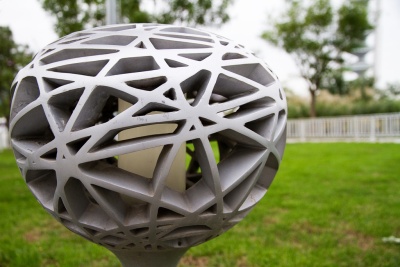} \\
        \includegraphics[width=\textwidth]{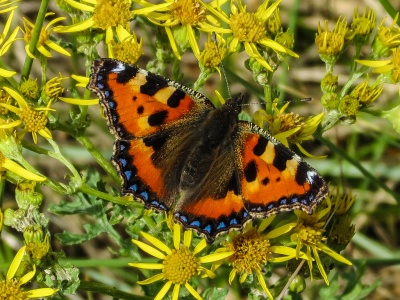} \\
        \includegraphics[width=\textwidth]{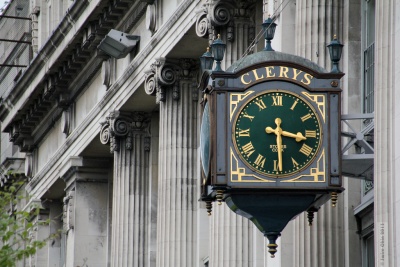} \\
        \caption{Image}
    \end{subfigure}
    \begin{subfigure}[b]{0.18\textwidth}
        \centering
        \includegraphics[width=\textwidth]{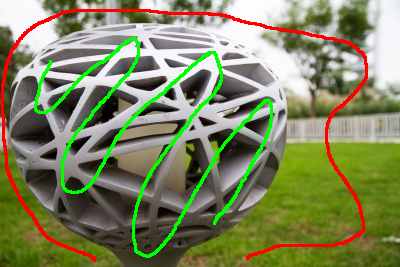} \\
        \includegraphics[width=\textwidth]{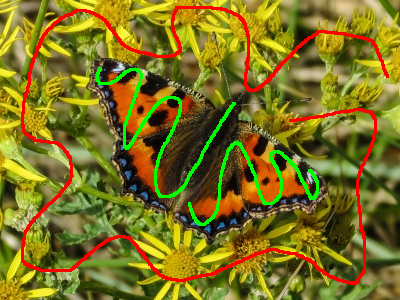} \\
        \includegraphics[width=\textwidth]{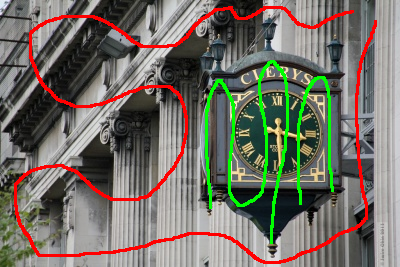} \\
        \caption{Scribbles}
    \end{subfigure}
    \begin{subfigure}[b]{0.18\textwidth}
        \centering
        \includegraphics[width=\textwidth]{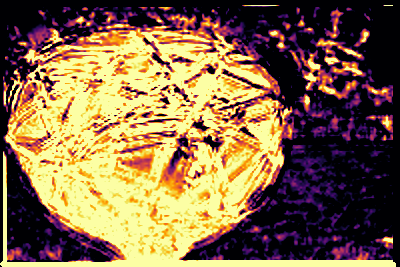} \\
        \includegraphics[width=\textwidth]{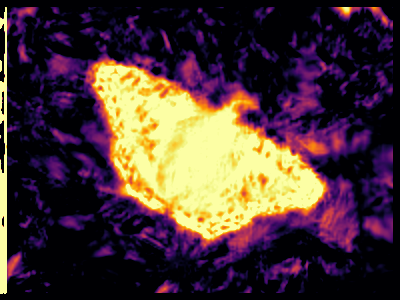} \\
        \includegraphics[width=\textwidth]{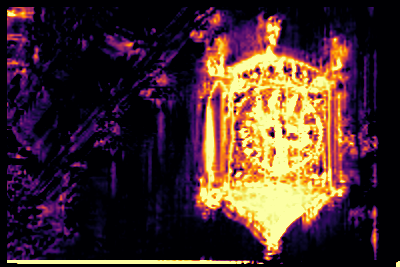} \\
        \caption{Proj}
    \end{subfigure}
    \begin{subfigure}[b]{0.18\textwidth}
        \centering
        \includegraphics[width=\textwidth]{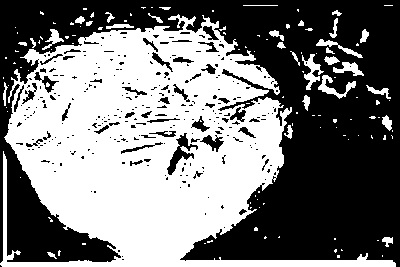} \\
        \includegraphics[width=\textwidth]{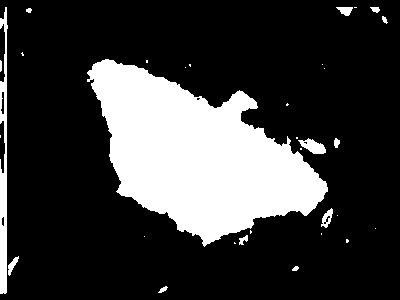} \\
        \includegraphics[width=\textwidth]{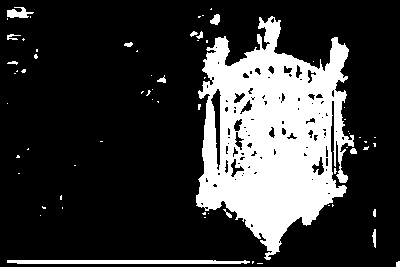} \\
        \caption{Thre}
    \end{subfigure}
        \begin{subfigure}[b]{0.18\textwidth}
        \centering
        \includegraphics[width=\textwidth]{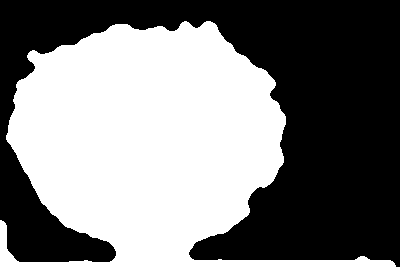} \\
        \includegraphics[width=\textwidth]{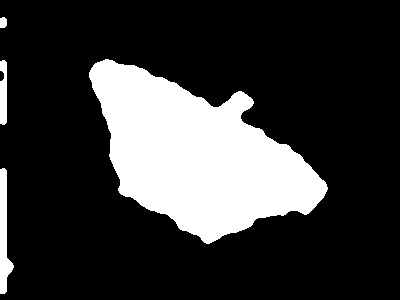} \\
        \includegraphics[width=\textwidth]{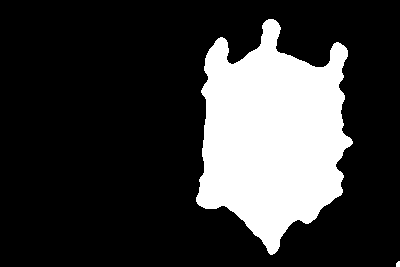} \\
        \caption{Iter}
    \end{subfigure}
    
    \caption{Images with oscillatory regions. $\lambda=70,30,30$ (top to bottom), respectively.}
    \label{fig:oscillate}
\end{figure}

% \begin{figure}[!ht]
%     \centering
%     \begin{subfigure}[b]{0.2\textwidth}
%         \centering
%         \includegraphics[width=\textwidth]{Figures/iterative_method/complicated/2010_005320_overlap.png} \\
%         \caption{Input}
%     \end{subfigure}
%     \begin{subfigure}[b]{0.2\textwidth}
%         \centering
%         \includegraphics[width=\textwidth]{Figures/iterative_method/complicated/2010_005320_prob0_rgb001xyInfr3.png} \\
%         \caption{Prob}
%     \end{subfigure}
%     \begin{subfigure}[b]{0.2\textwidth}
%         \centering
%         \includegraphics[width=\textwidth]{Figures/iterative_method/complicated/2010_005320_thre.png} \\
%         \caption{Thre}
%     \end{subfigure}
%         \begin{subfigure}[b]{0.2\textwidth}
%         \centering
%         \includegraphics[width=\textwidth]{Figures/iterative_method/complicated/2010_005320_PottsTD_ea50sig3B0_iter141.png} \\
%         \caption{Iter}
%     \end{subfigure}
    
%     \caption{The spatial kernel is not enough to detect and eliminate the noisy background}
    
%     \label{fig:eangle}
% \end{figure}

\paragraph{Discontinuous edges}

Some images contain weak or discontinuous boundaries between the object and the background. Gradient-based methods often fail in such cases, while region-based modeling remains effective.  
Figure~\ref{fig:europe} shows a high-resolution ($1008\times756$) night-light image of Europe cropped from \cite{roman2018nasa}. Object edges are only partially visible, and we provide two rough, thick scribbles. Using $\sigma_{I,0}=0.05$, $\sigma_{s,0}=3$, $R_0=3$, and $\lambda=10$, the model captures the coarse continental shape.

%=========Europe==================
\begin{figure}[!ht]
    \centering
    \begin{subfigure}[b]{0.3\textwidth}
        \centering
        \includegraphics[width=\textwidth]{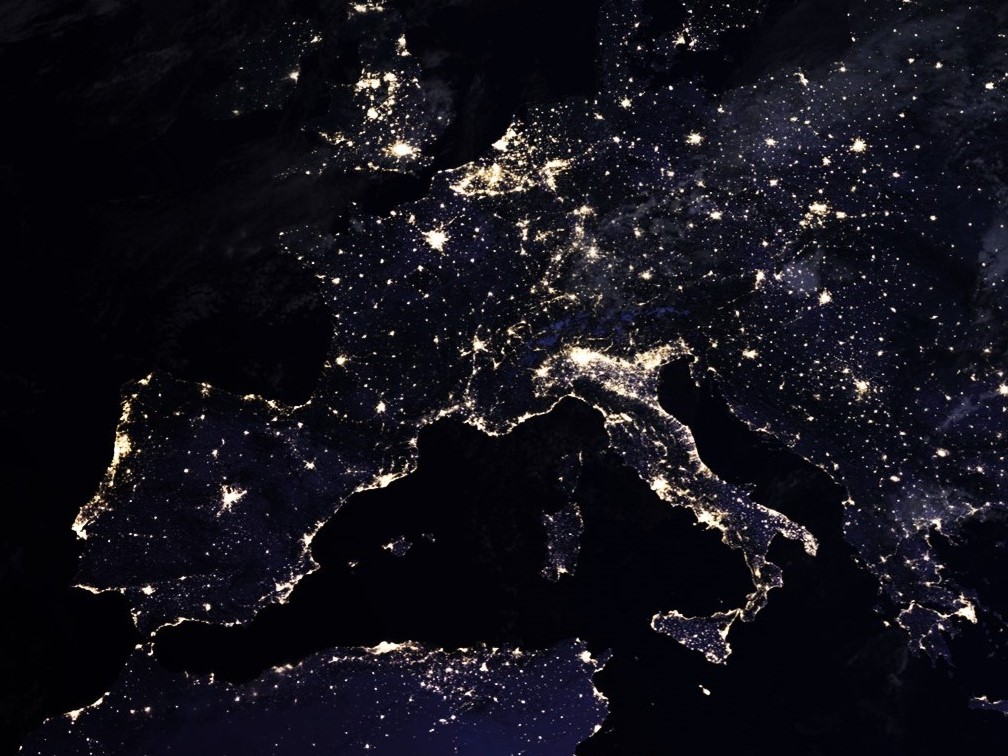} \\
        \caption{Image}
    \end{subfigure}
    \begin{subfigure}[b]{0.3\textwidth}
        \centering
        \includegraphics[width=\textwidth]{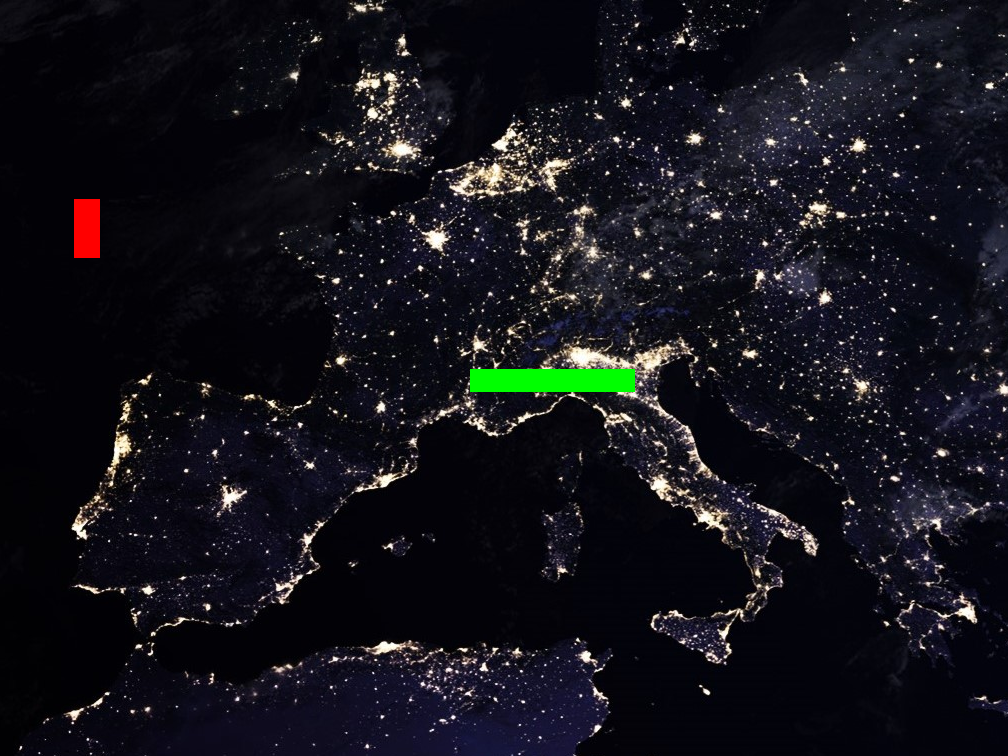} \\
        \caption{Scribbles}
    \end{subfigure}
    \begin{subfigure}[b]{0.3\textwidth}
        \centering
        \includegraphics[width=\textwidth]{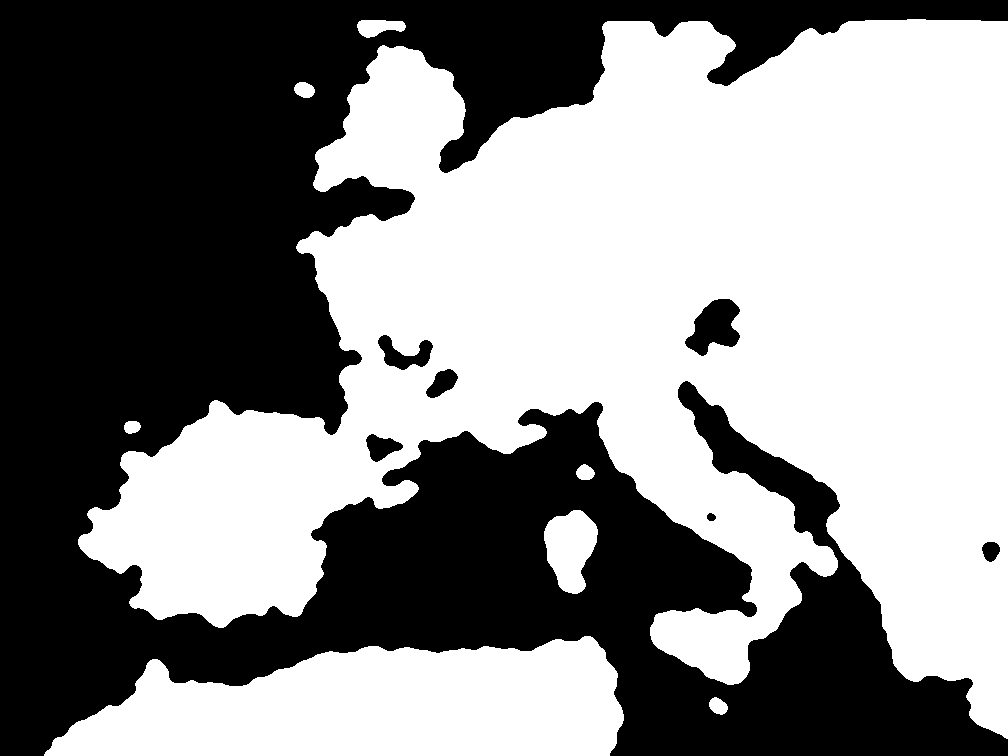} \\
        \caption{Iter}
    \end{subfigure}
    
    \caption{Discontinuous edges. European night light. Image size $1008\times756$.}
    
    \label{fig:europe}
\end{figure}

%=================Neural Network==============
\section{Experimental Results for Learning-Based Image Segmentation} \label{Sec:Network}

This section reports experiment results for the relaxed model in (\ref{eq:Potts_rkhs_loss}) under weak supervision. We analyze training effects and compare predictive performance with existing baselines.

\subsection{Implementation details}\label{sec:network_implementation}

We used ECSSD \cite{shi2015hierarchical}. The dataset contains 1000 semantically meaningful yet structurally complex images for salient-object detection. We manually annotated scribbles for the images. To improve robustness to scribble styles and semantic diversity, we additionally used 600 images from PASCAL VOC 2012 \cite{pascal-voc-2012} with scribble labels from \cite{lin2016scribblesup}.  
The first 200 images of ECSSD \cite{shi2015hierarchical}, also known as CSSD \cite{yan2013hierarchical}, were used as the test set. In total, we used 1400 training images and 200 test images.

We first precomputed all $\mathbf{u}_i$ with Algorithm~\ref{alg:u_binary}. Details follow Section~\ref{sec:u_implementation}. We fixed $R_{i,0}=3$. The parameters $\sigma_{I,i,0}$ and $\sigma_{s,i,0}$ were selected from three combinations:
\begin{itemize}
    \item PASCAL VOC 2012 images: $\sigma_{I,i,0}=0.05$, $\sigma_{s,i,0}=3$;
    \item ECSSD images: $\sigma_{I,i,0}=0.01$, $\sigma_{s,i,0}=1$;
    \item a small subset with easy color-based separation: $\sigma_{I,i,0}=0.01$, $\sigma_{s,i,0}=100$.
\end{itemize}
The computed $\hat{u}_{i,0}$ maps were stored as PNG files and loaded together with $I_i$ during training.

For the proposed energy $\mathcal{E}$, we set $\sigma=0.1$. Gaussian convolution was approximated by \texttt{GaussianBlur}, as in Section~\ref{sec:td_implementation}. We fixed $\lambda=5\mathrm{E}{-5}$. The ablation study for $\lambda$ is given in Appendix~\ref{sec:network_edge_weight}. Evaluation metrics are defined in Appendix~\ref{sec:evaluation_metrics}.

We used UNet \cite{ronneberger2015u} (31M parameters). Data augmentation was applied to inputs. We randomly flipped and rotated each $I_i$--$\hat{u}_{i,0}$ pair, then cropped patches of size $118\times132$. Images $I_i$ were further augmented with random color jitter, grayscale conversion, blur, and uniform noise.  
We used Adam~\cite{kingma2014adam}. All models were trained for 8000 epochs without strong image augmentation, then for another 8000 epochs with augmentation. The initial learning rate was $10^{-4}$ and halved halfway through each stage. Training dynamics are shown in Appendix~\ref{sec:network_training_dynamic}.

All code was implemented in PyTorch \cite{paszke2019pytorch}. Experiments were run on an NVIDIA RTX A6000 (48GB) GPU and an Intel Core i9-10980XE (3.00 GHz) CPU.

\subsection{Results on the Training Images} \label{sec:network_training}

We analyze the training effects through the evaluation scores and the generated images of the networks on the training images. In this sense, the predictions can be regarded as the solutions when jointly optimizing the individual energies via a common operator. The effect of the network and the TD regularization can be easily recognized by comparing the results of 'Thre' vs 'RKHS', and  'RKHS' vs 'RKHS $+$ TD', respectively. We find that training exhibits three major effects: denoising, object edge refinement, and boundary artifact elimination, in comparison to the no-training model. We also find that the TD regularization stabilizes the diffusion caused by the network.

\begin{table}[!h]
    \centering
    \begin{tabular}{l|c|ccc}
        \hline
        Methods & $\lambda$ & \hspace{2em} mIoU \hspace{2em} & \hspace{2em} mDice \hspace{2em} & \hspace{2em} mAcc \hspace{2em} \\
        \hline 
        Thre & (no training) & 81.56 & 89.65 & 91.72 \\
        RKHS & 0 & 82.04 (+0.48) & 89.96 (+0.31) & 92.21 (+0.49) \\
        RKHS $+$ TD & 5e-5 & 85.11 (+3.55) & 91.82 (+2.17) & 93.07 (+1.35) \\
        \hline
    \end{tabular}
    \caption{Evaluation scores of the network on the training set. 'Thre': the collection of $\hat{u}_0$ threshold, 'RKHS': results by the network trained with the fidelity term of (\ref{eq:Potts_rkhs_loss}) only, 'RKHS $+$ TD': that of trained via our full loss (\ref{eq:Potts_rkhs_loss}).}
    \label{tab:network_train_notrain}
\end{table}

\subsubsection{With/without training}

Table~\ref{tab:network_train_notrain} shows the gains from training. RKHS improves mIoU/mDice/mAcc by $+0.48/+0.31/+0.49$, and RKHS + TD improves by $+3.55/+2.17/+1.35$. In viewing the generated images, we summarize three phenomena.

\paragraph{Denoising}

This effect is common in homogeneous images or when kernel parameters for $u_0$ are suboptimal. Although the objects are visible in $u_0$, boundaries are not sharp, and many background pixels remain near $0.5$. Thresholding at $0.5$ can produce salt-and-pepper noise or edge leakage; as illustrated in Figure~\ref{fig:train_effect_denoise}(b)(c).  
Even with a large $\lambda$, solving (\ref{eq:main_model}) reduces but does not fully remove these noises, and may oversmooth details (see Figure~\ref{fig:train_effect_denoise}(d)).

\begin{figure}[!ht]
    \centering
    \begin{subfigure}[b]{0.15\textwidth}
        \centering
        \includegraphics[width=\textwidth]{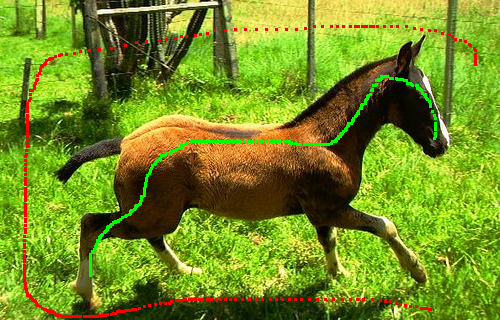} \\
        \includegraphics[width=\textwidth]{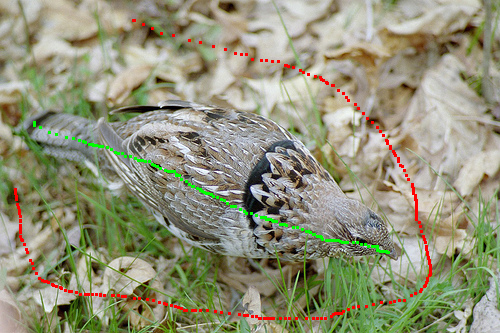} \\
        \caption{\\Input}
    \end{subfigure}
    \begin{subfigure}[b]{0.15\textwidth}
        \centering
        \includegraphics[width=\textwidth]{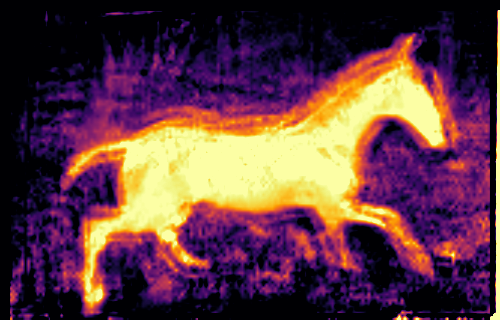} \\
        \includegraphics[width=\textwidth]{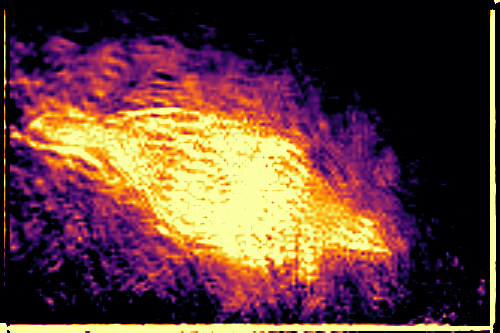} \\
        \caption{\\Proj}
    \end{subfigure}
    \begin{subfigure}[b]{0.15\textwidth}
        \centering
        \includegraphics[width=\textwidth]{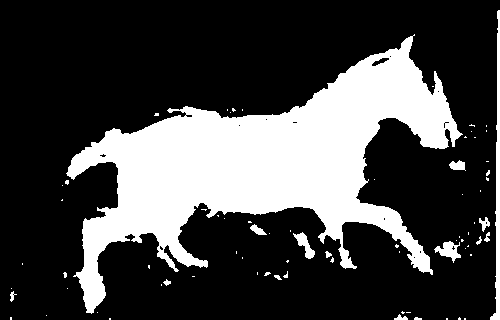} \\
        \includegraphics[width=\textwidth]{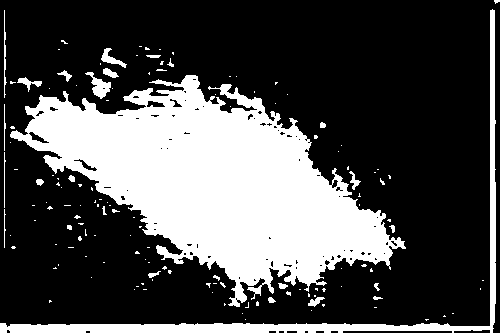} \\
        \caption{\\Thre}
    \end{subfigure}
        \begin{subfigure}[b]{0.15\textwidth}
        \centering
        \includegraphics[width=\textwidth]{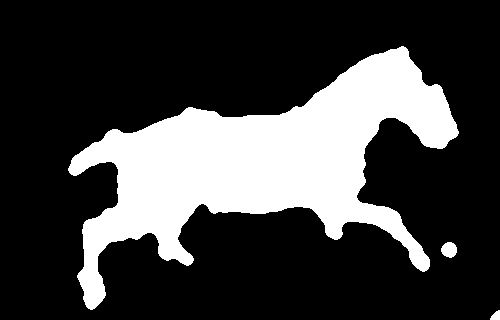} \\
        \includegraphics[width=\textwidth]{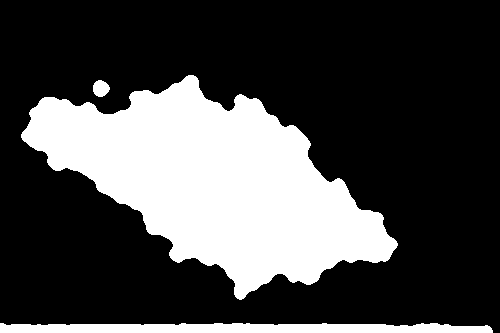} \\
        \caption{\\Iter}
    \end{subfigure}
        \begin{subfigure}[b]{0.15\textwidth}
        \centering
        \includegraphics[width=\textwidth]{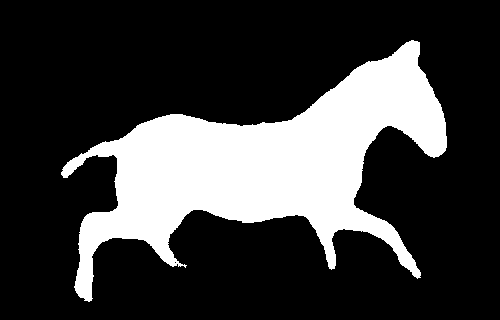} \\
        \includegraphics[width=\textwidth]{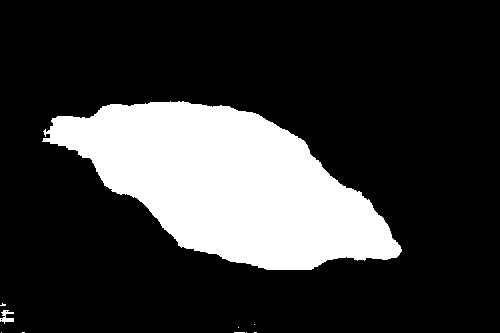} \\
        \caption{\\RKHS}
    \end{subfigure}
        \begin{subfigure}[b]{0.15\textwidth}
        \centering
        \includegraphics[width=\textwidth]{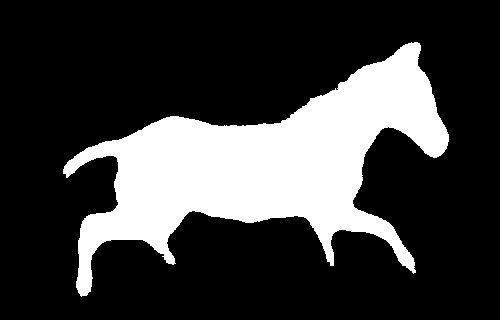} \\
        \includegraphics[width=\textwidth]{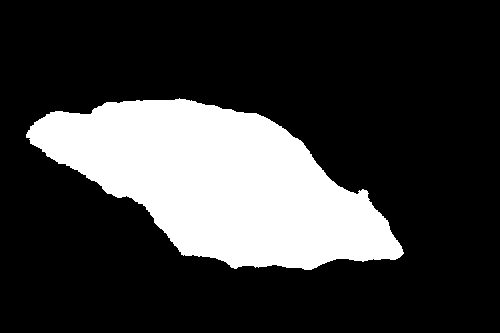} \\
        \caption{RKHS $+$ TD}
    \end{subfigure}
    
    \caption{Training effect: denoising.}
    
    \label{fig:train_effect_denoise}
\end{figure}

Training suppresses these noises effectively. As illustrated in Figure~\ref{fig:train_effect_denoise}(e)(f), the network seems to learn to threshold at a proper level, thereby having cleaner and sharper boundaries in the prediction. For example, we find it matches the threshold of 0.7 for the pigeon image. This behavior supports edge refinement.

\paragraph{Object-edge refinement}

Edges may be poorly preserved in $u_0$. In Figure~\ref{fig:train_effect_refine}, sheep legs are blurred because shadow regions are misclassified, and the parrot suffers erosion due to an overlarge $\sigma_{s,0}$. The standard model \eqref{eq:main_model_binary} has limited correction when initial edges are already smooth.

\begin{figure}[!ht]
    \centering
    \begin{subfigure}[b]{0.15\textwidth}
        \centering
        \includegraphics[width=\textwidth]{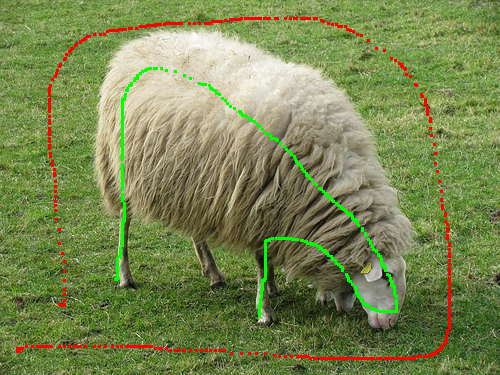} \\
        \includegraphics[width=\textwidth]{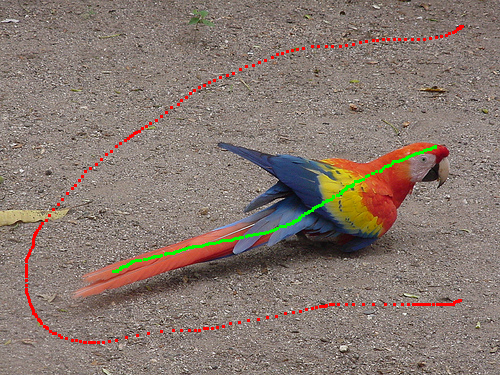} \\
        \caption{\\Input}
    \end{subfigure}
    \begin{subfigure}[b]{0.15\textwidth}
        \centering
        \includegraphics[width=\textwidth]{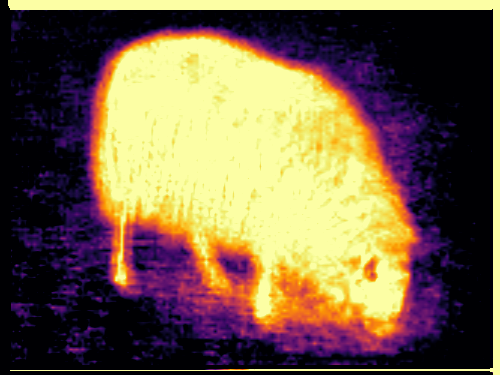} \\
        \includegraphics[width=\textwidth]{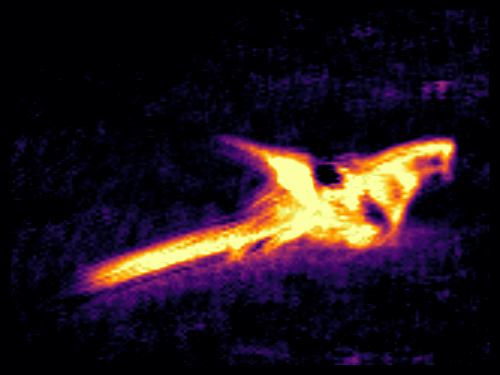} \\
        \caption{\\Proj}
    \end{subfigure}
    \begin{subfigure}[b]{0.15\textwidth}
        \centering
        \includegraphics[width=\textwidth]{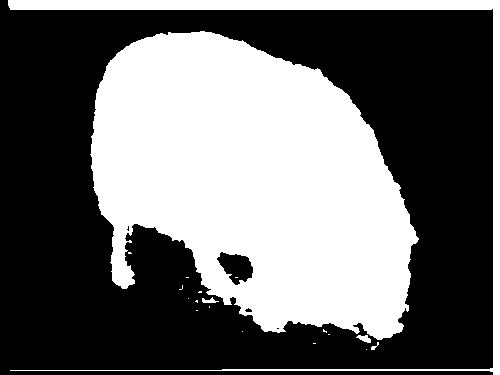} \\
        \includegraphics[width=\textwidth]{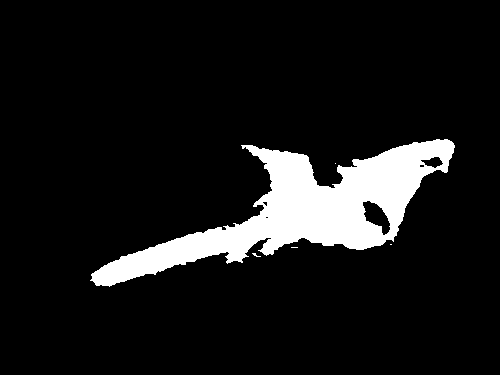} \\
        \caption{\\Thre}
    \end{subfigure}
        \begin{subfigure}[b]{0.15\textwidth}
        \centering
        \includegraphics[width=\textwidth]{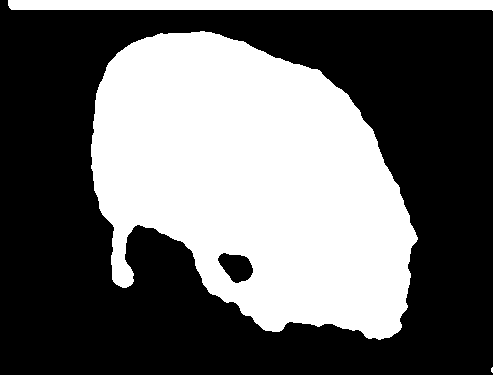} \\
        \includegraphics[width=\textwidth]{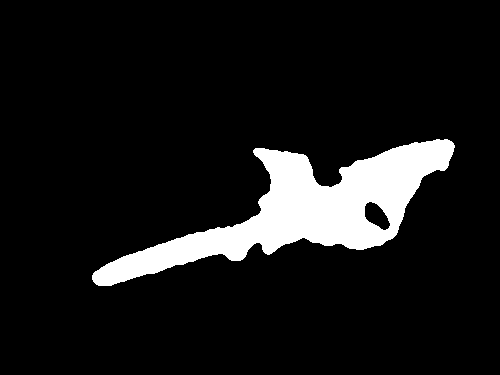} \\
        \caption{\\Iter}
    \end{subfigure}
        \begin{subfigure}[b]{0.15\textwidth}
        \centering
        \includegraphics[width=\textwidth]{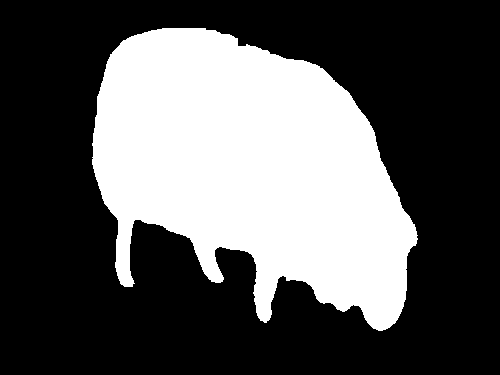} \\
        \includegraphics[width=\textwidth]{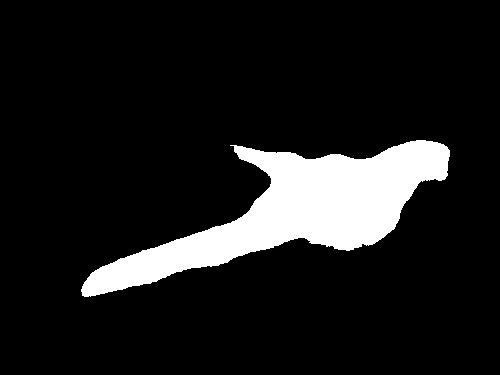} \\
        \caption{\\RKHS}
    \end{subfigure}
        \begin{subfigure}[b]{0.15\textwidth}
        \centering
        \includegraphics[width=\textwidth]{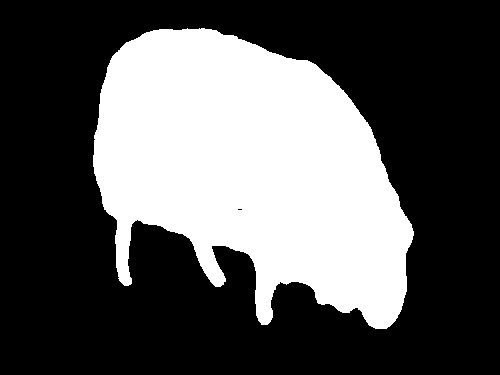} \\
        \includegraphics[width=\textwidth]{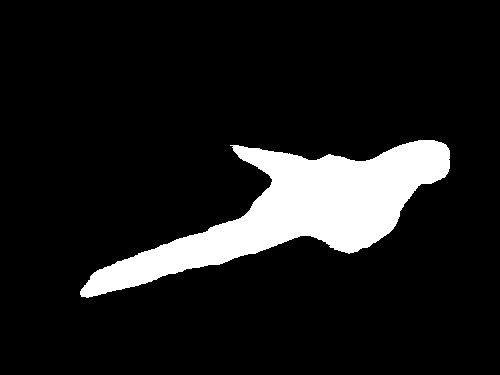} \\
        \caption{RKHS $+$ TD}
    \end{subfigure}
    
    \caption{Training effect: object-edge refinement.}
    
    \label{fig:train_effect_refine}
\end{figure}

After training, both RKHS and RKHS + TD align predictions better with object boundaries. In the sheep image, leg boundaries are recovered; in the parrot image, missing parts are reconstructed. This demonstrates that the networks have a deep understanding of what and where the foreground is. 

\paragraph{Boundary-artifact elimination}

Boundary artifacts appear irregularly in $u_0$, often as straight segments. Their cause can be the zero boundary condition, misleading the kernel value. Examples are shown in Figure~\ref{fig:train_effect_boundary}. The standard model \eqref{eq:main_model_binary} reduces but may not eliminate them, even for large $\lambda$. In contrast, both trained models remove these artifacts.

\begin{figure}[!ht]
    \centering
    \begin{subfigure}[b]{0.15\textwidth}
        \centering
        \includegraphics[width=\textwidth]{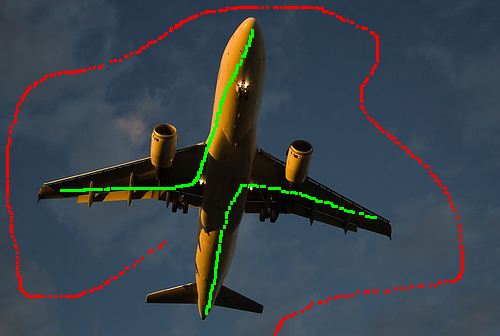} \\
        \includegraphics[width=\textwidth]{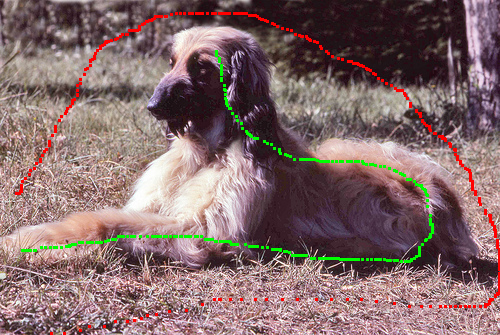} \\
        \caption{\\Input}
    \end{subfigure}
    \begin{subfigure}[b]{0.15\textwidth}
        \centering
        \includegraphics[width=\textwidth]{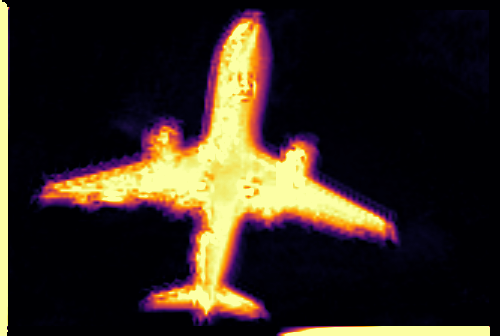} \\
        \includegraphics[width=\textwidth]{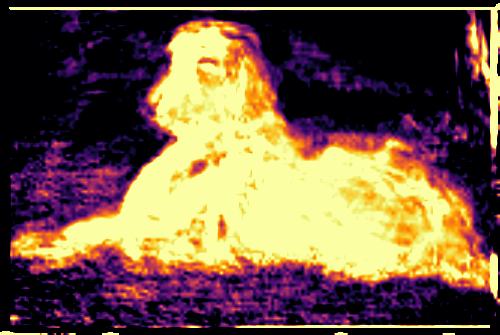} \\
        \caption{\\Proj}
    \end{subfigure}
    \begin{subfigure}[b]{0.15\textwidth}
        \centering
        \includegraphics[width=\textwidth]{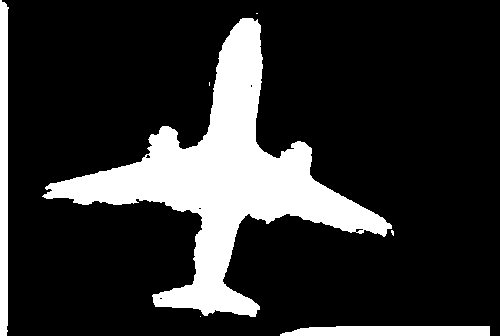} \\
        \includegraphics[width=\textwidth]{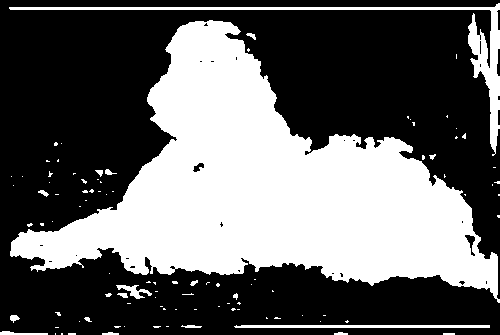} \\
        \caption{\\Thre}
    \end{subfigure}
        \begin{subfigure}[b]{0.15\textwidth}
        \centering
        \includegraphics[width=\textwidth]{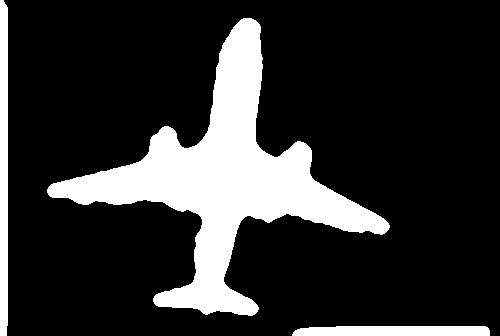} \\
        \includegraphics[width=\textwidth]{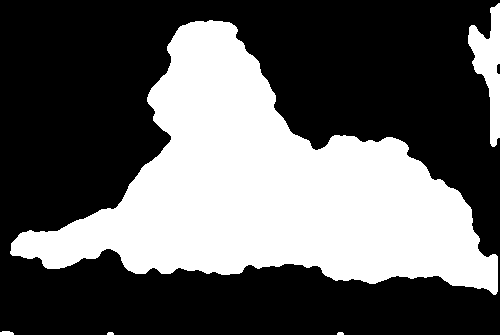} \\
        \caption{\\Iter}
    \end{subfigure}
        \begin{subfigure}[b]{0.15\textwidth}
        \centering
        \includegraphics[width=\textwidth]{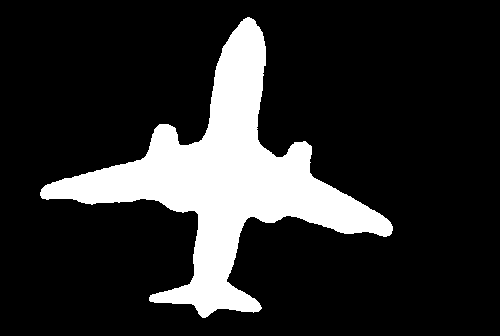} \\
        \includegraphics[width=\textwidth]{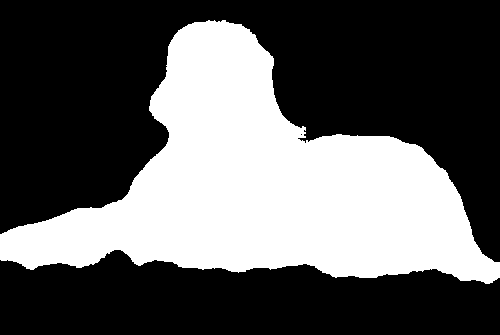} \\
        \caption{\\RKHS}
    \end{subfigure}
        \begin{subfigure}[b]{0.15\textwidth}
        \centering
        \includegraphics[width=\textwidth]{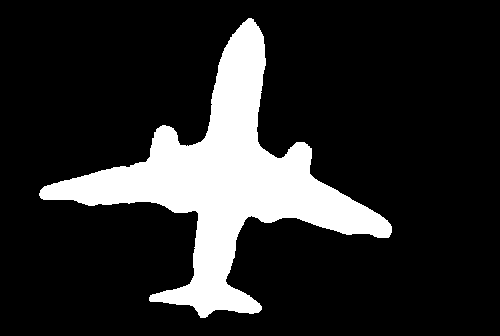} \\
        \includegraphics[width=\textwidth]{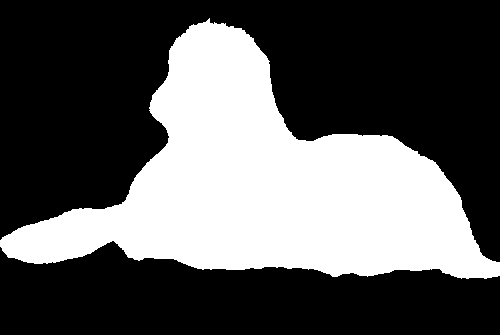} \\
        \caption{RKHS $+$ TD}
    \end{subfigure}
    
    \caption{Training effect: boundary-artifact elimination.}
    
    \label{fig:train_effect_boundary}
\end{figure}

In practice, these three effects often appear simultaneously. Figure~\ref{fig:training_difficult} shows several difficult cases. The iterative model improves pre-segmentation partially, while trained models recover object structure more completely. RKHS + TD gives the most stable results in severe cases. 

In the $\hat{u}_0$ of the first image, the edges of the horn are split. The thickets are misclassified as foreground. The body of the bull is corrupted by noise. Our standard model (\ref{eq:main_model_binary}) can only smooth part of the split edges, while the trained models regenerate the target shape, demonstrating both the denoising and object edge refinement effects. The second image has a much worse pre-segmentation using only $\hat{u}_0$. The salt and pepper noises caused by the houses and the large artifact produced from the shimmering water destroyed the target edges. The $\hat{v}_0$ is still bad, but 'RKHS' and 'RKHS $+$ TD' can restore the target. This again suggests that training can denoise and refine object edges even in more severe cases. For the third image, the $\hat{u}_0$ suffers from heavy salt and pepper noise caused by the inhomogeneous distributions, and thin boundary artifacts that appear at the bottom, top, and right boundary. These noises are damped in $v_0$, but the segmentation is not ideal. The 'RKHS' gives a more compact result, but it also generates new artifacts that may be caused by the background components. The problem is solved in 'RKHS $+$ TD'. A compact but more accurate result is provided.

\begin{figure}[!ht]
    \centering
    \begin{subfigure}[b]{0.15\textwidth}
        \centering
        \includegraphics[width=\textwidth]{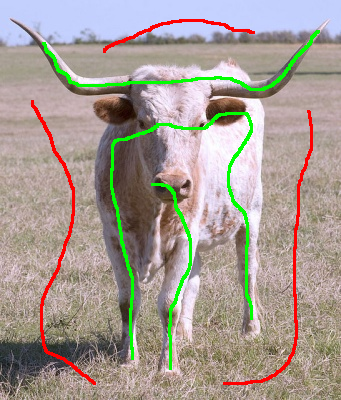} \\
        \includegraphics[width=\textwidth]{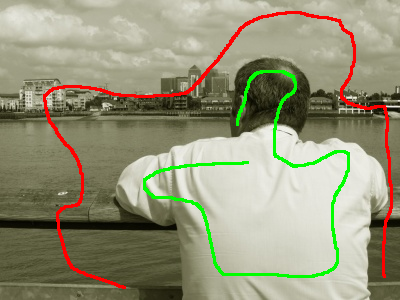} \\
        \includegraphics[width=\textwidth]{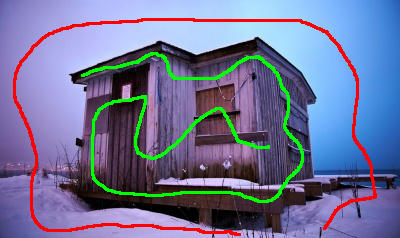} \\
        \caption{\\Input}
    \end{subfigure}
    \begin{subfigure}[b]{0.15\textwidth}
        \centering
        \includegraphics[width=\textwidth]{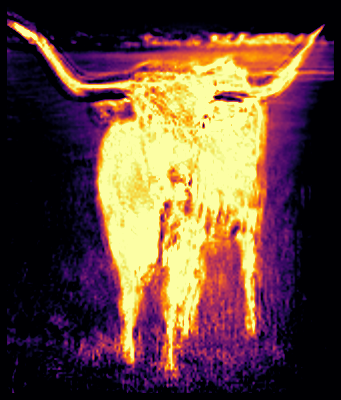} \\
        \includegraphics[width=\textwidth]{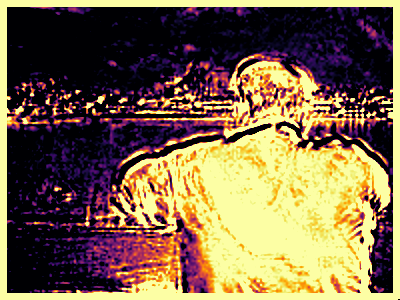} \\
        \includegraphics[width=\textwidth]{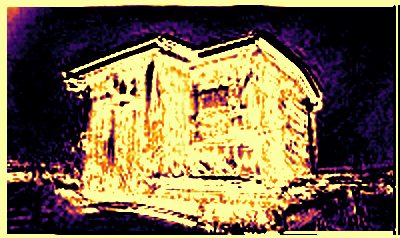} \\
        \caption{\\Proj}
    \end{subfigure}
    \begin{subfigure}[b]{0.15\textwidth}
        \centering
        \includegraphics[width=\textwidth]{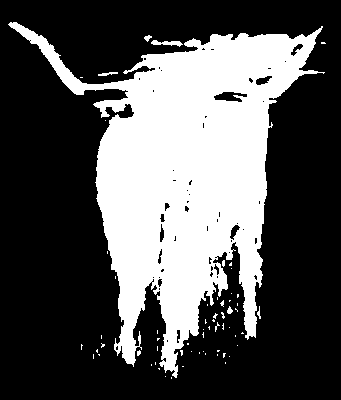} \\
        \includegraphics[width=\textwidth]{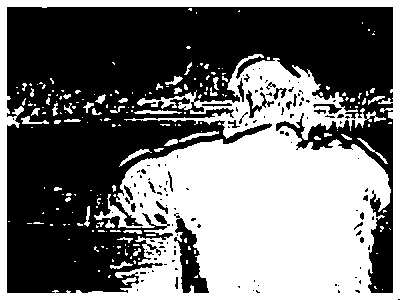} \\
        \includegraphics[width=\textwidth]{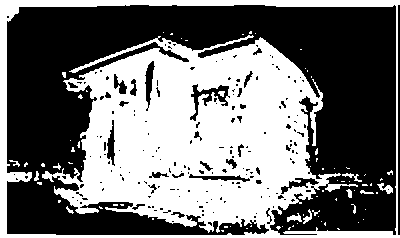} \\
        \caption{\\Thre}
    \end{subfigure}
        \begin{subfigure}[b]{0.15\textwidth}
        \centering
        \includegraphics[width=\textwidth]{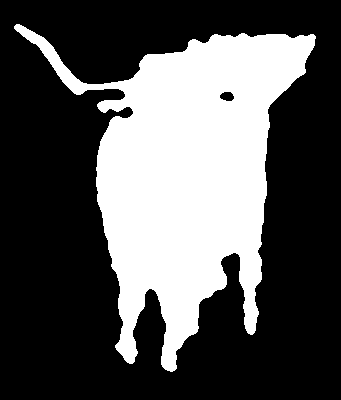} \\
        \includegraphics[width=\textwidth]{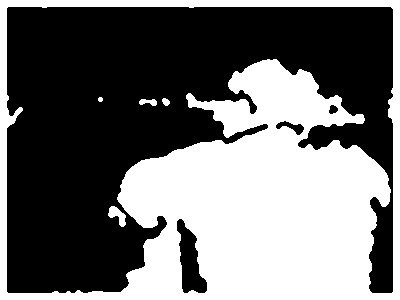} \\
        \includegraphics[width=\textwidth]{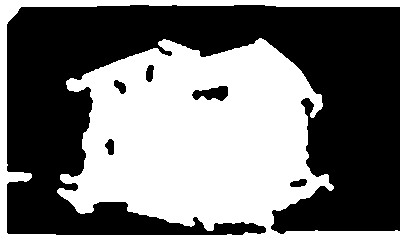} \\
        \caption{\\Iter}
    \end{subfigure}
        \begin{subfigure}[b]{0.15\textwidth}
        \centering
        \includegraphics[width=\textwidth]{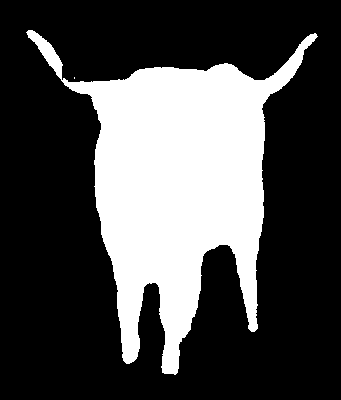} \\
        \includegraphics[width=\textwidth]{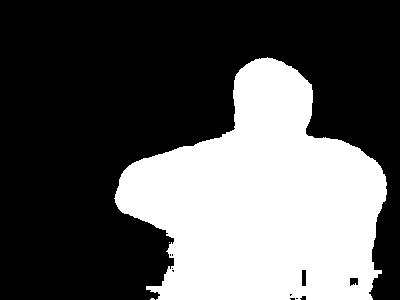} \\
        \includegraphics[width=\textwidth]{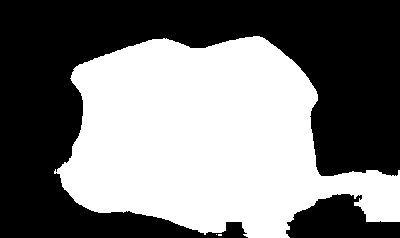} \\
        \caption{\\RKHS}
    \end{subfigure}
        \begin{subfigure}[b]{0.15\textwidth}
        \centering
        \includegraphics[width=\textwidth]{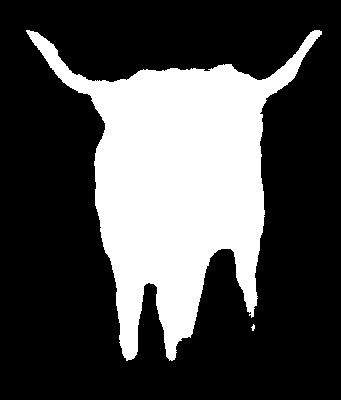} \\
        \includegraphics[width=\textwidth]{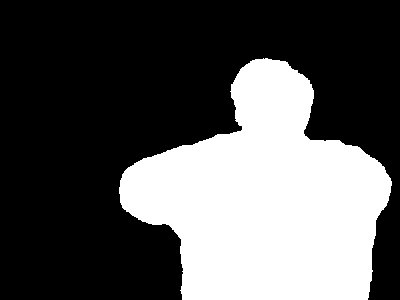} \\
        \includegraphics[width=\textwidth]{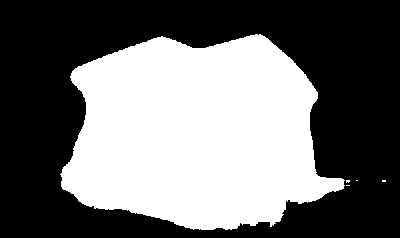} \\
        \caption{RKHS $+$ TD}
    \end{subfigure}
    
    \caption{Training effects in difficult cases.}
    
    \label{fig:training_difficult}
\end{figure}

\subsubsection{Effect of TD regularization during training}

The last case in figure~\ref{fig:training_difficult} may reveal the reason why there is a much more significant gain in evaluation scores, as stated in table~\ref{tab:network_train_notrain}, that the network trained by 'RKHS' may be misclassifying other components as the foreground. Some more instances are illustrated in figure~\ref{fig:training_regularization}. The network with 'RKHS' interprets the branches and rails to be part of the foreground, although they are not misclassified in $u_0$. This may be because the networks learn many other examples that are correct to do so. Notably, these artifacts significantly increase the total perimeter length. The additional TD term in 'RKHS $+$ TD' will penalize more heavily on these false positives (over-segmentation), which can drive the predictions to stop at the smallest objects that are most alike by $u_0$. Hence, the additional TD regularization can stabilize training by damping artifacts and achieve a higher evaluation score.

\begin{figure}[!ht]
    \centering
    \begin{subfigure}[b]{0.15\textwidth}
        \centering
        \includegraphics[width=\textwidth]{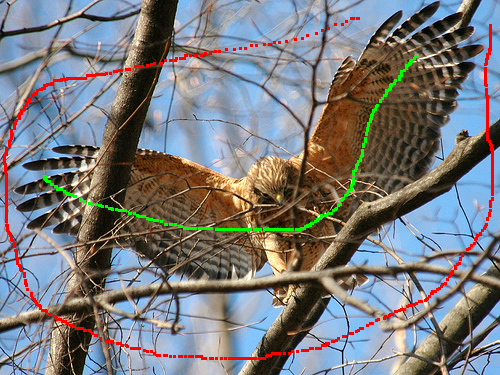} \\
        \includegraphics[width=\textwidth]{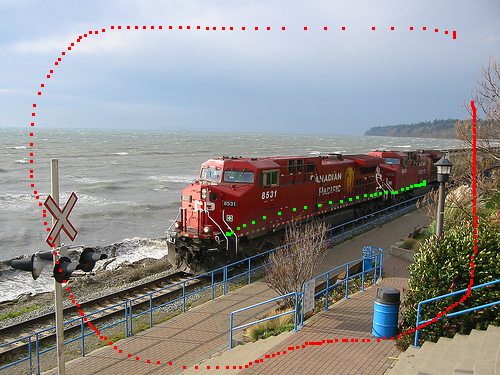} \\
        \caption{\\Input}
    \end{subfigure}
    \begin{subfigure}[b]{0.15\textwidth}
        \centering
        \includegraphics[width=\textwidth]{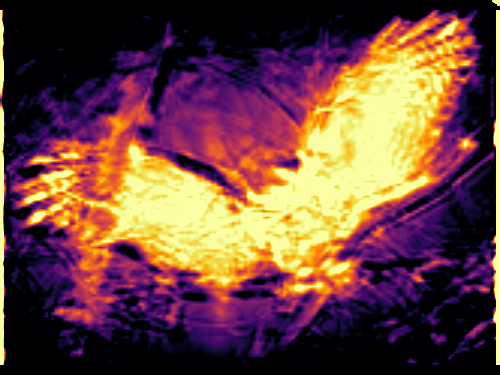} \\
        \includegraphics[width=\textwidth]{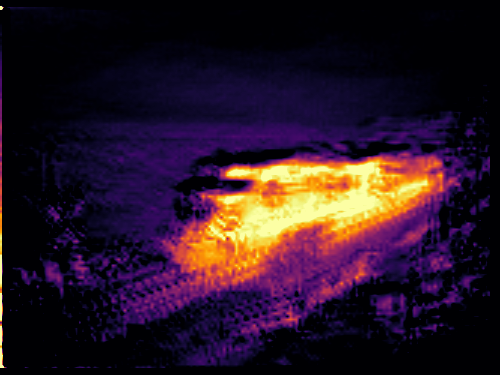} \\
        \caption{\\Proj}
    \end{subfigure}
    \begin{subfigure}[b]{0.15\textwidth}
        \centering
        \includegraphics[width=\textwidth]{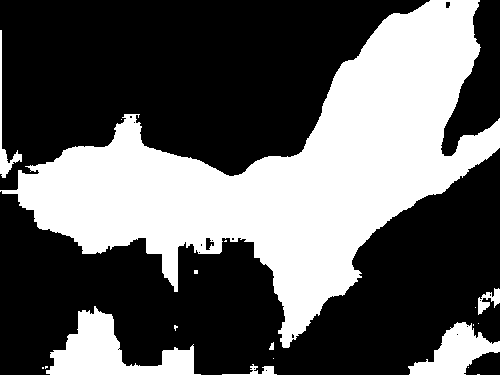} \\
        \includegraphics[width=\textwidth]{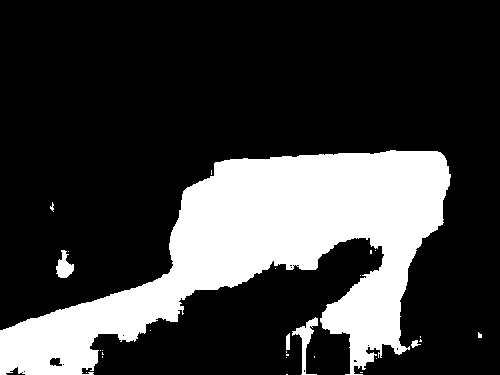} \\
        \caption{\\RKHS}
    \end{subfigure}
    \begin{subfigure}[b]{0.15\textwidth}
        \centering
        \includegraphics[width=\textwidth]{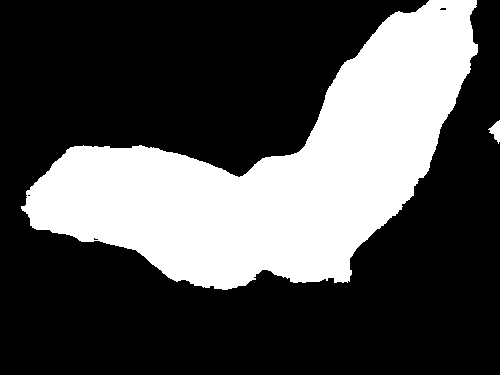} \\
        \includegraphics[width=\textwidth]{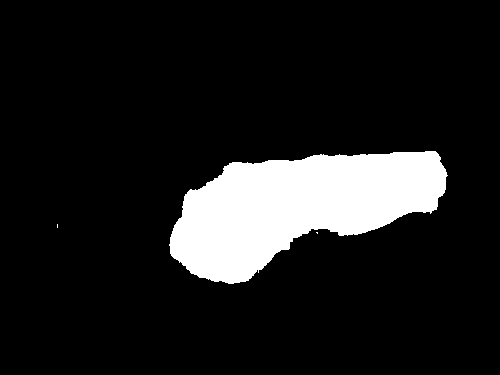} \\
        \caption{RKHS $+$ TD}
    \end{subfigure}
    
    \caption{TD regularization stabilizes training.}
    
    \label{fig:training_regularization}
\end{figure}

\subsection{Results on the Test Images}

We evaluate predictive performance against fully supervised and weakly supervised baselines. All models follow Section~\ref{sec:network_implementation}, but use different losses. Loss details are in Appendix~\ref{sec:network_energy_compare}.

\subsubsection{Evaluation scores}
Table~\ref{tab:network_all} lists the scores of the methods across different evaluation metrics. 'RKHS' denotes our loss function (\ref{eq:Potts_rkhs_loss}) that contains only the data-fidelity term. Comparing with the partial cross-entropy (PCE) baselines, our model has a significant gain in all three metrics, for $+2.98$, $+2.09$, and $+1.85$, respectively. Indeed, the performances do not lag far behind the fully-supervised setting ('CE'), with a difference in score for $-3.23$, $-2.13$, and $-0.51$, respectively, even though we are training with sparse annotations. This suggests that our model can provide good regularization to guide the training.

Our models also outperform other weakly-supervised methods. All these methods are based on 'PCE' plus regularization. They make a gain of at most $+1.54$, $+1.07$, $+1.61$ among the three evaluation metrics when comparing to vanilla 'PCE'. However, none of them outperforms our 'RKHS'. It may be because they were limited by the capability of 'PCE' in this experiment, where our 'RKHS', which used a completely different data-fidelity term, portrays the objects in a more effective way.

The TD regularization also raises the evaluation scores. Comparing 'PCE' and 'PCE $+$ TD', the additional TD regularization boosts the scores with $+0.73$, $+0.54$, $+1.01$, respectively. While for 'RKHS', it gives a marginal gain of $+0.1$, $+0.03$ on mIoU and mDice. Both effects are less significant than we have observed in training.
% {\color{red} One possible reason is the unbalanced distribution of the training and testing data, which may cause the TD regularization to be more effective on the training set than on the test set. The training set contains images from both ECSSD \cite{shi2015hierarchical} and PASCAL VOC 12 \cite{pascal-voc-2012} datasets, but the test set is chosen as the CSSD \cite{yan2013hierarchical}, which is a subset of the ECSSD. Even though images from two datasets are very close in content, this splitting may still cause a slight difference in the data distribution. Besides, the fixed TD strength $\lambda$ may not be optimal for all images, and it may over-smooth some test images while under-smoothing some training images. Moreover, in the test set, the baseline RKHS predictions are already reasonably good and very close to the 'supervised' case, leaving less room for TD to help. Even though the gain in metrics is small, we still observe many visual improvements as described in Section 5.3.2, indicating the TD term is still very helpful.}
The mild gain against 'RKHS' is because the benefit of TD regularization is most effective when the baseline RKHS prediction already captures the main part of the object, and the remaining errors are primarily false positives. However, if RKHS predictions are weak and contain many false negatives (under-segmentation), the same TD term can remove true object regions, thereby introducing more false negatives. An example can be found in Appendix~\ref{sec:network_td_undersegment}. During training, since the kernel parameters are chosen based on the distribution of the training images, the 'RKHS' results are already strong (over 80 in mIoU). Hence, the TD term can further improve performance. While in the testing stage, the network has not seen the test images. The predictions of the RKHS are weaker and even far from the targets. Although the TD regularization can still denoise the object boundary, it may introduce more false negatives. This explains why the improvement from `RKHS' to `RKHS $+$ TD' is modest compared with the gain on training sets. Notably, mIoU or mDice penalize false positives (over-segmentation) more heavily than mAcc, and mAcc is more sensitive to false negatives (under-segmentation) than mIoU or mDice. Hence, the gain in mAcc is modest compared with the mIOU in training, as reflected by Table~\ref{tab:network_train_notrain}, and even decreases by $-1.23$ in the testing set.
% The mild gain against 'RKHS' may be due to over-denoising in certain prediction cases, as we have seen in training that 'RKHS' $+$ 'TD' will prefer smaller objects. This can be the reason why it lowers the mAcc by $-1.23$. mAcc is a more sensitive metric. A shift in edge elements can change both the ratio of true positive (TP) and true negative (TN). If the prediction of objects is slightly smaller, the mAcc dropped more significantly than mIoU and mDice. On average, it can cause harm. 
We also note that 'CE $+$ TD' also has a slight drop in mIoU and mDice when compared to 'CE', for at most $0.1$. This can be because the additional TD regularization over-smooths the objects, as 'CE' already well recovered the object shapes. While for 'PCE' and 'RKHS', since we are weakly supervised, the TD term can still provide additional, useful regularization.

\begin{table}[!h]
    \centering
    \begin{tabular}{l|c|ccc}
        \hline
        Methods & $\lambda$ & \hspace{1em} mIoU \hspace{1em} & \hspace{1em} mDice  \hspace{1em} & \hspace{1em} mAcc \hspace{1em} \\
        \hline 
        PCE & 0 & 72.41 & 83.43 & 86.76 \\
        \textbf{RKHS (ours)} & 0 & \color{cyan}{75.39} & \color{cyan}{85.52} & \color{red}{88.61} \\
        \hline
        PCE $+$ TD & 5e-5 & 73.14 & 83.97 & 87.77 \\
        PCE $+$ NCut & 5e-5 & 72.22 & 83.35	& 87.73 \\
        PCE $+$ NCut $+$ KCut & 5e-5 & 73.95 & 84.50 & 87.24 \\
        PCE $+$ CV & 1e-1 & 73.70 & 84.38 & \color{cyan}{88.37} \\
        & 1 & \color{red}{75.76} & \color{red}{85.76} & \color{blue}{88.45} \\
        \textbf{RKHS $+$ TD (ours)} & 5e-5 & \color{blue}{75.49} & \color{blue}{85.55} & 87.38 \\
        \hline
        CE & 0 & 78.62 & 87.65 & 89.12 \\
        CE $+$ TD & 5e-5 & 78.52 & 87.59 & 89.29 \\
        \hline
    \end{tabular}
    \caption{Comparison with variational-loss baselines. Red/blue/cyan indicate first/second/third best, respectively.}
    \label{tab:network_all}
\end{table}

\subsubsection{Visual results}

Figures~\ref{fig:pred_bears}--\ref{fig:pred_flowers} show qualitative comparisons. More examples are provided in Appendix~\ref{sec:network_pred_figs_appendix}.  
Methods sharing the same fidelity term show similar behavior. PCE-based methods often produce thicker foreground masks, while RKHS-based methods are typically closer to object boundaries, suggesting finer foreground structure.

\begin{figure}[!h]
    \centering
    \addtolength{\tabcolsep}{-5pt}
    \begin{tabular}{ccccc}
        \begin{subfigure}[t]{0.15\textwidth}
            \setcounter{subfigure}{0}
            \centering
            \includegraphics[width=\textwidth]{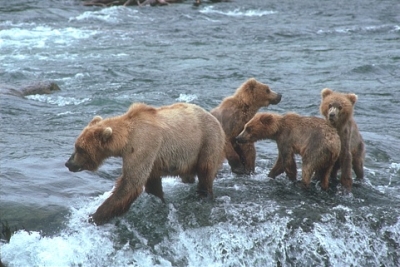}
            \caption{\\Image}
        \end{subfigure}
        &
        \begin{subfigure}[t]{0.15\textwidth}
            \setcounter{subfigure}{1}
            \centering
            \includegraphics[width=\textwidth]{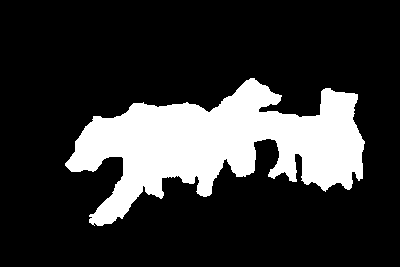}
            \caption{\\Mask}
        \end{subfigure}
        &
        \begin{subfigure}[t]{0.15\textwidth}
            \setcounter{subfigure}{2}
            \centering
            \includegraphics[width=\textwidth]{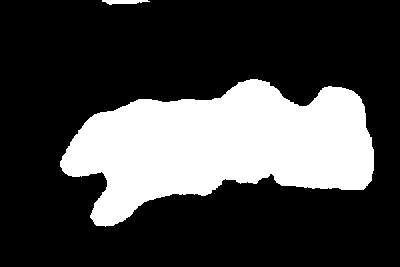}
            \caption{\\PCE}
        \end{subfigure}
        &
        \begin{subfigure}[t]{0.15\textwidth}
            \setcounter{subfigure}{3}
            \centering
            \includegraphics[width=\textwidth]{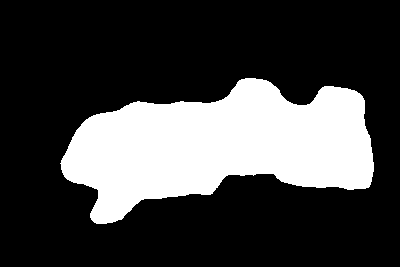}
            \caption{\\PCE+TD}
        \end{subfigure}
        &
        \begin{subfigure}[t]{0.15\textwidth}
            \setcounter{subfigure}{4}
            \centering
            \includegraphics[width=\textwidth]{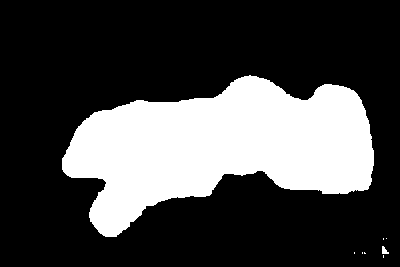}
            \caption{\\PCE+NCut}
        \end{subfigure}
        \\[0.5em]
        \multicolumn{5}{c}{
            \begin{tabular}{cccc}
                \begin{subfigure}[t]{0.15\textwidth}
                    \setcounter{subfigure}{5}
                    \centering
                    \includegraphics[width=\textwidth]{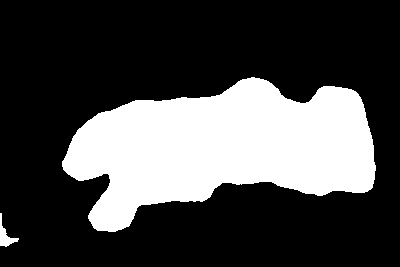}
                    \caption{PCE\\+NCut+KCut}
                \end{subfigure}
                &
                \begin{subfigure}[t]{0.15\textwidth}
                    \setcounter{subfigure}{6}
                    \centering
                    \includegraphics[width=\textwidth]{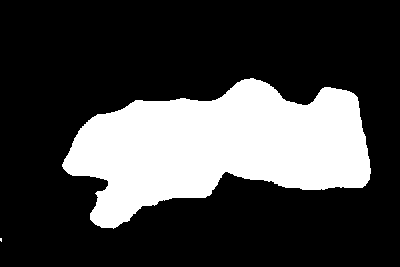}
                    \caption{\\PCE+CV}
                \end{subfigure}
                &
                \begin{subfigure}[t]{0.15\textwidth}
                    \setcounter{subfigure}{7}
                    \centering
                    \includegraphics[width=\textwidth]{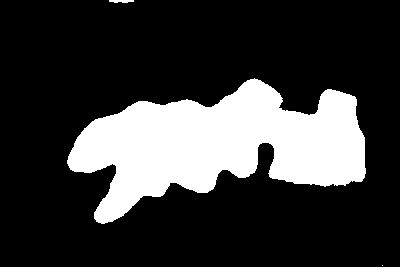}
                    \caption{\\\textbf{RKHS}}
                \end{subfigure}
                &
                \begin{subfigure}[t]{0.15\textwidth}
                    \setcounter{subfigure}{8}
                    \centering
                    \includegraphics[width=\textwidth]{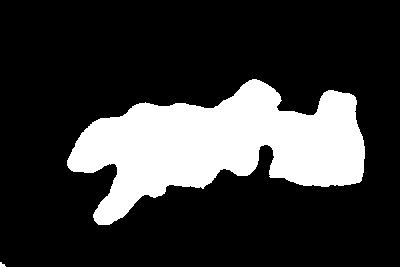}
                    \caption{\\\textbf{RKHS+TD}}
                \end{subfigure}
            \end{tabular}
        }
    \end{tabular}
    \addtolength{\tabcolsep}{5pt}
    \caption{Bears}
    \label{fig:pred_bears}
\end{figure}

Figure~\ref{fig:pred_lemur} shows another example. PCE predicts large boundary artifacts caused by dark background regions. PCE-based regularized variants reduce but do not fully remove them. RKHS reduces artifacts substantially, and RKHS + TD removes them in this case.

\begin{figure}[!h]
    \centering
    \addtolength{\tabcolsep}{-5pt}
    \begin{tabular}{ccccc}
        \begin{subfigure}[t]{0.15\textwidth}
            \setcounter{subfigure}{0}
            \centering
            \includegraphics[width=\textwidth]{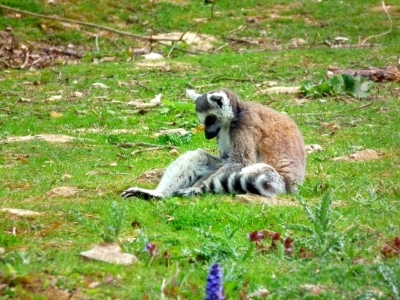}
            \caption{\\Image}
        \end{subfigure}
        &
        \begin{subfigure}[t]{0.15\textwidth}
            \setcounter{subfigure}{1}
            \centering
            \includegraphics[width=\textwidth]{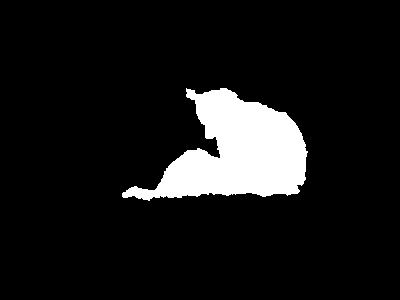}
            \caption{\\Mask}
        \end{subfigure}
        &
        \begin{subfigure}[t]{0.15\textwidth}
            \setcounter{subfigure}{2}
            \centering
            \includegraphics[width=\textwidth]{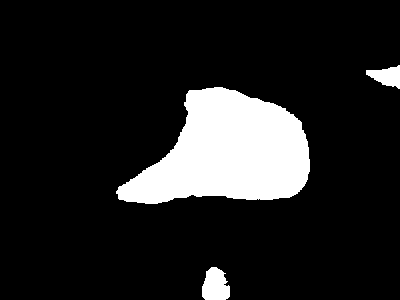}
            \caption{\\PCE}
        \end{subfigure}
        &
        \begin{subfigure}[t]{0.15\textwidth}
            \setcounter{subfigure}{3}
            \centering
            \includegraphics[width=\textwidth]{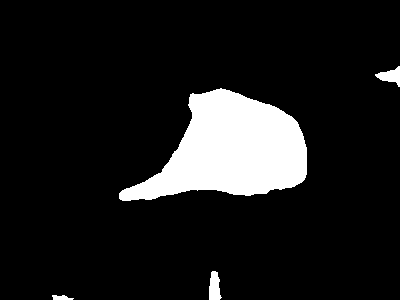}
            \caption{\\PCE+TD}
        \end{subfigure}
        &
        \begin{subfigure}[t]{0.15\textwidth}
            \setcounter{subfigure}{4}
            \centering
            \includegraphics[width=\textwidth]{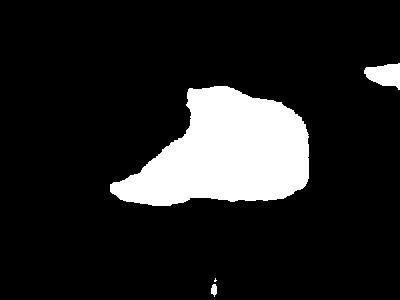}
            \caption{\\PCE+NCut}
        \end{subfigure}
        \\[0.5em]
        \multicolumn{5}{c}{
            \begin{tabular}{cccc}
                \begin{subfigure}[t]{0.15\textwidth}
                    \setcounter{subfigure}{5}
                    \centering
                    \includegraphics[width=\textwidth]{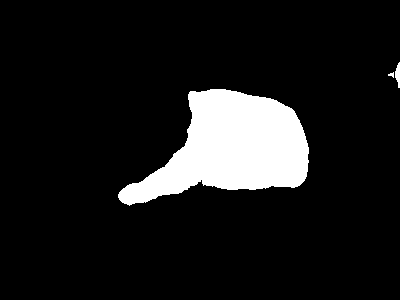}
                    \caption{PCE\\+NCut+KCut}
                \end{subfigure}
                &
                \begin{subfigure}[t]{0.15\textwidth}
                    \setcounter{subfigure}{6}
                    \centering
                    \includegraphics[width=\textwidth]{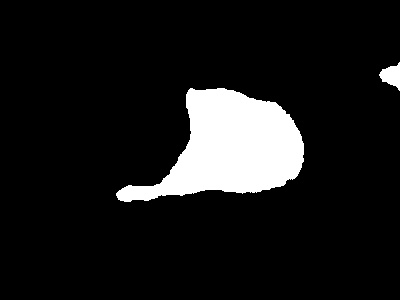}
                    \caption{\\PCE+CV}
                \end{subfigure}
                &
                \begin{subfigure}[t]{0.15\textwidth}
                    \setcounter{subfigure}{7}
                    \centering
                    \includegraphics[width=\textwidth]{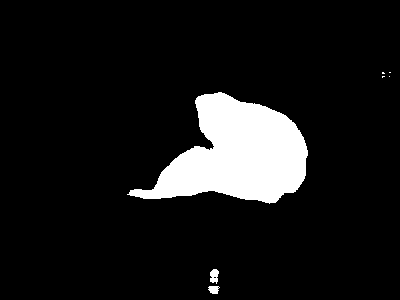}
                    \caption{\\\textbf{RKHS}}
                \end{subfigure}
                &
                \begin{subfigure}[t]{0.15\textwidth}
                    \setcounter{subfigure}{8}
                    \centering
                    \includegraphics[width=\textwidth]{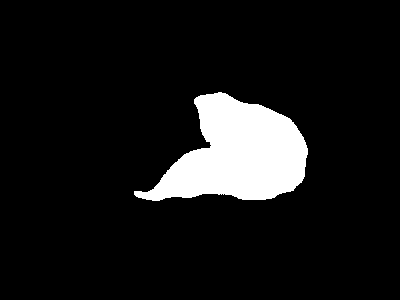}
                    \caption{\\\textbf{RKHS+TD}}
                \end{subfigure}
            \end{tabular}
        }
    \end{tabular}
    \addtolength{\tabcolsep}{5pt}
    \caption{A lemur.}
    \label{fig:pred_lemur}
\end{figure}

One more example is given in figure~\ref{fig:pred_flowers}. The 'PCE' model gives a slightly larger object phase than the preferred mask. The bottom edges of the flower are leaked to the boundary, which may be due to the effect of the shadow. The two 'NCut' models can give a slightly smaller object phase. However, the 'PCE $+$ TD' model has no improvement and presents a heavy leakage. This can be because the vanilla TD regularization does not count the boundary as the perimeter. Extending the object edges such that it aligned with the boundary reduces the TD loss. While giving larger object phases are not penalized in 'PCE', it may have mistakenly learn to extend and leak the boundary. Our 'RKHS $+$ TD' models have a smaller chance to suffer from this problem, since larger object phases will be penalized by our data-fidelity term. 

\begin{figure}[!ht]
    \centering
    \addtolength{\tabcolsep}{-5pt}
    \begin{tabular}{ccccc}
        \begin{subfigure}[t]{0.15\textwidth}
            \setcounter{subfigure}{0}
            \centering
            \includegraphics[width=\textwidth]{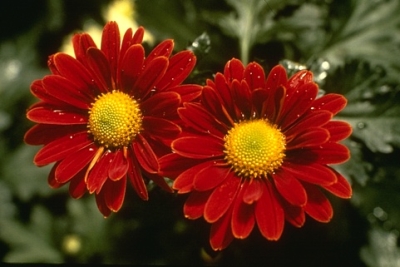}
            \caption{\\Image}
        \end{subfigure}
        &
        \begin{subfigure}[t]{0.15\textwidth}
            \setcounter{subfigure}{1}
            \centering
            \includegraphics[width=\textwidth]{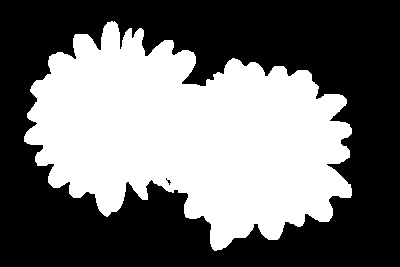}
            \caption{\\Mask}
        \end{subfigure}
        &
        \begin{subfigure}[t]{0.15\textwidth}
            \setcounter{subfigure}{2}
            \centering
            \includegraphics[width=\textwidth]{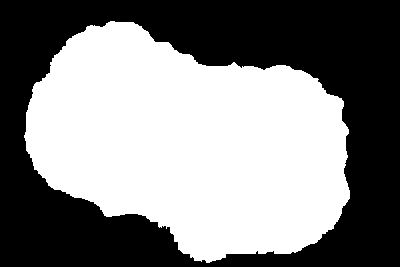}
            \caption{\\PCE}
        \end{subfigure}
        &
        \begin{subfigure}[t]{0.15\textwidth}
            \setcounter{subfigure}{3}
            \centering
            \includegraphics[width=\textwidth]{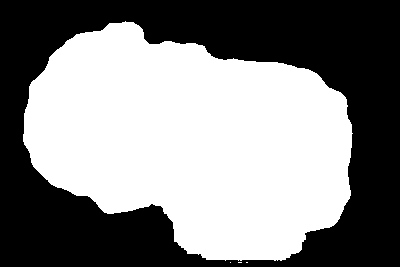}
            \caption{\\PCE+TD}
        \end{subfigure}
        &
        \begin{subfigure}[t]{0.15\textwidth}
            \setcounter{subfigure}{4}
            \centering
            \includegraphics[width=\textwidth]{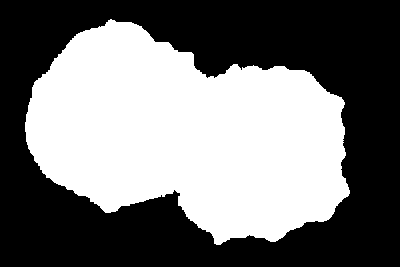}
            \caption{\\PCE+NCut}
        \end{subfigure}
        \\[0.5em]
        \multicolumn{5}{c}{
            \begin{tabular}{cccc}
                \begin{subfigure}[t]{0.15\textwidth}
                    \setcounter{subfigure}{5}
                    \centering
                    \includegraphics[width=\textwidth]{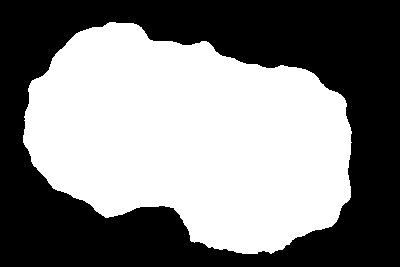}
                    \caption{PCE\\+NCut+KCut}
                \end{subfigure}
                &
                \begin{subfigure}[t]{0.15\textwidth}
                    \setcounter{subfigure}{6}
                    \centering
                    \includegraphics[width=\textwidth]{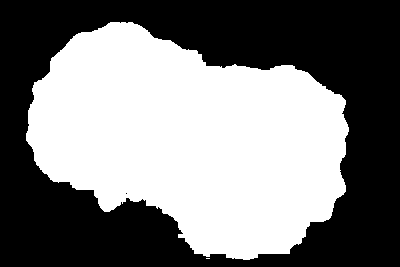}
                    \caption{\\PCE+CV}
                \end{subfigure}
                &
                \begin{subfigure}[t]{0.15\textwidth}
                    \setcounter{subfigure}{7}
                    \centering
                    \includegraphics[width=\textwidth]{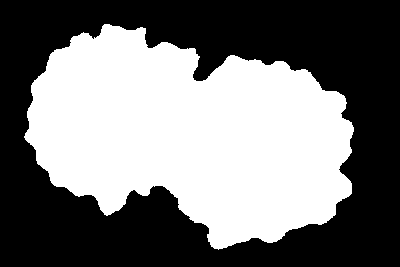}
                    \caption{\\\textbf{RKHS}}
                \end{subfigure}
                &
                \begin{subfigure}[t]{0.15\textwidth}
                    \setcounter{subfigure}{8}
                    \centering
                    \includegraphics[width=\textwidth]{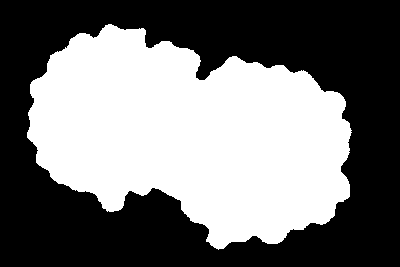}
                    \caption{\\\textbf{RKHS+TD}}
                \end{subfigure}
            \end{tabular}
        }
    \end{tabular}
    \addtolength{\tabcolsep}{5pt}
    \caption{Flowers}
    \label{fig:pred_flowers}
\end{figure}

%=======================Conclusion======================
\section{Conclusion} \label{Sec:Conclusion}
Motivated by sparse-label segmentation in both iterative and weakly supervised settings, we proposed a unified variational framework based on a simplex-constrained Potts model, smooth threshold-dynamics (TD) perimeter approximation, and RKHS-based label extension. The resulting energy is smooth and convex, and can be optimized efficiently by iterative threshold dynamics or gradient-based solvers without alternating projection.
The RKHS extension enables controllable regularization of region distributions through kernel design. Experiments in Section~\ref{Sec:Single_Image_Segmentation} show that the method handles inhomogeneous intensity distributions and remains robust in challenging cases, including water splashes, illumination bias, and strong noise.
We further derived the variational model into a weakly supervised training loss (Section~\ref{sec:loss_function}). The observed training effects (Section~\ref{sec:network_training})---denoising, object-edge refinement, and boundary-artifact elimination---indicate a strong synergy between variational modeling and deep learning. 
Notably, our method still exhibits several limitations for practical network training. On the one hand, the proposed strategy introduces mild additional computational and storage costs that can be non-negligible for a large dataset. On the other hand, the choice of kernel and weight parameters are still empirical. These suggest several future extensions of our work, such as a sparse function extension algorithm and an adaptive parameter choosing method.

\section*{Funding information}
Lingfeng Li is partially supported by HKRGC Grants No.LU13300125.
Sung Ha Kang is partially supported by Simons Foundation grant 584960.
Xue-Cheng Tai is partially supported by RCN project SAPPHIRE 358034 and NORCE Kompetanseoppbygging program.

\section*{Author contribution}
\noindent
Y.C. did algorithm design, programming, numerical experiments, and writing. L.L. participated in the writing. S.K. participated in the algorithm design. J.Z. participated in some discussion. X.T. participated in the algorithm design and writing.

\section*{Data availability}
The data being used in this paper are publicly available and can be found in \url{https://www.cse.cuhk.edu.hk/leojia/projects/hsaliency/dataset.html} (ECSSD) and \url{https://www.robots.ox.ac.uk/~vgg/projects/pascal/VOC/voc2012/index.html#data} ( PASCAL VOC 2012).  

% {\color{red} 
% Here, we list several possible directions for future extensions. First, in this manuscript, we only consider the weakly-supervised learning using pixel-level information, it can be further combined with image-level information like those used in \cite{kim2019mumford,liang2022tree,lin2016scribblesup,tang2018normalized,tang2018regularized,wu2023sparsely,zhang2023weakly} to further enhance the segmentation accuracy. 
% Second, when training the network using the 'RKHS+TD' loss, we apply a constant $\lambda$ to all images, we will try to use an extra network module to infer an image-dependent $\lambda$ to each sample, and design corresponding training algorithms.}
%
% \bibliographystyle{splncs04} %This citation style cannot use blanket in \cite. i.e. \cite{a,b} not \cite{a, b}
\bibliography{ref}
%

%====================Appendix===========================

\section*{Appendix}
\appendix

\section{Reproducing Kernel Hilbert Space}\label{sec:rkhs}

Let $D$ be a nonempty set, let $\mathcal V$ be a real Hilbert space with inner product $\langle\cdot,\cdot\rangle_{\mathcal V}$, and let $\mathcal V^D$ denote the set of all functions from $D$ to $\mathcal V$.  
Denote by $\mathcal B(\mathcal V)$ the Banach space of bounded linear operators on $\mathcal V$.

Assume $\mathcal K:D\times D\to\mathcal B(\mathcal V)$ is an operator-valued kernel.  
We call $\mathcal K$ symmetric if
\begin{align}
\mathcal K(x,y)=\mathcal K(y,x)^*, \qquad \forall (x,y)\in D\times D,
\label{eq:symmetric_kernel}
\end{align}
and positive definite if
\begin{align*}
\sum_{i,j=1}^N \big\langle \nu_i,\mathcal K(x_i,x_j)\nu_j\big\rangle_{\mathcal V}> 0
\end{align*}
for every finite set $\{x_i\}_{i=1}^N\subset D$, non-zero set $\{\nu_i\}_{i=1}^N\subset\mathcal V$, and $N\in\mathbb N$.

For each $x\in D$ and $\nu\in\mathcal V$, define
\[
\mathcal K_x\nu(\cdot):=\mathcal K(\cdot,x)\nu\in\mathcal V^D.
\]
Consider
\[
\mathcal H_0:=\mathrm{span}\{\mathcal K_x\nu:\ x\in D,\ \nu\in\mathcal V\}\subset\mathcal V^D.
\]
For
\[
\psi=\sum_{i=1}^N\mathcal K_{x_i}\nu_i,\qquad
\phi=\sum_{j=1}^M\mathcal K_{y_j}\omega_j\in\mathcal H_0,
\]
define
\begin{align}
\langle \psi,\phi\rangle_{\mathcal H_0}
:=\sum_{i=1}^N\sum_{j=1}^M
\big\langle \nu_i,\mathcal K(x_i,y_j)\omega_j\big\rangle_{\mathcal V}.
\label{eq:rkhs_innerproduct}
\end{align}
The completion of $\mathcal H_0$ under this inner product is a Hilbert space, denoted by $\mathcal H_{\mathcal K}$.

By the Moore--Aronszajn theorem \cite{aronszajn1950theory}, $\mathcal H_{\mathcal K}$ is the unique RKHS with reproducing kernel $\mathcal K$, i.e.,
\begin{align}
\langle \psi(x),\nu\rangle_{\mathcal V}
=
\langle \psi,\mathcal K_x\nu\rangle_{\mathcal H_{\mathcal K}},
\qquad
\forall \psi\in\mathcal H_{\mathcal K},\ \nu\in\mathcal V.
\label{eq:reproducing_property}
\end{align}
For instance, if $\psi=\sum_{i=1}^{\infty}\mathcal K_{x_i}\nu_i\in\mathcal H_{\mathcal K}$, then
\begin{align*}
\langle \psi(x),\nu\rangle_{\mathcal V}
&=\left\langle \sum_{i=1}^{\infty}\mathcal K(x,x_i)\nu_i,\nu\right\rangle_{\mathcal V}
=\sum_{i=1}^{\infty}\langle \nu_i,\mathcal K(x_i,x)\nu\rangle_{\mathcal V} \\
&=\langle \psi,\mathcal K_x\nu\rangle_{\mathcal H_{\mathcal K}}.
\end{align*}

This also implies boundedness of the evaluation map $\mathcal E_x:\mathcal H_{\mathcal K}\to\mathcal V$, $\mathcal E_x\psi=\psi(x)$:
\begin{align*}
\|\mathcal K_x\nu\|_{\mathcal H_{\mathcal K}}^2
&=\langle \nu,\mathcal K(x,x)\nu\rangle_{\mathcal V}
\le \|\mathcal K(x,x)\|\,\|\nu\|_{\mathcal V}^2,\\
\|\psi(x)\|_{\mathcal V}
&\le \|\mathcal K_x^*\|\,\|\psi\|_{\mathcal H_{\mathcal K}}
\le \sqrt{\|\mathcal K(x,x)\|}\,\|\psi\|_{\mathcal H_{\mathcal K}}.
\end{align*}
Hence, if
\[
\kappa:=\sup_{x\in D}\sqrt{\|\mathcal K(x,x)\|}<\infty,
\]
then
\[
\|\psi\|_{\infty}\le \kappa\|\psi\|_{\mathcal H_{\mathcal K}}.
\]

In the rest of the paper, ``reproducing kernel'' means an operator-valued kernel $\mathcal K$ that is symmetric and positive definite.

\section{Proofs}

\subsection{Proof of Lemma~\ref{lem:proj1}}
\begin{proof}
Since \eqref{eq:simplex_proj} is convex and differentiable, we use KKT conditions.  
Its Lagrangian is
\begin{align*}
\mathcal L(\mathbf u,\xi,\zeta)
=
\frac12\|\mathbf u-\Psi^\gamma\|_2^2
+\xi\!\left(\sum_{k=1}^K u_k-1\right)
-\langle \zeta,\mathbf u\rangle,
\end{align*}
where $\xi\in\mathbb R$ and $\zeta=(\zeta_1,\ldots,\zeta_K)\in\mathbb R_+^K$.

Stationarity and complementary slackness give
\[
u_k-\Psi_k^\gamma+\xi-\zeta_k=0,\qquad
\zeta_k u_k=0,\qquad
u_k\ge 0.
\]
Hence
\[
u_k^*=\max\{\Psi_k^\gamma-\xi,0\}.
\]
Let $\rho=\max\{j\in[K]:u_{(j)}^*>0\}$. Then
\begin{align*}
1
=\sum_{q=1}^K u_q^*
=\sum_{q=1}^{\rho}u_{(q)}^*
=\sum_{q=1}^{\rho}\left(\Psi_{(q)}^\gamma-\xi\right),
\end{align*}
so
\begin{align*}
\xi=\frac1{\rho}\left(\sum_{q=1}^{\rho}\Psi_{(q)}^\gamma-1\right).
\end{align*}
This proves \eqref{eq:proj_lemma1_u}.
\end{proof}

\subsection{Proof of Lemma~\ref{lem:proj2}}
\begin{proof}
From the definition
\[
\rho=\max\{j\in[K]:u_{(j)}^*>0\},
\]
we have
\begin{align*}
\Psi_{(\rho)}^\gamma
>
\frac1{\rho}\left(\sum_{q=1}^{\rho}\Psi_{(q)}^\gamma-1\right),
\qquad
\Psi_{(\rho+1)}^\gamma
\le
\frac1{\rho}\left(\sum_{q=1}^{\rho}\Psi_{(q)}^\gamma-1\right).
\end{align*}
For $r=1,\ldots,\rho-1$,
\begin{align*}
\Psi_{(\rho-r)}^\gamma
\ge \Psi_{(\rho)}^\gamma
>
\frac1{\rho-r}\left(\sum_{q=1}^{\rho-r}\Psi_{(q)}^\gamma-1\right).
\end{align*}
Also,
\begin{align}
\Psi_{(\rho+1)}^\gamma\le \frac1{\rho}\left(\sum_{q=1}^{\rho}\Psi_{(q)}^\gamma-1\right)
\iff
\Psi_{(\rho+1)}^\gamma\le \frac1{\rho+1}\left(\sum_{q=1}^{\rho+1}\Psi_{(q)}^\gamma-1\right).
\label{eq:proj_lemma2_pf_rho_1}
\end{align}
Applying \eqref{eq:proj_lemma2_pf_rho_1} inductively yields, for $r=1,\ldots,K-\rho$,
\begin{align*}
\Psi_{(\rho+r)}^\gamma
\le
\frac1{\rho+r}\left(\sum_{q=1}^{\rho+r}\Psi_{(q)}^\gamma-1\right).
\end{align*}
Therefore
\[
\rho=
\max\left\{
j\in[K]:
\Psi_{(j)}^\gamma-\frac1j\left(\sum_{q=1}^{j}\Psi_{(q)}^\gamma-1\right)>0
\right\}.
\]
The converse direction follows similarly from the monotone ordering
\[
\Psi_The {(1)}^\gamma\ge\cdots\ge\Psi_{(K)}^\gamma.
\]
\end{proof}

\section{The Threshold Dynamics Algorithm} \label{sec:td_implementation}

Consider the energy
\begin{align*}
\mathcal{E}(\mathbf{v})
=
\sum_{k=1}^K \int_\Omega \bigl(1-2u_k(x)\bigr)v_k(x)\,dx
+\lambda \sum_{k=1}^K \int_\Omega \bigl(1-v_k(x)\bigr)\bigl(G_\sigma * v_k\bigr)(x)\,dx .
\end{align*}
Since $\mathcal{E}$ is convex, at iteration $t$ we update $\mathbf v^{(t+1)}$ by minimizing the affine majorization (supporting hyperplane) at $\mathbf v^{(t)}$:
\begin{align*}
\mathbf{v}^{(t+1)}
=
\arg\min_{\mathbf{v}\in\tilde{\Delta}^K}
\sum_{k=1}^K \int_\Omega g_k^{(t)}(x)\,v_k(x)\,dx ,
\end{align*}
where
\begin{align*}
g_k^{(t)}(x)
=
1-2u_k(x)+\lambda\bigl(G_\sigma*(1-v_k^{(t)})\bigr)(x).
\end{align*}
Hence, pointwise in $x$, the update is
\begin{align}
\label{eq:td_update}
v_k^{(t+1)}(x)=
\begin{cases}
1, & \text{if } k\in\arg\min_{k'} g_{k'}^{(t)}(x),\\
0, & \text{otherwise}.
\end{cases}
\end{align}
Algorithm~\ref{alg:TD} summarizes the full procedure. Here $\mathbbm{1}_{\mathcal X}$ denotes the indicator function of a set $\mathcal X$. The thresholded $\mathbf u$ is used for initialization. For binary segmentation, we rewrite \eqref{eq:main_model_binary} as \eqref{eq:main_model} with $K=2$ by setting $u_1=u_0$ and $u_2=1-u_0$; then the final binary mask is $v_0=v_1$.

This model and algorithm were introduced in \cite{Liu2011} and refined in \cite{wang2017efficient}. Monotone energy decay and convergence to a minimizer were proved in \cite[Proposition 2]{Liu2011} and \cite[Theorem 2.2]{wang2017efficient}.

\begin{algorithm}[!ht]
\caption{The Threshold Dynamics Method \cite{Liu2011,wang2017efficient}}
\label{alg:TD}
\SetKwInOut{Input}{Input}
\SetKw{Return}{Return:}
\SetKwComment{Comment}{/* }{ */}
\Input{Label-based fuzzy membership function $\mathbf{u}$; weight $\lambda$; kernel scale $\sigma$.}
Initialize
\[
v_k^{(0)}=\mathbbm{1}_{\{x:\,u_k(x)>\max_{k'\neq k}u_{k'}(x)\}},\quad k=1,\ldots,K.
\]
\For{$t=0,1,\ldots$}{
    \For{$k=1,\ldots,K$}{
        $g_k^{(t)}(x)=1-2u_k(x)+\lambda\bigl(G_\sigma*(1-v_k^{(t)})\bigr)(x)$\;
        $\Omega_k^{(t+1)}=\{x:\,g_k^{(t)}(x)<\min_{k'\neq k}g_{k'}^{(t)}(x)\}$\;
        $v_k^{(t+1)}=\mathbbm{1}_{\Omega_k^{(t+1)}}$\;
    }
    Stop when a stopping criterion is met\;
}
\Return{$\mathbf{v}^{(t+1)}$}
\end{algorithm}

\section{Evaluation Metrics}\label{sec:evaluation_metrics}

We define evaluation metrics used in Section~\ref{Sec:Network}. Let $\hat{\psi}_i=(\hat{\psi}_{i,1},\ldots,\hat{\psi}_{i,K})$ be the discretized ground-truth label for image $i$, with $\hat{\psi}_{i,k}\in\{0,1\}^{h\times w}$. The predicted mask is obtained by thresholding the network output:
\[
N_{\theta,k}(\hat{I}_i)[j] = \mathbbm{1}\!\left[k = \arg\max_{k'} N_{\theta,k'}(\hat{I}_i)[j]\right].
\]
Define the standard confusion counts for class $k$ and image $i$:
\begin{align*}
    \mathrm{TP}_{i,k} &= \bigl|\{j : N_{\theta,k}(\hat{I}_i)[j]=1,\ \hat{\psi}_{i,k}[j]=1\}\bigr|, \\
    \mathrm{FP}_{i,k} &= \bigl|\{j : N_{\theta,k}(\hat{I}_i)[j]=1,\ \hat{\psi}_{i,k}[j]=0\}\bigr|, \\
    \mathrm{FN}_{i,k} &= \bigl|\{j : N_{\theta,k}(\hat{I}_i)[j]=0,\ \hat{\psi}_{i,k}[j]=1\}\bigr|,
\end{align*}
for $i=1,\ldots,n$ and $k=1,\ldots,K$. The three metrics are defined as follows.

\paragraph{Mean Intersection over Union (mIoU)}
\begin{align*}
    \mathrm{mIoU} = \frac{1}{K}\sum_{k=1}^K\frac{\sum_{i=1}^n\mathrm{TP}_{i,k}}{\sum_{i=1}^n\left(\mathrm{TP}_{i,k}+\mathrm{FP}_{i,k}+\mathrm{FN}_{i,k}\right)}.
\end{align*}

\paragraph{Mean Dice Coefficient (mDice)}
\begin{align*}
    \mathrm{mDice} = \frac{1}{K}\sum_{k=1}^K\frac{2\sum_{i=1}^n\mathrm{TP}_{i,k}}{\sum_{i=1}^n\left(2\,\mathrm{TP}_{i,k}+\mathrm{FP}_{i,k}+\mathrm{FN}_{i,k}\right)}.
\end{align*}

\paragraph{Mean Pixel Accuracy (mAcc)}
\begin{align*}
    \mathrm{mAcc} = \frac{1}{K}\sum_{k=1}^K\frac{\sum_{i=1}^n\mathrm{TP}_{i,k}}{\sum_{i=1}^n\left(\mathrm{TP}_{i,k}+\mathrm{FN}_{i,k}\right)}.
\end{align*}

\section{Loss Functions of the Comparing Methods}\label{sec:network_energy_compare}

This appendix collects the training objectives used in Table~\ref{tab:network_all}.
All methods share the same network $N_\theta$ and differ only in the loss.

\subsection{Common notation}

Let $\{(\hat I_i,\hat\psi_i)\}_{i=1}^n$ be the training set, where $\hat I_i$ is an image on
$\hat\Omega$ of size $h\times w$, and $\hat\psi_i=(\hat\psi_{i,1},\ldots,\hat\psi_{i,K})$
is the (possibly sparse) label, defined only on $\hat D_i\subseteq\hat\Omega$.
The network output is
\[
\hat p_i := N_\theta(\hat I_i)=(\hat p_{i,1},\ldots,\hat p_{i,K}),
\qquad
\hat p_{i,k}\in[0,1]^{h\times w},\ \sum_{k=1}^K \hat p_{i,k}=1.
\]
For a set $\hat D\subseteq\hat\Omega$, define the discrete inner product
\[
\langle A,B\rangle_{\hat D}
:=\sum_{[l,j]\in\hat D} A[l,j]\,B[l,j].
\]

\paragraph{TD regularizer (per image).}
We use the same smooth Potts/TD perimeter surrogate as in \eqref{eq:Potts_rkhs_loss}:
\[
\mathrm{TD}_i(\theta)
:=
\sum_{k=1}^K
\Bigl\langle
1-\hat p_{i,k},\,
\hat G_\sigma \mathbin{\hat *}\hat p_{i,k}
\Bigr\rangle_{\mathbb R^{h\times w}}.
\]
When included in a loss, we average it over images:
\[
\mathrm{TD}(\theta):=\frac1n\sum_{i=1}^n \mathrm{TD}_i(\theta).
\]

\subsection{Fully supervised baseline}

\subsubsection{Cross-Entropy (CE)}
When full labels are available ($\hat D_i=\hat\Omega$), we use standard cross-entropy:
\begin{align}
\label{eq:ce_loss}
\mathrm{CE}(\theta)
=
-\frac1n\sum_{i=1}^n \frac{1}{hw}\sum_{k=1}^K
\Bigl\langle
\hat\psi_{i,k},\,\log \hat p_{i,k}
\Bigr\rangle_{\mathbb R^{h\times w}}.
\end{align}
We also report \textbf{CE + TD}, defined as $\mathrm{CE}(\theta)+\lambda\,\mathrm{TD}(\theta)$.

\subsection{Weakly supervised baselines (scribbles / sparse pixels)}

\subsubsection{Partial Cross-Entropy (PCE)}
The standard weakly supervised loss evaluates cross-entropy only on labeled pixels:
\[
\mathrm{PCE}(\theta)
=
-\frac1n\sum_{i=1}^n \frac{1}{|\hat D_i|}\sum_{k=1}^K
\Bigl\langle
\hat\psi_{i,k},\,\log \hat p_{i,k}
\Bigr\rangle_{\hat D_i}.
\]
Outside $\hat D_i$, this term imposes no explicit regularization.

\subsubsection{PCE + TD}
We add the TD regularizer to PCE:
\[
\mathrm{PCE}(\theta)+\lambda\,\mathrm{TD}(\theta).
\]

\subsubsection{PCE + NCut \cite{tang2018normalized}}
This variant adds a normalized-cut style regularizer \cite{shi2000normalized}:
\[
\mathrm{PCE}(\theta)+\lambda\,\mathrm{NCut}(\theta),
\]
\[\mathrm{NCut}(\theta)
=
\frac1n\sum_{i=1}^n
\sum_{k=1}^K
\frac{
\displaystyle\sum_{l,l'}\sum_{j,j'}
\hat{\mathcal G}_{\sigma_I,\sigma_s}(\hat I_i)[l,j,l',j']
\bigl(1-\hat p_{i,k}[l,j]\bigr)\hat p_{i,k}[l',j']
}{
\displaystyle\sum_{l,j}\hat p_{i,k}[l,j]
}.\]
The affinity kernel is the (discrete) bilateral weight
\begin{align}
\label{eq:ncut_loss}
\hat{\mathcal G}_{\sigma_I,\sigma_s}(\hat I_i)[l,j,l',j']
=
\frac{1}{C}
\exp\!\left(
-\frac{\|\hat I_i[l,j]-\hat I_i[l',j']\|_2^2}{\sigma_{I}}
\right)
\exp\!\left(
-\frac{(l-l')^2+(j-j')^2}{\sigma_{s}}
\right),
\end{align}
with normalization constant $C$.

\subsubsection{PCE + NCut + KCut \cite{tang2018normalized}}
This adds an additional (unnormalized) cut term:
\[
\mathrm{PCE}(\theta)+\lambda\bigl(\mathrm{NCut}(\theta)+\mathrm{KCut}(\theta)\bigr),
\]
\[\mathrm{KCut}(\theta)
=
\frac1n\sum_{i=1}^n
\sum_{k=1}^K
\sum_{l,l'}\sum_{j,j'}
\hat{\mathcal G}_{\sigma_I,\sigma_s}(\hat I_i)[l,j,l',j']
\bigl(1-\hat p_{i,k}[l,j]\bigr)\hat p_{i,k}[l',j'].\]

\subsubsection{PCE + CV \cite{kim2019mumford}}
Following \cite{kim2019mumford}, we combine PCE with a Chan--Vese-type regularizer \cite{chan2001active}:
\[
\mathrm{PCE}(\theta)+\lambda\,\mathrm{CV}(\theta),
\]
where
\[
\mathrm{CV}(\theta)
=
\frac1n\sum_{i=1}^n
\sum_{k=1}^K
\left(
\frac{1}{hw}
\bigl\langle \hat I_i-c_{i,k},\,\hat p_{i,k}\bigr\rangle_{\mathbb R^{h\times w}}
+
\mu\sum_{l,j}\bigl\|\hat\nabla \hat p_{i,k}[l,j]\bigr\|_1
\right).
\]
Here
\[
c_{i,k}
=
\frac{
\langle \hat I_i,\hat p_{i,k}\rangle_{\mathbb R^{h\times w}}
}{
\sum_{l,j}\hat p_{i,k}[l,j]
},
\]
$\hat\nabla$ is the forward-difference gradient with zero Dirichlet boundary condition, and
$\mu=10^{-6}$ as in \cite{kim2019mumford}.

\section{Ablation Studies for Network Training}

\subsection{Weight of the TD term}\label{sec:network_edge_weight}

Table~\ref{tab:ablation_weight_edge} reports validation scores for RKHS + TD under different $\lambda$. The effect is nonlinear. We use $\lambda=5\times10^{-5}$ throughout Section~\ref{Sec:Network} as it gives the best mIoU and mDice.

\begin{table}[]
    \centering
    \begin{tabular}{c|cccc}
        \hline
        $\lambda$ & \hspace{2em}5E-4 \hspace{1em}& \hspace{1em}5E-5\hspace{1em} &\hspace{1em} 5E-6\hspace{1em} & \hspace{1em}5E-7\hspace{1em} \\
        \hline
        mIoU & 74.04 & \color{red}{75.35} & 75.19 & 75.30\\
        mDice & 84.52 & \color{red}{85.47} & 85.37 & 85.42  \\
        mAcc & 86.18 & 87.94 & \color{red}{88.26} & 87.53 \\
        \hline
    \end{tabular}
    \caption{Ablation study on the TD weight $\lambda$. Best score per row in \textcolor{red}{red}.}
    \label{tab:ablation_weight_edge}
\end{table}

\subsection{Training dynamics}\label{sec:network_training_dynamic}

Figure~\ref{fig:training_dynamic} shows loss and evaluation-score curves during the second training stage (with augmentation). Halving the learning rate at the stage boundary stabilizes training for all methods. The effect is most pronounced for PCE, where it produces a sharp drop in loss and reduces score oscillation. For RKHS and RKHS $+$ TD the improvement in stability is subtler but still visible.

\begin{figure}[!ht]
    \centering
    \begin{subfigure}[b]{0.28\textwidth}
        \centering
        \includegraphics[width=\textwidth]{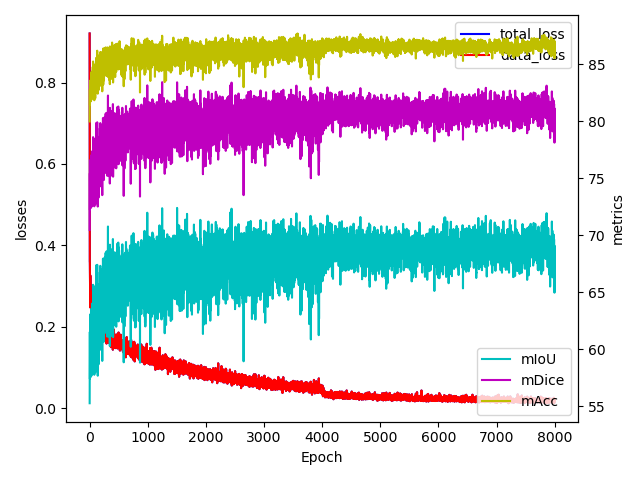} \\
        \caption{PCE}
    \end{subfigure}
    \begin{subfigure}[b]{0.28\textwidth}
        \centering
        \includegraphics[width=\textwidth]{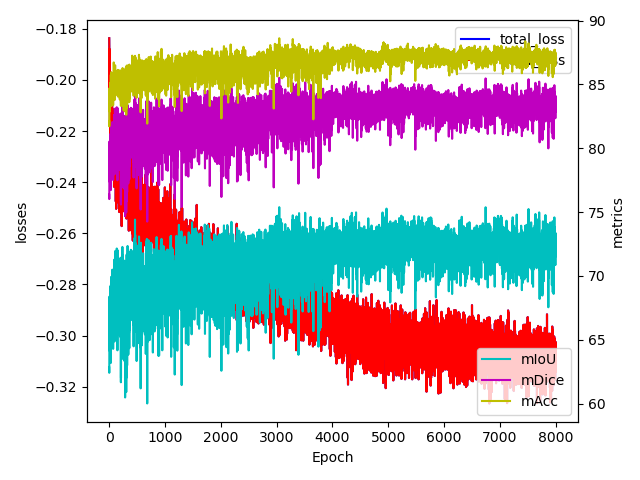} \\
        \caption{RKHS}
    \end{subfigure}
    \begin{subfigure}[b]{0.28\textwidth}
        \centering
        \includegraphics[width=\textwidth]{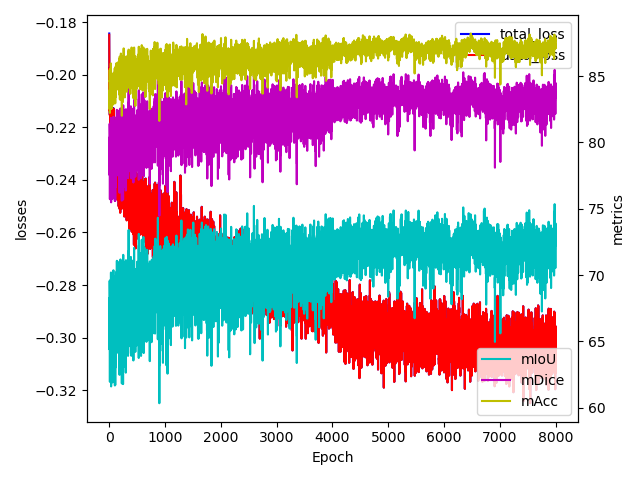} \\
        \caption{RKHS+TD}
    \end{subfigure}
    \caption{Loss and evaluation-score dynamics during the second training stage.
    \textcolor{cyan}{Cyan}: mIoU. \textcolor{violet}{Violet}: mDice. \textcolor{olive}{Olive}: mAcc.
    \textcolor{red}{Red}: data-fidelity loss. \textcolor{blue}{Blue}: total loss.}
    \label{fig:training_dynamic}
\end{figure}

\section{Additional Figures}

\subsection{Patch Intensity Distribution of Homogeneous Images}\label{sec:homo_img_dist}

Figure~\ref{fig:rabbit_dist} shows the patch intensity distributions for the two homogeneous images discussed in Section~\ref{sec:td_results}. In both cases, the foreground and background distributions overlap substantially, which makes standard data-fidelity terms ineffective. Our RKHS kernel exploits patch-level intensity and spatial cues to separate the two phases despite this overlap.

\begin{figure}[!ht]
    \centering
    \begin{subfigure}[b]{0.18\textwidth}
        \centering
        \includegraphics[width=\textwidth]{Figures/iterative_method/dolphin/0797.jpg} \\
        \includegraphics[width=\textwidth]{Figures/iterative_method/main/ILSVRC2012_test_00017483.jpg} \\
        \caption{Image}
    \end{subfigure}
    \begin{subfigure}[b]{0.26\textwidth}
        \centering
        \includegraphics[width=\textwidth]{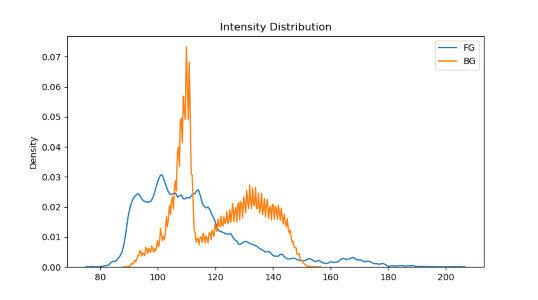} \\
        \includegraphics[width=\textwidth]{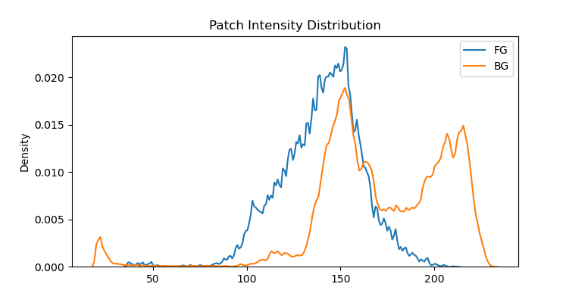} \\
        \caption{Distribution}
    \end{subfigure}
    \caption{Patch intensity distributions ($R_0=3$) for the dolphin (top) and rabbit (bottom) images. \textcolor{cyan}{Cyan}: foreground. \textcolor{orange}{Orange}: background. Large overlap makes phase separation difficult for standard methods.}
    \label{fig:rabbit_dist}
\end{figure}

\subsection{Network Prediction} \label{sec:network_pred_figs_appendix}

Figures~\ref{fig:pred_plane}--\ref{fig:pred_deer} provide additional qualitative comparisons. In general, PCE-based methods produce less precise masks, often including background noise or artifacts, especially in complex scenes (Figures~\ref{fig:pred_giraffe}, \ref{fig:pred_horse}). While regularizers like TD or NCut offer some improvement, they are often insufficient. In contrast, our RKHS-based methods consistently generate cleaner, more accurate segmentations that better adhere to true object boundaries. They successfully handle challenging cases, such as distinguishing the giraffe from a textured background (Figure~\ref{fig:pred_giraffe}) and preserving fine structures like the deer's antlers (Figure~\ref{fig:pred_deer}), demonstrating the superiority of the learned fuzzy membership function $\mathbf{u}$ as a supervisory signal.

\begin{figure}[!ht]
    \centering

    \addtolength{\tabcolsep}{-5pt}
    \begin{tabular}{ccccc}
        \begin{subfigure}[t]{0.15\textwidth}
            \setcounter{subfigure}{0}
            \centering
            \includegraphics[width=\textwidth]{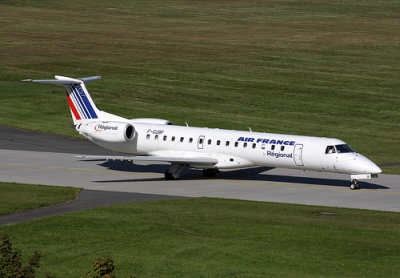}
            \caption{\\Image}
        \end{subfigure}
        &
        \begin{subfigure}[t]{0.15\textwidth}
            \setcounter{subfigure}{1}
            \centering
            \includegraphics[width=\textwidth]{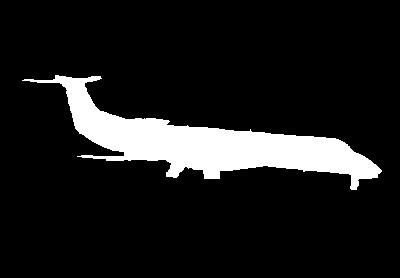}
            \caption{\\Mask}
        \end{subfigure}
        &
        \begin{subfigure}[t]{0.15\textwidth}
            \setcounter{subfigure}{2}
            \centering
            \includegraphics[width=\textwidth]{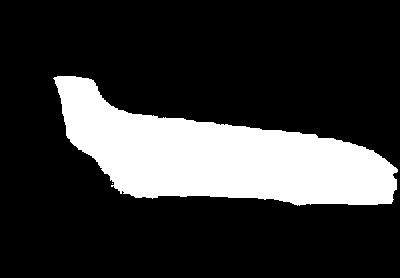}
            \caption{\\PCE}
        \end{subfigure}
        &
        \begin{subfigure}[t]{0.15\textwidth}
            \setcounter{subfigure}{3}
            \centering
            \includegraphics[width=\textwidth]{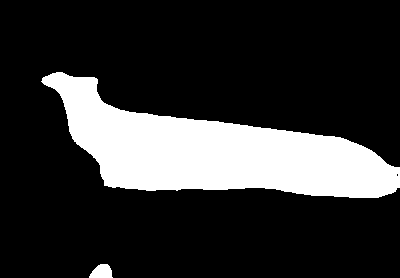}
            \caption{\\PCE+TD}
        \end{subfigure}
        &
        \begin{subfigure}[t]{0.15\textwidth}
            \setcounter{subfigure}{4}
            \centering
            \includegraphics[width=\textwidth]{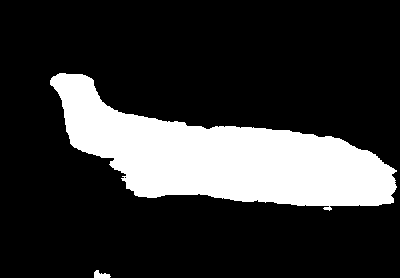}
            \caption{\\PCE+NCut}
        \end{subfigure}
        \\[0.5em]
        \multicolumn{5}{c}{
            \begin{tabular}{cccc}
                \begin{subfigure}[t]{0.15\textwidth}
                    \setcounter{subfigure}{5}
                    \centering
                    \includegraphics[width=\textwidth]{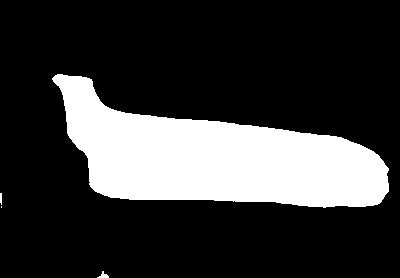}
                    \caption{PCE\\+NCut+KCut}
                \end{subfigure}
                &
                \begin{subfigure}[t]{0.15\textwidth}
                    \setcounter{subfigure}{6}
                    \centering
                    \includegraphics[width=\textwidth]{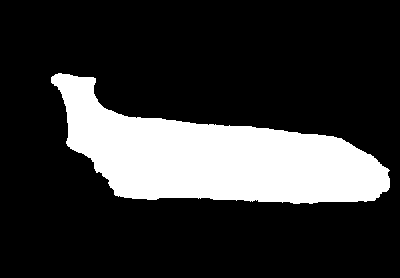}
                    \caption{\\PCE+CV}
                \end{subfigure}
                &
                \begin{subfigure}[t]{0.15\textwidth}
                    \setcounter{subfigure}{7}
                    \centering
                    \includegraphics[width=\textwidth]{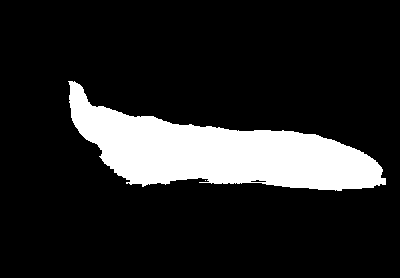}
                    \caption{\\\textbf{RKHS}}
                \end{subfigure}
                &
                \begin{subfigure}[t]{0.15\textwidth}
                    \setcounter{subfigure}{8}
                    \centering
                    \includegraphics[width=\textwidth]{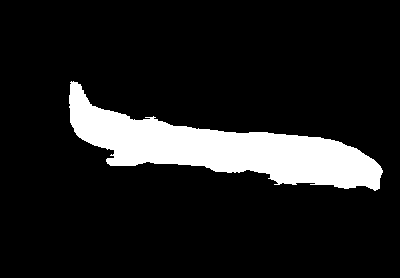}
                    \caption{\\\textbf{RKHS+TD}}
                \end{subfigure}
            \end{tabular}
        }
    \end{tabular}
    \addtolength{\tabcolsep}{5pt}
    \caption{Qualitative comparison on an image of a plane. PCE-based methods produce inaccurate masks. Our RKHS-based methods generate clean segmentations that closely match the ground truth mask.}
    \label{fig:pred_plane}
\end{figure}

\begin{figure}[!ht]
    \centering
    \addtolength{\tabcolsep}{-5pt}
    \begin{tabular}{ccccc}
        \begin{subfigure}[t]{0.15\textwidth}
            \setcounter{subfigure}{0}
            \centering
            \includegraphics[width=\textwidth]{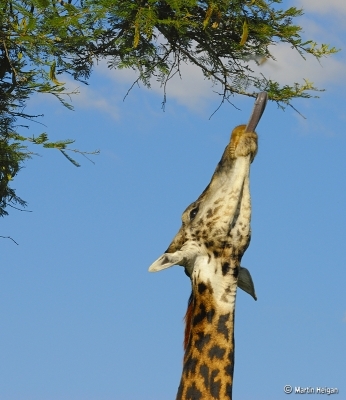}
            \caption{\\Image}
        \end{subfigure}
        &
        \begin{subfigure}[t]{0.15\textwidth}
            \setcounter{subfigure}{1}
            \centering
            \includegraphics[width=\textwidth]{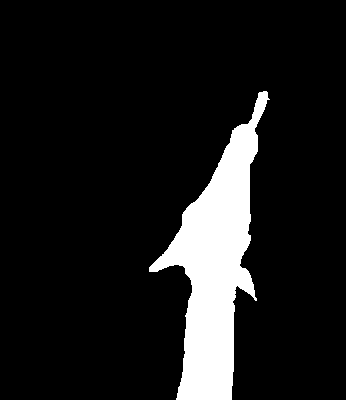}
            \caption{\\Mask}
        \end{subfigure}
        &
        \begin{subfigure}[t]{0.15\textwidth}
            \setcounter{subfigure}{2}
            \centering
            \includegraphics[width=\textwidth]{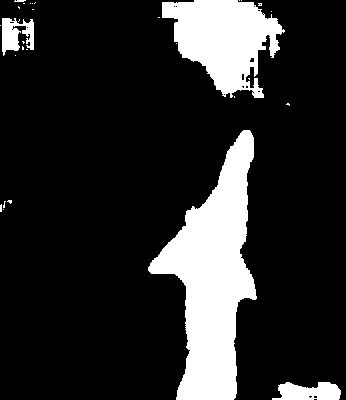}
            \caption{\\PCE}
        \end{subfigure}
        &
        \begin{subfigure}[t]{0.15\textwidth}
            \setcounter{subfigure}{3}
            \centering
            \includegraphics[width=\textwidth]{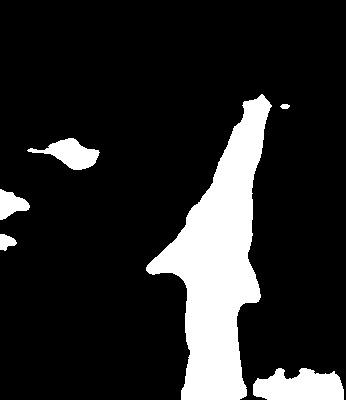}
            \caption{\\PCE+TD}
        \end{subfigure}
        &
        \begin{subfigure}[t]{0.15\textwidth}
            \setcounter{subfigure}{4}
            \centering
            \includegraphics[width=\textwidth]{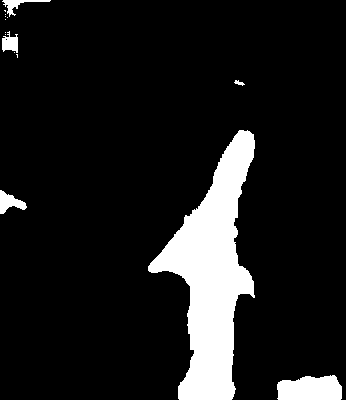}
            \caption{\\PCE+NCut}
        \end{subfigure}
        \\[0.5em]
        \multicolumn{5}{c}{
            \begin{tabular}{cccc}
                \begin{subfigure}[t]{0.15\textwidth}
                    \setcounter{subfigure}{5}
                    \centering
                    \includegraphics[width=\textwidth]{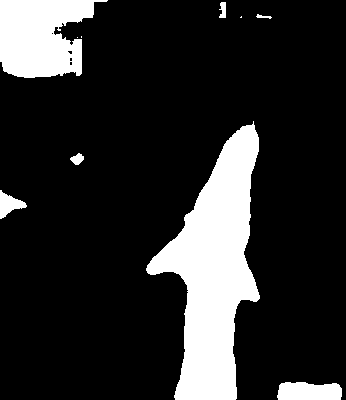}
                    \caption{PCE\\+NCut+KCut}
                \end{subfigure}
                &
                \begin{subfigure}[t]{0.15\textwidth}
                    \setcounter{subfigure}{6}
                    \centering
                    \includegraphics[width=\textwidth]{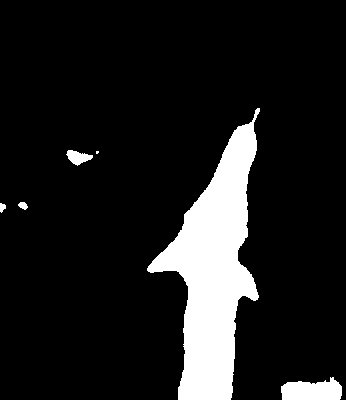}
                    \caption{\\PCE+CV}
                \end{subfigure}
                &
                \begin{subfigure}[t]{0.15\textwidth}
                    \setcounter{subfigure}{7}
                    \centering
                    \includegraphics[width=\textwidth]{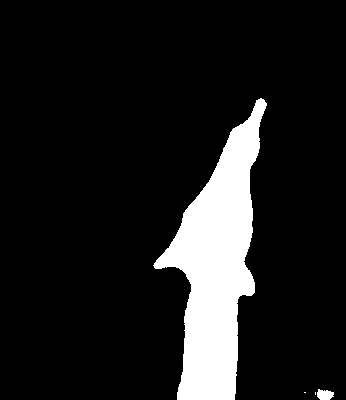}
                    \caption{\\\textbf{RKHS}}
                \end{subfigure}
                &
                \begin{subfigure}[t]{0.15\textwidth}
                    \setcounter{subfigure}{8}
                    \centering
                    \includegraphics[width=\textwidth]{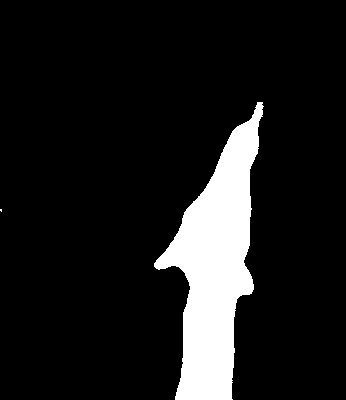}
                    \caption{\\\textbf{RKHS+TD}}
                \end{subfigure}
            \end{tabular}
        }
    \end{tabular}
    \addtolength{\tabcolsep}{5pt}
    \caption{Qualitative comparison on an image of a giraffe. The background has a similar texture to the object. PCE-based methods incorrectly segment large parts of the background, while our RKHS-based methods are more precise.}
    \label{fig:pred_giraffe}
\end{figure}

\begin{figure}[!ht]
    \centering
    \addtolength{\tabcolsep}{-5pt}
    \begin{tabular}{ccccc}
        \begin{subfigure}[t]{0.15\textwidth}
            \setcounter{subfigure}{0}
            \centering
            \includegraphics[width=\textwidth]{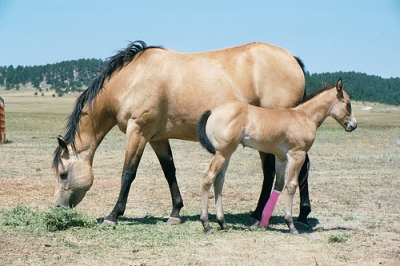}
            \caption{\\Image}
        \end{subfigure}
        &
        \begin{subfigure}[t]{0.15\textwidth}
            \setcounter{subfigure}{1}
            \centering
            \includegraphics[width=\textwidth]{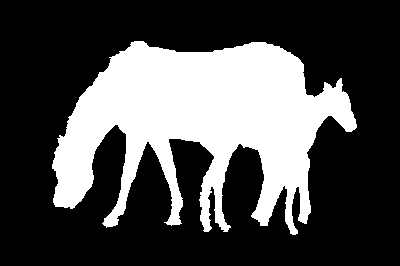}
            \caption{\\Mask}
        \end{subfigure}
        &
        \begin{subfigure}[t]{0.15\textwidth}
            \setcounter{subfigure}{2}
            \centering
            \includegraphics[width=\textwidth]{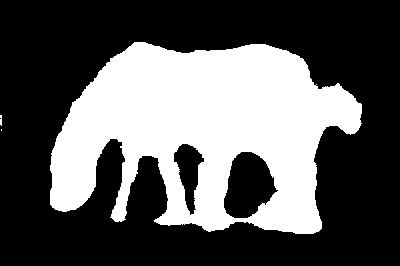}
            \caption{\\PCE}
        \end{subfigure}
        &
        \begin{subfigure}[t]{0.15\textwidth}
            \setcounter{subfigure}{3}
            \centering
            \includegraphics[width=\textwidth]{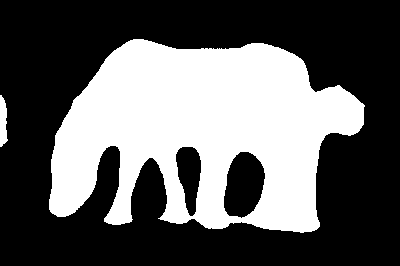}
            \caption{\\PCE+TD}
        \end{subfigure}
        &
        \begin{subfigure}[t]{0.15\textwidth}
            \setcounter{subfigure}{4}
            \centering
            \includegraphics[width=\textwidth]{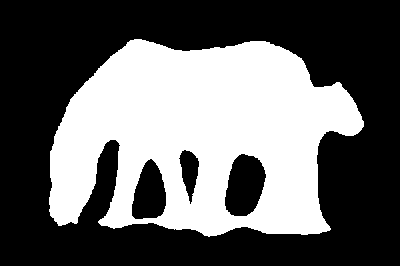}
            \caption{\\PCE+NCut}
        \end{subfigure}
        \\[0.5em]
        \multicolumn{5}{c}{
            \begin{tabular}{cccc}
                \begin{subfigure}[t]{0.15\textwidth}
                    \setcounter{subfigure}{5}
                    \centering
                    \includegraphics[width=\textwidth]{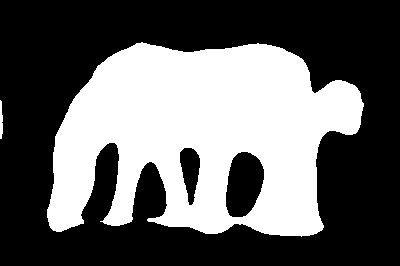}
                    \caption{PCE\\+NCut+KCut}
                \end{subfigure}
                &
                \begin{subfigure}[t]{0.15\textwidth}
                    \setcounter{subfigure}{6}
                    \centering
                    \includegraphics[width=\textwidth]{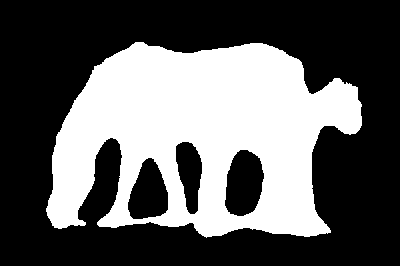}
                    \caption{\\PCE+CV}
                \end{subfigure}
                &
                \begin{subfigure}[t]{0.15\textwidth}
                    \setcounter{subfigure}{7}
                    \centering
                    \includegraphics[width=\textwidth]{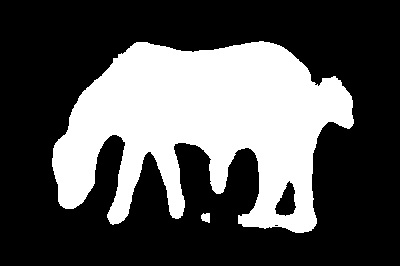}
                    \caption{\\\textbf{RKHS}}
                \end{subfigure}
                &
                \begin{subfigure}[t]{0.15\textwidth}
                    \setcounter{subfigure}{8}
                    \centering
                    \includegraphics[width=\textwidth]{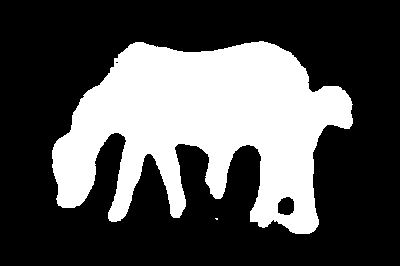}
                    \caption{\\\textbf{RKHS+TD}}
                \end{subfigure}
            \end{tabular}
        }
    \end{tabular}
    \addtolength{\tabcolsep}{5pt}
    \caption{Qualitative comparison on an image with houses. The scene contains complex structures and textures. Our RKHS-based methods produce a cleaner and more focused segmentation of the salient object.}
    \label{fig:pred_horse}
\end{figure}

\begin{figure}[!ht]
    \centering
    \addtolength{\tabcolsep}{-5pt}
    \begin{tabular}{ccccc}
        \begin{subfigure}[t]{0.15\textwidth}
            \setcounter{subfigure}{0}
            \centering
            \includegraphics[width=\textwidth]{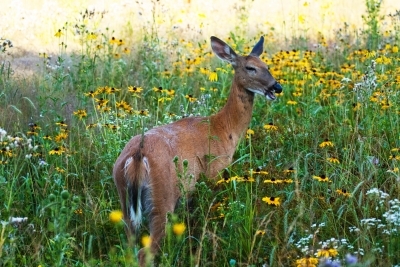}
            \caption{\\Image}
        \end{subfigure}
        &
        \begin{subfigure}[t]{0.15\textwidth}
            \setcounter{subfigure}{1}
            \centering
            \includegraphics[width=\textwidth]{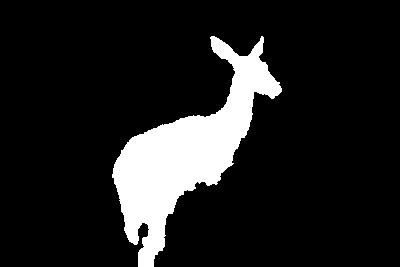}
            \caption{\\Mask}
        \end{subfigure}
        &
        \begin{subfigure}[t]{0.15\textwidth}
            \setcounter{subfigure}{2}
            \centering
            \includegraphics[width=\textwidth]{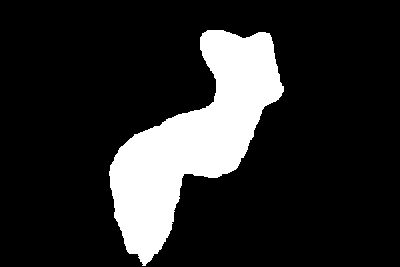}
            \caption{\\PCE}
        \end{subfigure}
        &
        \begin{subfigure}[t]{0.15\textwidth}
            \setcounter{subfigure}{3}
            \centering
            \includegraphics[width=\textwidth]{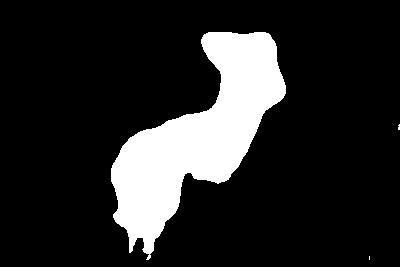}
            \caption{\\PCE+TD}
        \end{subfigure}
        &
        \begin{subfigure}[t]{0.15\textwidth}
            \setcounter{subfigure}{4}
            \centering
            \includegraphics[width=\textwidth]{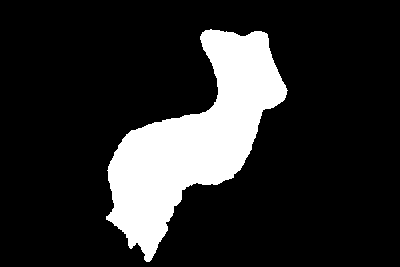}
            \caption{\\PCE+NCut}
        \end{subfigure}
        \\[0.5em]
        \multicolumn{5}{c}{
            \begin{tabular}{cccc}
                \begin{subfigure}[t]{0.15\textwidth}
                    \setcounter{subfigure}{5}
                    \centering
                    \includegraphics[width=\textwidth]{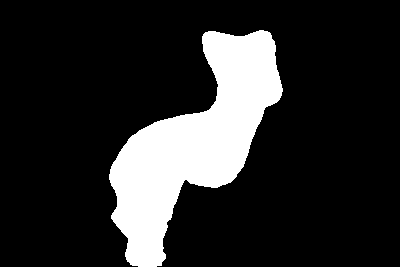}
                    \caption{PCE\\+NCut+KCut}
                \end{subfigure}
                &
                \begin{subfigure}[t]{0.15\textwidth}
                    \setcounter{subfigure}{6}
                    \centering
                    \includegraphics[width=\textwidth]{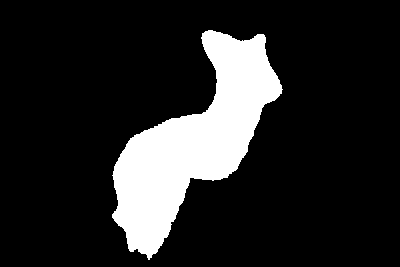}
                    \caption{\\PCE+CV}
                \end{subfigure}
                &
                \begin{subfigure}[t]{0.15\textwidth}
                    \setcounter{subfigure}{7}
                    \centering
                    \includegraphics[width=\textwidth]{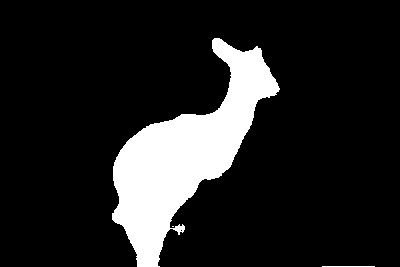}
                    \caption{\\\textbf{RKHS}}
                \end{subfigure}
                &
                \begin{subfigure}[t]{0.15\textwidth}
                    \setcounter{subfigure}{8}
                    \centering
                    \includegraphics[width=\textwidth]{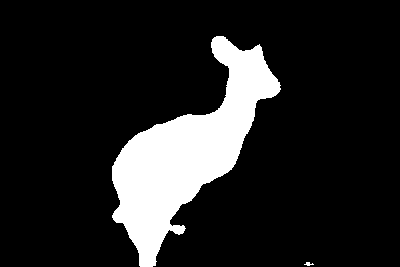}
                    \caption{\\\textbf{RKHS+TD}}
                \end{subfigure}
            \end{tabular}
        }
    \end{tabular}
    \addtolength{\tabcolsep}{5pt}
    \caption{Qualitative comparison on an image of a deer. The object has thin structures (antlers) that are challenging to segment. Our RKHS-based methods better preserve these details compared to PCE-based methods.}
    \label{fig:pred_deer}
\end{figure}

\subsection{Effect of TD Regularization when Under-segmentation Occurs in Network Prediction } \label{sec:network_td_undersegment}

Figure~\ref{fig:td_undersegment} provides an example of when under-segmentation happened. The 'RKHS' cannot capture the pillar effectively. With the TD regularization, it regards the prediction as an artifact and tends to dampen it.  

\begin{figure}[!ht]
    \centering
    \begin{subfigure}[b]{0.2\textwidth}
        \centering
        \includegraphics[width=\textwidth]{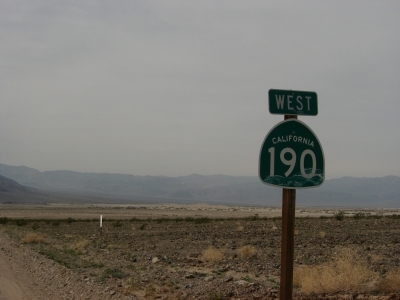} \\
        \caption{Image}
    \end{subfigure}
    \begin{subfigure}[b]{0.2\textwidth}
        \centering
        \includegraphics[width=\textwidth]{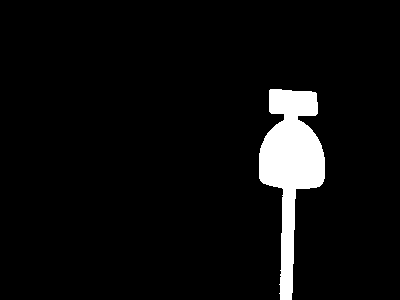} \\
        \caption{Mask}
    \end{subfigure}
    \begin{subfigure}[b]{0.2\textwidth}
        \centering
        \includegraphics[width=\textwidth]{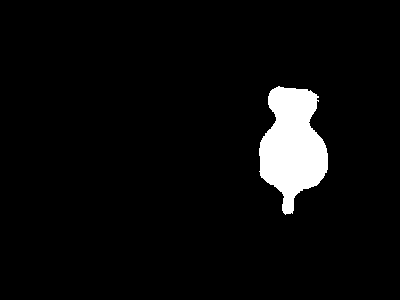} \\
        \caption{RKHS}
    \end{subfigure}
    \begin{subfigure}[b]{0.2\textwidth}
        \centering
        \includegraphics[width=\textwidth]{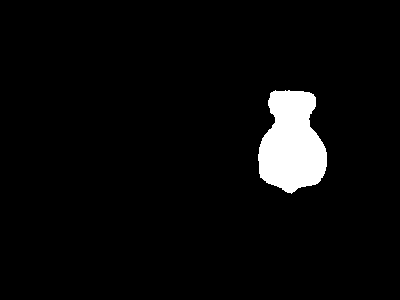} \\
        \caption{RKHS+TD}
    \end{subfigure}
    \caption{Effect of TD Regularization when under-segment appeared in 'RKHS'. True positives may be regarded as artifacts and have been smoothed}
    \label{fig:td_undersegment}
\end{figure}

\end{document}